\documentclass{article}
\usepackage[utf8]{inputenc}
\usepackage{microtype}
\usepackage{graphicx}
\usepackage{subfigure}
\usepackage{booktabs} 
\usepackage[a4paper, total={6.75in, 9in}]{geometry}

\usepackage[round]{natbib}
\usepackage{amsmath,amsthm,amssymb,amsfonts}
\usepackage[dvipsnames]{xcolor}
\usepackage{tikz}
\usepackage{pgfplots}
\pgfplotsset{compat = newest}
\usetikzlibrary{calc}
\usepackage{sidecap} 
\usepackage{enumitem} 
\usepackage{array}
\newcolumntype{L}[1]{>{\raggedright\let\newline\\\arraybackslash\hspace{0pt}}m{#1}}
\newcolumntype{C}[1]{>{\centering\let\newline\\\arraybackslash\hspace{0pt}}m{#1}}
\newcolumntype{R}[1]{>{\raggedleft\let\newline\\\arraybackslash\hspace{0pt}}m{#1}}

\definecolor{babyblue}{rgb}{0.54, 0.81, 0.94}
\definecolor{bananamania}{rgb}{0.98, 0.91, 0.71}
\definecolor{darkpastelgreen}{rgb}{0.01, 0.75, 0.24}
\definecolor{darkcandyapplered}{rgb}{0.64, 0.0, 0.0}
\definecolor{darkred}{RGB}{140,0,0}

\usepackage[colorlinks=true, allcolors=blue]{hyperref}
\usepackage{algorithmic}
\usepackage{algorithm}

\input{mymacros.tex}
\usepackage{minitoc}

\usepackage{xspace}
\newcommand{\oful}{{\small\textsc{OFUL}}\xspace}
\newcommand{\linucb}{{\small\textsc{LinUCB}}\xspace}

\newcommand{\hls}{{\small\textsc{HLS}}\xspace}
\newcommand{\cmb}{{\small\textsc{CMB}}\xspace}
\newcommand{\wys}{{\small\textsc{WYS}}\xspace}
\newcommand{\bbk}{{\small\textsc{BBK}}\xspace}

\newcommand{\A}{\mathcal{A}}
\newcommand{\X}{\mathcal{X}}
\newcommand{\transp}{\mathsf{T}}

\renewcommand{\Re}{\mathbb{R}}
\newcommand{\algo}{\textsc{LEADER}\xspace}
\newcommand{\regbal}{\textsc{RegBal}\xspace}
\newcommand{\regbalelim}{\textsc{RegBalElim}\xspace}
\newcommand{\corral}{\textsc{Corral}\xspace}
\newcommand{\expthree}{\textsc{EXP3.P}\xspace}
\newcommand{\expfour}{\textsc{EXP4.IX}\xspace}

\newcommand{\argmin}{\operatornamewithlimits{argmin}}
\newcommand{\argmax}{\operatornamewithlimits{argmax}}

\newtheorem{definition}{Definition}
\newtheorem{assumption}{Assumption}
\newtheorem{proposition}{Proposition}
\newtheorem{corollary}{Corollary}
\newtheorem{theorem}{Theorem}
\newtheorem{lemma}{Lemma}

\newtheorem*{theorem*}{Theorem}
\newtheorem*{proposition*}{Proposition}
\newtheorem*{corollary*}{Corollary}

\usepackage{authblk}

\title{Leveraging Good Representations in Linear Contextual Bandits}
\author[1]{Matteo Papini\thanks{Work done when Matteo Papini was with Facebook AI Research.}}
\author[1]{Andrea Tirinzoni}
\author[1]{Marcello Restelli}
\author[2]{Alessandro Lazaric}
\author[2]{Matteo Pirotta}
\affil[1]{Politecnico di Milano}
\affil[2]{Facebook AI Research}
\date{}

\begin{document}

\doparttoc 
\faketableofcontents 

\maketitle

\begin{abstract}
	The linear contextual bandit literature is mostly focused on the design of efficient learning algorithms for a given representation.
	However, a contextual bandit problem may admit multiple linear representations, each one with different characteristics that directly impact the regret of the learning algorithm. In particular, recent works showed that there exist ``good'' representations for which constant problem-dependent regret can be achieved.
	In this paper, we first provide a systematic analysis of the different definitions of ``good'' representations proposed in the literature. We then propose a novel selection algorithm able to adapt to the best representation in a set of $M$ candidates. We show that the regret is indeed never worse than the regret obtained by running \textsc{LinUCB} on the best representation (up to a $\ln M$ factor). As a result, our algorithm achieves constant regret whenever a ``good'' representation is available in the set. Furthermore, we show that the algorithm may still achieve constant regret by implicitly constructing a ``good'' representation, even when none of the initial representations is ``good''. Finally, we empirically validate our theoretical findings in a number of standard contextual bandit problems.
\end{abstract}


\vspace{-0.2in}
\section{Introduction}
\vspace{-0.03in}
The stochastic contextual bandit is a general framework to formalize sequential decision-making problems in which at each step the learner observes a context drawn from a fixed distribution, it plays an action, and it receives a noisy reward. The goal of the learner is to maximize the reward accumulated over $n$ rounds, and the performance is typically measured by the regret w.r.t.\ playing the optimal action in each context.
This paradigm has found application in a large range of domains, including recommendation systems, online advertising, and clinical trials~\citep[e.g.,][]{bouneffouf2019a-survey}. 
Linear contextual bandit~\citep{lattimore2020bandit} is one of the most studied instances of contextual bandit due to its efficiency and strong theoretical guarantees. In this setting, the reward for each context $x$ and action $a$ is assumed to be representable as the linear combination between $d$-dimensional features $\phi(x,a)\in \Re^d$ and an unknown parameter $\theta^\star \in \Re^d$. In this case, we refer to $\phi$ as a realizable representation.
Algorithms based on the optimism-in-the-face-of-uncertainty principle such as \linucb~\citep{chu2011linucb} and \oful~\citep{abbasi2011improved}, have been proved to achieve minimax regret bound $O\big(Sd\sqrt{n} \ln(nL)\big)$ and problem-dependent regret $O\big(\frac{S^2 d^2}{\Delta} \ln^2 (nL) \big)$, where $\Delta$ is the minimum gap between the reward of the best and second-best action across contexts, and $L$ and $S$ are upper bounds to the $\ell_2$-norm of the features $\phi$ and $\theta^{\star}$, respectively.

Unfortunately, the dimension $d$, and the norm upper bounds $L$ and $S$, are not the only characteristics of a representation to have an effect on the regret and existing bounds may fail at capturing the impact of the context-action features on the performance of the algorithm. In fact, as illustrated in Fig.~\ref{fig:example_intro}, running \linucb with different realizable representations with same parameters $d$ and $S$ may lead to significantly different performance. Notably, there are ``good'' representations for which \linucb achieves \textit{constant} regret, i.e., not scaling with the horizon $n$. Recent works identified different conditions on the representation that can be exploited to achieve constant regret for \linucb~\citep{hao2020adaptive,wu2020stochastic}. Similar conditions have also been leveraged to prove other interesting learning properties, such as sub-linear regret for greedy algorithms~\citep{bastani2020mostly}, or regret guarantees for model selection between linear and multi-arm representations~\citep{chatterji2020osom,ghosh2020problem}.
While all these conditions, often referred to as \textit{diversity conditions}, depend on how certain context-arm features span the full $\Re^d$ space, there is no systematic analysis of their connections and of which ones can be leveraged to achieve constant regret in linear contextual bandits.

In this paper, we further investigate the concept of ``good'' representations in linear bandit and we provide the following contributions: \textbf{1)} We review the diversity conditions available in the literature, clarify their relationships, and discuss how they are used. We then focus on our primary goal, which is to characterize the assumptions needed to achieve constant regret for \linucb. \textbf{2)} We introduce a novel algorithm that effectively selects the best representation in a given set, thus achieving constant regret whenever at least one ``good'' representation is provided. \textbf{3)} Furthermore, we show that, in certain problems, the algorithm is able to combine given representations to implicitly form a ``good'' one, 
thus achieving constant problem-dependent regret even when running \linucb on any of the representations would not. \textbf{4)} Finally, we empirically validate our theoretical findings on a number of contextual bandit problems.

\begin{figure}[t]
\begin{center}	
	\includegraphics{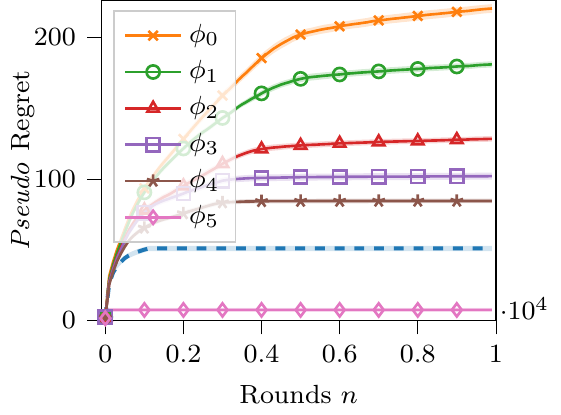}
	\caption{Regret of \linucb with different \emph{realizable} representations with same dimension $d$ and parameter bound $S$. The dashed blue line is \algo, our proposed representation selection algorithm.
	Details in App.~\ref{app.exp.toy}.
	}	\label{fig:example_intro}
\end{center}
	\vspace{-0.2in}
\end{figure}

\textbf{Related work.}
The problem of selecting the best representation in a given set can be seen as a specific instance of the problem of \textit{model selection} in bandits. In model selection, the objective is to choose the best candidate in a set of \textit{base learning algorithms}. 
At each step, a \textit{master algorithm} is responsible for selecting a base algorithm, which in turn prescribes the action to play and the reward is then provided as feedback to the base algorithms. 
Examples of model selection methods include adversarial masters --e.g., \textsc{EXP4}~\citep{auer2002nonstochastic,maillardM11} and \corral~\citep{agarwal2017corral,pacchiano2020stochcorral}-- and stochastic masters~\citep{abbasiyadkori2020regret,lee2020online,bibaut2020rateadaptive,pacchiano2020regret}. For a broader discussion refer to App.~\ref{app:relatedwork} or~\citep[Sec. 2]{pacchiano2020regret}.
Most of these algorithms achieve the regret of the best base algorithm up to a polynomial dependence on the number $M$ of base algorithms~\citep{agarwal2017corral}. While existing model selection methods are general and can be applied to any type of base algorithms,~\footnote{Most of existing methods only require prior knowledge of the regret of the optimal base algorithm or a bound on the regret of all base algorithms. \corral also requires the base algorithms to satisfy certain stability conditions.} they may not be effective in problems with a specific structure.

An alternative approach is to design the master algorithm for a specific category of base algorithms. An instance of this case is the representation-selection problem, where the base algorithms only differ by the representation used to estimate the reward. \citet{foster2019nested} and ~\citet{ghosh2020problem} consider a set of nested representations, where the best representation is the one with the smallest dimensionality for which the reward is realizable. 
Finally, \citet{chatterji2020osom} focus on the problem of selecting between a linear and a multi-armed bandit representation. In this paper, we consider an alternative representation-selection problem in linear contextual bandits, where the objective is to exploit constant-regret ``good'' representations.
Differently from our work, \citet{lattimore2020learning} say that a linear representation is ``good'' if it has a low \emph{misspecification} (i.e., it represents the reward up to a small approximation error), while we focus on \emph{realizable} representations for which \linucb achieves constant-regret.


\vspace{-0.1in}
\section{Preliminaries}\label{sec:preliminaries}
\vspace{-0.03in}
We consider the stochastic contextual bandit problem (\emph{contextual problem} for short) with context space $\X$ and finite action set $\A = [K] = \{1, \ldots, K\}$.
At each round $t \geq 1$, the learner observes a context $x_t$ sampled i.i.d.\ from a distribution $\rho$ over $\X$, it selects an arm $a_t \in [K]$ and it receives a reward $y_t = \mu(x_t,a_t) + \eta_t$ where $\eta_t$ is a $\sigma$-subgaussian noise. The learner's objective is to minimize the pseudo-regret $R_n = \sum_{t=1}^n \mu^\star(x_t) - \mu(x_t, a_t)$ for any $n >0$,
where $\mu^\star(x_t) := \max_{a \in [K]} \mu(x_t,a)$. We define the minimum gap as $\Delta = \inf_{x \in \X: \rho(x) > 0,a \in [K], \Delta(x,a)>0} \{\Delta(x,a)\}$ where $\Delta(x,a) = \mu^\star(x) - \mu(x,a)$.
A \emph{realizable} $d_\phi$-dimensional linear representation is a feature map $\phi:\X\times[K]\to\Re^{d_\phi}$ for which there exists an unknown parameter vector $\theta_\phi^\star\in\Re^{d_\phi}$ such that $\mu(x,a) = \langle \phi(x,a), \theta^\star_\phi \rangle$. When a realizable linear representation is available, the problem is called (stochastic) linear contextual bandit and can be solved using, among others, optimistic algorithms like \linucb\citep{chu2011linucb} or \oful~\citep{abbasi2011improved}.

Given a realizable representation $\phi$, at each round $t$, \linucb builds an estimate $\theta_{t\phi}$ of $\theta^\star_{\phi}$ by ridge regression using the observed data. Denote by $V_{t\phi} = \lambda I_{d_\phi} + \sum_{k=1}^{t-1} \phi(x_k,a_k) \phi(x_k,a_k)^\transp$ the $(\lambda > 0)$-regularized design matrix at round $t$, then $\theta_{t\phi} = V_{t\phi}^{-1} \sum_{k=1}^{t-1} \phi(x_k,a_k) y_k$.
Assuming that $\|\theta^\star_{\phi}\|_2\leq S_{\phi}$ and $\sup_{x,a} \|\phi(x,a)\|_2 \leq L_{\phi}$, \linucb builds a confidence ellipsoid $\mathcal{C}_{t\phi}(\delta) = \big\{ \theta \in \Re^{d_{\phi}} : \big\|\theta_{t\phi} - \theta\big\|_{V_{t\phi}} \leq \beta_{t\phi}(\delta) \big\}$. As shown in~\citep[Thm. 1]{abbasi2011improved}, when
\begin{equation*}\label{eq:linucb.beta}
    \beta_{t\phi}(\delta) := \sigma\sqrt{2\ln\left(\frac{\det(V_{t\phi})^{1/2}\det(\lambda I_{d_{\phi}})^{-1/2}}{\delta}\right)} + \sqrt{\lambda} S_{\phi},
\end{equation*}
then $\mathbb{P}(\forall t\geq 1, \theta^\star_{\phi} \in \mathcal{C}_{t\phi}(\delta)) \geq 1-\delta$. At each step $t$, \linucb plays the action with the highest upper-confidence bound $a_t = \argmax_{a \in [K]}\max_{\theta \in \mathcal{C}_{t\phi}(\delta)} \langle \phi(x_t,a), \theta \rangle$,
and it is shown to achieve a regret bounded as reported in the following proposition.

\begin{proposition}[\citealp{abbasi2011improved}, Thm. 3, 4]\label{prop:oful.log.regret}
For any linear contextual bandit problem with $d_{\phi}$-dimensional features, $\sup_{x,a}\|\phi(x,a)\|_2\leq L_{\phi}$, an unknown parameter vector $\|\theta^\star_{\phi}\|_2 \leq S_{\phi}$, with probability at least $1-\delta$, \linucb suffers regret
%
	$R_n = O\big(S_{\phi}d_{\phi}\sqrt{n} \ln(nL_{\phi}/\delta)\big)$.
%
Furthermore, if the problem has a minimum gap $\Delta > 0$, then the regret is bounded as\footnote{The logarithmic bound reported in Prop.~\ref{prop:oful.log.regret} is slightly different than the one in~\citep{abbasi2011improved} since we do not assume that the optimal feature is unique.}
%
	$R_n = O\bigg(\frac{S_{\phi}^2 d^2_{\phi}}{\Delta} \ln^2 (nL_{\phi}/\delta) \bigg)$.
%
\end{proposition}

\vspace{-0.1in}
In the rest of the paper, we assume w.l.o.g. that all terms $\lambda$, $\Delta_{\max}=\max_{x,a} \Delta(x,a)$, $S_\phi$, $\sigma$ are larger than $1$ to simplify the expression of the bounds.

\vspace{-0.1in}
\section{Diversity Conditions}
\vspace{-0.03in}
Several assumptions, usually referred to as \textit{diversity conditions}, 
have been proposed to define linear bandit problems with specific properties that can be leveraged to derive improved learning results. While only a few of them were actually leveraged to derive constant regret guarantees for \linucb (others have been used to prove e.g., sub-linear regret for the greedy algorithm, or regret guarantees for model selection algorithms), they all rely on very similar conditions on how certain context-action features span the full $\Re^{d_\phi}$ space.
In this section, we provide a thorough review of these assumptions, their connections, and how they are used in the literature. As diversity conditions are getting more widely used in bandit literature, we believe this review may be of independent interest. Sect.~\ref{sec:hls} will then specifically focus on the notion of \emph{good} representation for \linucb.


We first introduce additional notation. For a realizable representation $\phi$, let $\phi^\star(x):= \phi(x,a^\star_x)$, where $a^\star_x \in \argmax_{a \in [K]}\mu(x,a)$ is an optimal action, be the vector of \emph{optimal features} for context $x$. In the following we make the assumption that $\phi^\star(x)$ is unique. Also, let $\X^\star(a) =\{x \in \X : \mu(x,a) = \mu^\star(x)\}$ denote the set of contexts where $a$ is optimal. Finally, for any matrix $A$, we denote by $\lambda_{\min}(A)$ its minimum eigenvalue.
For any contextual problem with reward $\mu$ and context distribution $\rho$, the diversity conditions introduced in the literature are summarized in Tab.~\ref{tab:diversity.conditions} together with how they were leveraged to obtain regret bounds in different settings.\footnote{In some cases, we adapted conditions originally defined in the disjoint-parameter setting, where features only depend on the context (i.e., $\phi(x)$) and the unknown parameter $\theta^\star_a$ is different for each action $a$, to the shared-parameter setting (i.e., where features are functions of both contexts and actions) introduced in Sect.~\ref{sec:preliminaries}.}

\begin{figure*}[t]
    \begin{minipage}{0.75\textwidth}
    \centering
\renewcommand{\arraystretch}{1.2}
\begin{footnotesize}
\begin{tabular}{|C{1.5cm}|C{7.4cm}|C{2.4cm}|}
\hline
Name & Definition & Application \\
\hline
\hline
Non-redundant & $\lambda_{\min}\Big(1/K\sum_{a\in[K]}\mathbb{E}_{x\sim\rho} \big[\phi(x,a) \phi(x,a)^\transp \big]\Big)>0$ & \\
\hline
\cmb & $\forall a$, $\lambda_{\min}\Big(\mathbb{E}_{x\sim\rho} \big[\phi(x,a) \phi(x,a)^\transp \big]\Big)>0$ & \textit{Model selection}\\
\hline
\bbk & $\forall a, \boldsymbol{u}\in\Reals^d$, $\lambda_{\min}\Big(\EV_{x}\left[\phi(x,a)\phi(x,a)^\transp\indi{\phi(x,a)^\transp \boldsymbol{u}\ge 0}\right]\Big)>0$ & \textit{Logarithmic regret for greedy} \\ 
\hline
\hls & $\lambda_{\min}\Big(\mathbb{E}_{x\sim\rho} \big[\phi^\star(x) \phi^\star(x)^\transp \big]\Big)>0$ & \textit{Constant regret for \textsc{LinUCB}} \\
\hline
\wys & $\forall a$, $\lambda_{\min}\Big(\mathbb{E}_{x\sim\rho} \big[\phi(x,a) \phi(x,a)^\transp \indi{x\in\X^\star(a)} \big]\Big)>0$ & \textit{Constant regret for \textsc{LinUCB}}\\
\hline
\hline
\end{tabular}
\end{footnotesize}
\caption{Diversity conditions proposed in the literature adapted to the shared-parameter setting. The names refer to the authors who first introduced similar conditions.}\label{tab:diversity.conditions}
\end{minipage}
~~
\begin{minipage}{0.2\textwidth}
    \centering
    \begin{tikzpicture}
        \draw[draw=babyblue!70!blue, double, fill=babyblue, fill opacity=0.4] (0,0.4) ellipse (40pt and 40pt);
        \draw[draw=orange, dashed, fill=bananamania, fill opacity=0.5] (0, -.7) ellipse (40pt and 40pt);
        \draw[draw=Green, dotted, fill=darkpastelgreen, fill opacity=0.2] (0,-0.6) ellipse (29pt and 31pt);
        \node[blue] (hls) at (0,1.2) {\hls};
        \node[Orange] (cmb) at (0,-1.9) {\cmb};
        \node[Green] (bbk) at (0, -1.3) {\bbk};
        \node[draw=darkcandyapplered,fill=darkcandyapplered,fill opacity=0.2, text opacity=1, text=darkcandyapplered,circle] (wys) at (0,-0.25) {\wys};
        \draw[draw=gray] (-1.6,1.9) rectangle (1.6,-2.3);
        \node[anchor=west,gray] at (-1.6,2.1) {\small Non-redundant};
    \end{tikzpicture}
    \caption{Categorization of diversity conditions.}
    \label{fig:rep.asm.classification}
\end{minipage}
\end{figure*}

We first notice that all conditions refer to the smallest eigenvalue of a design matrix constructed on specific context-action features. In other words, diversity conditions require certain features to span the full $\Re^{d_\phi}$ space. The non-redundancy condition is a common technical assumption~\citep[e.g.,][]{foster2019nested} and it simply defines a problem whose dimensionality cannot be reduced without losing information.  
Assuming the context distribution $\rho$ is full support, \bbk and \cmb are structural properties of the representation that are independent from the reward. For example, \bbk requires that, for each action, there must be feature vectors lying in all orthants of $\Re^{d_\phi}$. In the case of finite contexts, this implies there must be at least $2^{d_\phi}$ contexts.
\wys and \hls involve the notion of reward optimality. In particular, \wys requires that all actions are optimal for at least a context (in the continuous case, for a non-negligible set of contexts), while \hls only focuses on optimal actions.

We now review how these conditions (or variations thereof) were applied in the literature.
\cmb is a rather strong condition that requires the features associated with each individual action to span the whole $\Re^{d_\phi}$ space. \citet{chatterji2020osom} leverage a \cmb-like assumption to prove regret bounds for \textsc{OSOM}, a model-selection algorithm that unifies multi-armed and linear contextual bandits. More precisely, they consider a variation of \cmb, where the context distribution induces stochastic feature vectors for each action that are independent and centered. The same condition was adopted by~\citet{ghosh2020problem} to study representation-selection problems and derive algorithms able to adapt to the (unknown) norm of $\theta_{\phi}^\star$ or select the smallest realizable representation in a set of nested representations.  
\citet[][Assumption 3]{bastani2020mostly} introduced a condition similar to \bbk for the \emph{disjoint-parameter} case. In their setting, they prove that a non-explorative greedy algorithm achieves $O(\ln(n))$ problem-dependent regret in linear contextual bandits (with $2$ actions).\footnote{Whether this is enough for the optimality of the greedy algorithm in the shared-parameter setting is an interesting problem, but it is beyond the scope of this paper.} 
\citet[][Theorem 3.9]{hao2020adaptive} showed that \hls representations can be leveraged to prove constant problem-dependent regret for \linucb in the shared-parameter case. 
Concurrently, \citet{wu2020stochastic} showed that, under \wys, \linucb achieves constant expected regret in the disjoint-parameter case. A \wys-like condition was also used  by~\citet[][Assumption 4]{bastani2020mostly} to extend the result of sublinear regret for the greedy algorithm to more than two actions.
The relationship between all these conditions is derived in the following lemma.
\begin{lemma}\label{prop:rep.asm.properties}
For any contextual problem with reward $\mu$ and context distribution $\rho$, let $\phi$
be a realizable linear representation. The relationship between the diversity conditions in Tab.~\ref{tab:diversity.conditions} is summarized in Fig.~\ref{fig:rep.asm.classification}, where each inclusion is in a strict sense and each intersection is non-empty.
\end{lemma}

This lemma reveals non-trivial connections between the diversity conditions, better understood through the examples provided in the proof (see App.~\ref{app:rep.structures}). \bbk is indeed stronger than \cmb, and thus it is sufficient for the model selection results by~\citet{chatterji2020osom}. 
By superficially examining their definitions, \cmb may appear stronger than \hls, but the two properties are actually non-comparable, as there are representations that satisfy one condition but not the other. The implications of Fig.~\ref{fig:rep.asm.classification} on constant-regret guarantees are particularly relevant for our purposes.
There are representations that satisfy \bbk or \cmb and are neither \hls nor \wys and thus may not enable constant regret for \linucb. 
We notice that \wys is a stronger condition than \hls. Although \wys may be necessary for \linucb to achieve constant regret in the disjoint-parameter case, \hls is sufficient for the shared-parameter case we consider in this paper. For this reason, in the following section we  adopt \hls to define \emph{good} representations for \linucb and provide a more complete characterization.
\vspace{-0.1in}
\section{\textit{Good} Representations for Constant Regret}\label{sec:hls}
\vspace{-0.03in}
The \hls condition was introduced by~\citet{hao2020adaptive}, who provided a first analysis of its properties. In this section, we complement those results by providing a complete proof of a constant regret bound, a proof of the fact that \hls is actually necessary for constant regret, and a novel characterization of the existence of \hls representations.
In the following we define $\lambda_{\phi,\hls} := \lambda_{\min}\big(\mathbb{E}_{x\sim\rho} \big[\phi^\star(x) \phi^\star(x)^\transp \big]\big)$, which is strictly positive for \hls representations. 

\subsection{Constant Regret Bound}
We begin by deriving a constant problem-dependent regret bound for \linucb under the \hls condition.

\begin{lemma}\label{prop:hls.regret}
    Consider a contextual bandit problem with realizable linear representation $\phi$ satisfying the \hls condition (see Tab.~\ref{tab:diversity.conditions}). Assume $\Delta > 0$, $\max_{x,a}\|\phi(x,a)\|_2 \leq L$ and $\|\theta^\star_\phi\|_2 \leq S$. Then, with probability at least $1-2\delta$, the regret of \oful after $n \geq 1$ steps is at most

	\vspace{-0.2in}
	{\small
	\begin{align*}
		R_n \leq 
		&\frac{32\lambda\Delta_{\max}^2 S_{\phi}^2\sigma^2}{\Delta} 
		\left(2\ln\left(\frac{1}{\delta}\right)
		+d_{\phi} \ln \left( 1+\frac{{\color{darkred}\tau_{\phi}} L_{\phi}^2}{\lambda d_{\phi}} \right) \right)^2,
	\end{align*}
	}

	\vspace{-0.15in}
	where $\Delta_{\max} = \max_{x,a} \Delta(x,a)$ is the maximum gap and
	{\small
	\begin{align*}
		\tau_{\phi} \leq \max\bigg\{
			&\frac{384^2 d_\phi^2L_\phi^2S_\phi^2\sigma^2 \lambda}{\lambda_{\phi,\hls}\Delta^2} \ln^2\left(\frac{64 d_\phi^2L_\phi^3\sigma S_\phi \sqrt{\lambda}}{\sqrt{\lambda_{\phi,\hls}}\Delta\delta} \right),\\
			&~~\frac{768 L_\phi^4}{\lambda_{\phi,\hls}^2}\ln\left(\frac{512 d_\phi L_\phi^4}{\delta \lambda_{\phi,\hls}^2}\right) \bigg\}.
	\end{align*}
	}
\end{lemma}
We first notice that $\tau_\phi$ is independent from the horizon $n$, thus making the previous bound a constant only depending on the problem formulation (i.e., gap $\Delta$, norms $L_\phi$ and $S_\phi$) and the value $\lambda_{\phi,\hls}$ which measures ``how much'' the representation $\phi$ satisfies the \hls condition. Furthermore, one can always take the minimum between the constant regret in Lem.~\ref{prop:hls.regret} and any other valid regret bound for \oful (e.g., $O(\log(n)/\Delta))$), which may be tighter for small values of $n$. While Lem.~\ref{prop:hls.regret} provides high-probability guarantees, we can easily derive a constant expected-regret bound by running \linucb with a decreasing schedule for $\delta$ (e.g., $\delta_t \propto 1/t^3$) and with a slightly different proof (see App.~\ref{app:hls.regret} and the proof sketch below).

\textbf{Proof sketch (full proof in App.~\ref{app:hls.regret}).}
Following~\citet{hao2020adaptive}, the idea is to show that the instantaneous regret $r_{t+1} = \langle \theta^\star, \phi^\star(x_{t+1}) - \phi(x_{t+1},a_{t+1}) \rangle$ is zero for sufficiently large (but constant) time $t$. By using the standard regret analysis, we have
\begin{align*}
r_{t+1} \leq 2\beta_{t+1}(\delta)\norm{\phi(x_{t+1},a_{t+1})}_{V_{t+1}^{-1}} \leq \frac{2L\beta_{t+1}(\delta)}{\sqrt{\lambda_{\min}(V_{t+1})}}.
\end{align*}
Given the minimum-gap assumption, a sufficient condition for $r_{t+1} = 0$ is that the previous upper bound is smaller than $\Delta$, which gives $\lambda_{\min}(V_{t+1}) > 4L^2 \beta_{t+1}^2(\delta)/\Delta^2$.
Since $\Delta>0$, the problem-dependent regret bound in Prop.~\ref{prop:oful.log.regret} holds, and the number of pulls to suboptimal arms up to time $t$ is bounded by $g_t(\delta) = O\big((d\ln(t/\delta)/\Delta)^2\big)$. Hence, the optimal arms are pulled linearly often and, by leveraging the \hls assumption, we are able to show that the minimum eigenvalue of the design matrix grows linearly in time as
\begin{align*}
	\lambda_{\min}(V_{t+1})
	&\ge \lambda + t\lambda_{\hls} - 8L^2\sqrt{t\ln\left(\frac{2dt}{\delta}\right)} - L^2g_t(\delta).
\end{align*}
By relating the last two equations, we obtain an inequality of the form $t\lambda_{\hls} - o(t) > o(t)$. If we define $\tau < \infty$ as the smallest (deterministic) time such that this inequality holds, we have that after $\tau$ the immediate regret is zero, thus concluding the proof. Note that, if we wanted to bound the expected regret, we could set $\delta_t \propto 1/t^3$ and the above inequality would still be of the same form (although the resulting $\tau$ would be slightly different).

\paragraph{Comparison with existing bounds.}

\citet[][Theorem 3.9]{hao2020adaptive} prove that \linucb with \hls representations achieves $\lim\sup_{n\to\infty}R_n<\infty$, without characterizing the time at which the regret vanishes. Instead, our Lem.~\ref{prop:hls.regret} provides an explicit problem-dependent constant regret bound. \citet[][Theorem 2]{wu2020stochastic} consider the disjoint-parameter setting and rely on the \wys condition. While they indeed prove a constant regret result, their bound depends on the the minimum probability of observing a context (or, in the continuous case, a properly defined meta-context). This reflects the general tendency, in previous works, to frame diversity conditions simply as a property of the context distribution $\rho$. On the other hand, our characterization of $\tau$ in terms of $\lambda_{\phi,\hls}$ (Lem.~\ref{prop:hls.regret}) allows relating the regret to the ``goodness'' of the representation $\phi$ for the problem at hand.

\subsection{Removing the Minimum-Gap Assumption}\label{sec:zero.gap}

Constant-regret bounds for \linucb rely on a minimum-gap assumption ($\Delta>0$). In this section we show that \linucb can still benefit from \hls representations when $\Delta=0$, but a margin condition holds~\citep[e.g.,][]{rigollet2010nonparametric,ReeveM018}. Intuitively, we require that the probability of observing a context $x$ decays proportionally to its minimum gap $\Delta(x) = \min_{a} \Delta(x,a)$.
\begin{assumption}[Margin condition] \label{asm:margin.gap}
    There exists $C,\alpha>2$ such that for all $\epsilon>0$:
        $\rho\big(\{x \in \X : \Delta(x) \leq \epsilon \}\big) \leq C\epsilon^{\alpha}$.
\end{assumption}
The following theorem provides a problem-dependent regret bound for \linucb under this margin assumption.
\begin{theorem}\label{th:hls.margin}
	Consider a linear contextual bandit problem satisfying the margin condition (Asm.~\ref{asm:margin.gap}).
	Assume $\max_{x,a}\|\phi(x,a)\|_2 \leq L_{\phi}$ and $\|\theta_{\phi}^\star\|_2 \leq S_{\phi}$.
	Then, given a representation $\phi$, with probability at least $1-3\delta$, the regret of \oful after $n \geq 1$ steps is at most
    \begin{align*}
    R_n \leq {O}\bigg( \Big( \lambda(\Delta_{\max}S_{\phi}\sigma d_{\phi})^2 {\color{darkred}n^{1/\alpha}}  + \sqrt{Cd_{\phi}} \Big)\ln^2 (L_{\phi}n/\delta) \bigg).
    \end{align*}
    When $\phi$ is \hls ($\lambda_{\phi,\hls}>0$), let $\tau_{\phi} \propto (\lambda_{\phi,\hls})^{\frac{\alpha}{2 - \alpha}}$, then
     \begin{align*}
    R_n \leq {O}\Big(\Delta_{\max} {\color{darkred}\tau_{\phi}} +\sqrt{Cd_{\phi}}\ln^2(L_{\phi}n/\delta) \Big).
    \end{align*}
\end{theorem} 

We first notice that in general, \linucb suffers $\wt{O}(n^{1/\alpha})$ regret, which can be significantly larger than in the minimum-gap case. On the other hand, with \hls representations, \linucb achieves logarithmic regret, regardless of the value of $\alpha$. The intuition is that, when the \hls condition holds, the algorithm collects sufficient information about $\theta^\star_{\phi}$ by pulling the optimal arms in rounds with large minimum gap, which occur with high probability by the margin condition. This yields at most constant regret in such rounds (first term above), while it can be shown that the regret in steps when the minimum gap is very small is at most logarithmic (second term above).

\subsection{Further Analysis of the \hls Condition}

While Lem.~\ref{prop:hls.regret} shows that \hls is sufficient for achieving constant regret, the following proposition shows that it is also necessary.  
While this property was first mentioned by~\citet{hao2020adaptive} as a remark in a footnote, we provide a formal proof in App.~\ref{app:hls.regret.iif}.	 

\begin{proposition}
	\label{prop:constant.iif.hls}
	For any contextual problem with finite contexts, full-support context distribution, and given a realizable representation $\phi$, \linucb achieves sub-logarithmic regret if and only if $\phi$ satisfies the \hls condition.
\end{proposition}

Finally, we derive the following important existence result. 

\begin{lemma}\label{prop:existence.hls}
	For any contextual bandit problem with optimal reward \footnote{This condition is technical and it can be easily relaxed.} $\mu^\star(x)\neq 0$ for all $x\in\X$, that has either
	\emph{i)} a finite context set with at least $d$ contexts with nonzero probability, or
	\emph{ii)} a Borel context space and a non-degenerate context distribution\footnote{For instance, if $\mathcal{X}=\Reals^m$ and the context distribution must have positive variance in all directions.},
	for any dimension $d\ge 1$, there exists an infinite number of $d$-dimensional realizable \hls representations.
\end{lemma}

This result crucially shows that the \hls condition is ``robust'', since in any contextual problem, it is possible to construct an infinite number of representations satisfying the \hls condition. In App.~\ref{app:rep.structures.hls.existence}, we indeed provide an oracle procedure for constructing an \hls representation. This result also supports the starting point of next section, where we assume that a learner is provided with a set of representations that may contain at least a ``good'' representation, i.e., an \hls representation. 

\vspace{-0.1in}
\section{Representation Selection}
\vspace{-0.03in}
In this section, we study the problem of \emph{representation selection} in linear bandits.
We consider a linear contextual problem with reward $\mu$ and context distribution $\rho$.
Given a set of $M$ realizable linear representations $\{\phi_i : \X \times [K] \to \Re^{d_i}\}$, the objective is to design a learning algorithm able to perform as well as the best representation, and thus achieve constant regret when a ``good'' representation is available. 
As usual, we assume $\theta^\star_i \in \Re^{d_i}$ is unknown, but the algorithm is provided with a bound on the parameter and feature norms of the different representations.

\subsection{The \texorpdfstring{\algo}{LEADER} Algorithm}
We introduce \algo (\emph{Linear rEpresentation bAnDit mixER}), see Alg.~\ref{alg:our.algo}.
At each round $t$, \algo 
builds an estimate $\theta_{ti}$ of the unknown parameter $\theta_i^\star$ of each representation $\phi_i$.\footnote{We use the subscript $i \in [M]$ instead of $\phi_i$ to denote quantities related to representation $\phi_i$.} These estimates are by nature off-policy, and thus \emph{all} the samples $(x_l, a_l, y_l)_{l < t}$ can be used to solve \emph{all} ridge regression problems. 
For each $\phi_i$, define $V_{ti} = \lambda I_{d_i} + \sum_{l=1}^{t-1} \phi_i(x_l,a_l) \phi_i(x_l, a_l)^\transp$, $\theta_{ti}$ and $\mathcal{C}_{ti}(\delta/M)$ as in Sec.~\ref{sec:preliminaries}.
Since all the representations are realizable, we have that $\mathbb{P}\left(\forall i \in [M], \theta^\star_i \in \mathcal{C}_{ti}(\delta/M) \right) \geq 1-\delta$.
As a consequence, for each representation $\phi_i$ we can build an upper-confidence bound to the reward such that, $\forall x \in \X,a \in \A$, with high probability
\begin{equation}
        \mu(x,a)
    \leq \max_{\theta \in \mathcal{C}_{ti}(\delta/M)} \langle \phi_i(x,a),\theta\rangle := U_{ti}(x,a).
\end{equation} 
Given this, \algo uses the tightest available upper-confidence bound to evaluate each action and then it selects the one with the largest value, i.e., 
\begin{align}\label{eq:algo.actionselection}
    a_t \in \argmax_{a \in [K]} \min_{i \in [M]} \{U_{ti}(x_t,a)\}.
\end{align}
Let $i_t = \argmin_{i \in [M]} \{U_{ti}(x_t,a_t)\}$ be the representation associated to the pulled arm $a_t$.
Interestingly, despite $a_t$ being optimistic, in general it  may not correspond to the optimistic action of representation $\phi_{i_t}$, i.e., $a_t \notin \argmax_{a} \{ U_{t,i_t}(x_t,a) \}$. If a representation provides an estimate that is good along the direction associated to a context-action pair, but possibly very uncertain on other actions, \algo is able to leverage this key feature to reduce the overall uncertainty and achieve a tighter optimism.
Space and time complexity of \algo scales linearly in the number of representations, although the updates for each representation could be carried out in parallel.

\paragraph{Regret bound.}
For ease of presentation, we assume a non-zero minimum gap ($\Delta > 0$).
The analysis can be generalized to $\Delta = 0$ as done in Sec.~\ref{sec:zero.gap}.
Thm.~\ref{thm:algo.regret.positivegap} establishes the regret guarantee of \algo (Alg.~\ref{alg:our.algo}).

\begin{theorem}\label{thm:algo.regret.positivegap}
Consider a contextual bandit problem with reward $\mu$, context distribution $\rho$ and $\Delta >0$.
Let $(\phi_i)$ be a set of $M$ linearly realizable representations such that $\max_{x,a} \|\phi_i(x,a)\|_2 \leq L_i$ and $\|\theta^\star_i\|_i \leq S_i$. Then, for any $n \geq 1$, with probability $1-2\delta$, \algo suffers a regret

\vspace{-0.15in}
{\small
\begin{align*}
    R_n \leq 
    &\min_{i\in[M]}\bigg\{
        \frac{32\lambda\Delta_{\max}^2 S_i^2 \sigma^2}{\Delta} 
        \times \\
    & \times \left(2\ln\left(\frac{M}{\delta}\right)
        +d_i \ln \left( 1+\frac{{\color{darkred}\min\{\tau_i, n\} }L_i^2}{\lambda d_i} \right) \right)^2 \bigg\}
\end{align*}
}

\vspace{-0.1in}
where $\tau_i \propto (\lambda_{i,\hls} \Delta)^{-2}$ if $\phi_i$ is \hls and $\tau_i = +\infty$ otherwise. 
\end{theorem}
This shows that the problem-dependent regret bound of \algo is not worse than the one of the best representation (see Prop.~\ref{prop:oful.log.regret}), up to a $\ln M$ factor.
This means that the cost of representation selection is almost negligible.
Furthermore, Thm.~\ref{thm:algo.regret.positivegap} shows that \algo not only achieves a constant regret bound when an \hls representation is available, but this bound scales as the one of the \textit{best} \hls representation. In fact, notice that the ``quality'' of an \hls representation does not depend only on known quantities such as $d_i$, $L_i$, $S_i$, but crucially on \hls eigenvalue $\lambda_{i,\hls}$, which is usually not known in advance, as it depends on the features of the optimal arms.

\begin{algorithm}[tb]
    \begin{small}
    \begin{algorithmic}
        \STATE {\bfseries Input:} representations $(\phi_i)_{i \in [M]}$ with values $(L_i,S_i)_{i\in [M]}$, regularization factor $\lambda \geq 1$, confidence level $\delta \in (0,1)$.
        \caption{\algo Algorithm}\label{alg:our.algo}
        \STATE Initialize $V_{1i} = \lambda I_{d_i}$, $\theta_{1i} = 0_{d_i}$ for each $i \in [M]$
        \FOR{$t=1, \ldots$}
            \STATE Observe context $x_t$
            \STATE Pull action
            $a_t \in \argmax_{a \in [K]} \min_{i \in [M]} \{U_{ti}(x_t,a)\}$
            \STATE Observe reward $r_t$ and, for each $i\in [M]$, set\\
                $V_{t+1,i}= V_{ti} + \phi_i(x_t,a_t)  \phi_i(x_t,a_t)^\transp$ and\\
                $\theta_{t+1,i} =V_{t+1,i}^{-1} \sum_{l=1}^t \phi_i(x_l,a_l) r_l$
        \ENDFOR
    \end{algorithmic}
    \end{small}
\end{algorithm}

\subsection{Combining Representations}
In the previous section, we have shown that \algo can perform as well as the best representation in the set. 
However, by inspecting the action selection rule (Eq.~\ref{eq:algo.actionselection}), we notice that, to evaluate the reward of an action in the current context, \algo selects the representation with the smallest uncertainty, thus potentially using different representations for different context-action pairs.
This leads to the question: \emph{can \algo do better than the best representation in the set?}

We show that, in certain cases, \algo is able to combine representations and achieve constant regret when none of the individual representations would. 
The intuition is that a subset of ``locally good'' representations can be combined to recover a condition similar to \hls.
This property is formally stated in the following definition.
\begin{definition}[Mixing \hls]\label{def:hls.mix}
    Consider a linear contextual problem with reward $\mu$ and context distribution $\rho$, and a set of $M$ realizable linear representations $\phi_1, \ldots, \phi_M$.
    Define $M_i = \mathbb{E}_{x\sim \rho}\Big[\phi_i^\star(x)\phi_i^\star(x)^\transp\Big]$
    and let $Z_i=\{(x,a)\in\X \times\A \mid \phi_i(s,a)\in\Imm(M_i)\}$ be the set of context-action pairs whose features belong to the column space of $M_i$, i.e., that lie in the span of optimal features. We say that the set $(\phi_i)$ satisfies the mixed-\hls condition if $\X\times\A \subseteq \bigcup_{i=1}^M Z_i$.
\end{definition}
Let $\lambda_i^+ = \lambda_{\min}^+(M_i)$ be the minimum \emph{nonzero} eigenvalue of $M_i$.
Intuitively, the previous condition relies on the observation that every representation satisfies a ``restricted'' \hls condition on the context-action pairs $(x,a)$ whose features $\phi_i(x,a)$ are spanned by optimal features $\phi^\star(x)$. In this case, the characterizing eigenvalue is $\lambda_i^+$, instead of the smallest eigenvalue $\lambda_{i,\hls}$ (which may be zero). If every context-action pair is in the restriction $Z_i$ of some representation, we have the mixed-\hls property. In particular, if representation $i$ is \hls, ${\lambda}_i^+=\lambda_{i,\hls}$ and $Z_i=\mathcal{S}\times\mathcal{A}$. So, \hls is a special case of mixed-\hls.
In App.~\ref{app:rep.selection.mixing}, we provide simple examples of sets of representations satisfying Def.~\ref{def:hls.mix}.
Note that, strictly speaking, there is not a single ``mixed representation'' solving the whole problem. Even defining one would be problematic since each representation may have a different parameter and even a different dimension. Instead, each representation ``specializes'' on a different portion of the context-action space. If together they cover the whole space, the benefits of \hls are recovered, as illustrated in the following theorem.

\begin{theorem}\label{thm:algo.mix.regret.constant}
    Consider a stochastic bandit problem with reward $\mu$, context distribution $\rho$ and $\Delta >0$.
    Let $(\phi_i)$ be a set of $M$ realizable linear representations satisfying the mixed-\hls property in Def.~\ref{def:hls.mix}.
    Then, with probability at least $1-2\delta$, there exists a time $\tau < \infty$ independent from $n$ such that, for any $n \geq 1$, the pseudo-regret of \algo is bounded as
    %
    \begin{align*}
        R_n \leq 
        &\min_{i\in[M]} \bigg\{
            \frac{32\lambda\Delta_{\max}^2 S_i^2 \sigma^2}{\Delta} 
            \times \\
        & ~~\times \left(2\ln\left(\frac{M}{\delta}\right)
            +d_i \ln \left( 1+\frac{{\color{darkred}\tau} L_i^2}{\lambda d_i} \right) \right)^2 \bigg\}.
    \end{align*}
\end{theorem}
First, note that we are still scaling with the characteristics of the best representation in the set (i.e., $d_i$, $L_i$ and $S_i$). However, the time $\tau$ to constant regret is a global value rather than being different for each representation. 
This highlights that mixed-\hls is a global property of the set of representations rather than being individual as before. In particular, whenever no representation is (globally) \hls (i.e., $\lambda_{i,\hls} =0$ for all $\phi_i$), we can show that in the worst case $\tau$ scales as $(\min_{i} \lambda^+_i)^{-2}$. In practice, we may expect \algo to even behave better than that since \textbf{i)} not all the representations may contribute actively to the mixed-\hls condition; and \textbf{ii)} multiple representations may cover the same region of the context-action space. In the latter case, since \algo leverages all the representations at once, its regret would rather scale with the largest minimum non-zero eigenvalue $\lambda^+_i$ among all the representations covering such region. 
We refer to App.~\ref{app:rep.selection.mixing} for a more complete discussion.

\subsection{Discussion}
Most of the model selection algorithms reviewed in the introduction could be readily applied to select the best representation for \linucb. However, the generality of their objective comes with several shortcomings when instantiated in our specific problem (see App.~\ref{app:relatedwork} for a more detailed comparison).
First, model selection methods achieve the performance of the best algorithm, up to a polynomial dependence on the number $M$ of models. This already makes them a weaker choice compared to \algo, which, by leveraging the specific structure of the problem, suffers only a logarithmic dependence on $M$. 
Second, model selection algorithms are often studied in a worst-case analysis, which reveals a high cost for adaptation. For instance, corralling algorithms~\citep{agarwal2017corral,pacchiano2020stochcorral} pay an extra $\sqrt{n}$ regret, which would make them unsuitable to target the constant regret of good representations. Similar costs are common to other approaches~\citep{abbasiyadkori2020regret,pacchiano2020regret}. It is unclear whether a problem-dependent analysis can be carried out and whether this could shave off such dependence.
Third, these algorithms are generally designed to adapt to a specific best base algorithm. At the best of our knowledge, there is no evidence that model selection methods could combine algorithms to achieve better performance than the best candidate, a behavior that we proved for \algo in our setting. 

On the other hand, model selection algorithms effectively deal with non-realizable representations in certain cases~\citep[e.g.,][]{foster2020misspecified,abbasiyadkori2020regret,pacchiano2020regret}, while \algo is limited to the realizable case. While a complete study of the model misspecification case is beyond the scope of this paper, in App.~\ref{app:elimination}, we discuss how a variation of the approach presented in~\citep{agarwal2012predictablerew} could be paired to \algo to discard misspecified representations and possibly recover the properties of ``good'' representations.


\begin{figure*}[t]
	\includegraphics[width=.252\textwidth]{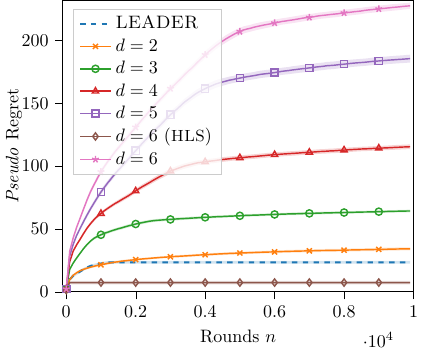}
	\includegraphics[width=.239\textwidth]{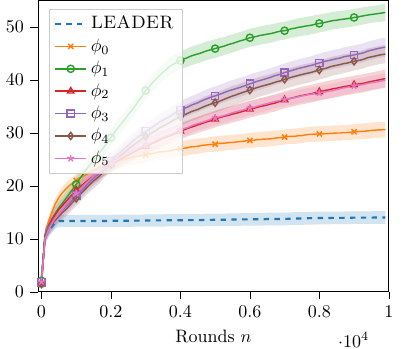}
	\includegraphics[width=.245\textwidth]{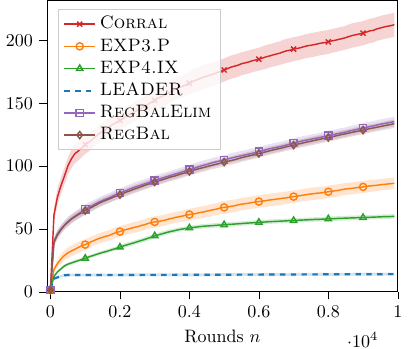}
	\includegraphics[width=.245\textwidth]{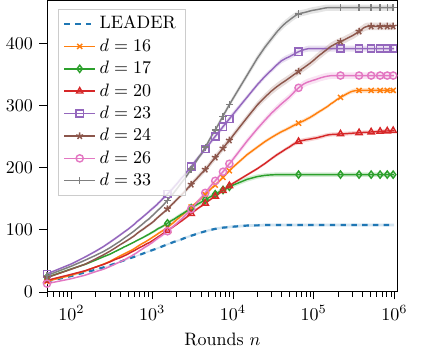}
	\vspace{-0.1in}
	\caption{Regret of \algo and model-selection baselines on different linear contextual bandit problems. (\emph{left}) Synthetic problem with varying dimensions. (\emph{middle left}) Representation mixing. (\emph{middle right}) Comparison to model selection baselines. (\emph{right}) Jester dataset.}
	\label{fig:all}
\end{figure*}

\vspace{-0.1in}
\section{Experiments} \label{sec:experiments}
\vspace{-0.03in}

In this section, we report experimental results on two synthetic and one dataset-based problems. For each problem, we evaluate the behavior of \algo with \linucb and model selection algorithms: \expfour~\citep{neu2015explore}, \corral and \expthree in the stochastic version by~\citet{pacchiano2020stochcorral} and Regret Balancing with and without elimination (\regbalelim and \regbal)~\citep[][]{abbasiyadkori2020regret,pacchiano2020regret}. See App.~\ref{app:experiments} for a detailed discussion and additional experiments. All results are averaged over $20$ independent runs, with shaded areas corresponding to $2$ standard deviations. We always set the parameters to $\lambda=1$,
$\delta=0.01$,
and $\sigma=0.3$.
All the representations we consider are normalized to have $\norm{\theta^\star_i}=1$.

\textbf{Synthetic Problems.} We define a randomly-generated contextual bandit problem, for which we construct sets of  realizable linear representations with different properties (see App.~\ref{app.exp.toy} for details).
The purpose of these experiments is twofold: to show the different behavior of \linucb with different representations, and to evaluate the ability of \algo of selecting and mixing representations.

\emph{Varying dimension.}
We construct six representations of varying dimension from $2$ up to $6$. Of the two representations of dimension $d=6$, one is \hls. Fig.~\ref{fig:all}(\emph{left}) shows that in this case, \linucb with the \hls representation outperforms any non-\hls representation, even if they have smaller dimension. 
This property is inherited by \algo, which performs better than \linucb with non-\hls representations even of much smaller dimension $2$. 

\emph{Mixing representations.}
We construct six representations of the same dimension $d=6$, \emph{none of which} is \hls. However, they are constructed so that together they satisfy the weaker mixed-\hls assumption (Def.~\ref{def:hls.mix}). Fig.~\ref{fig:all}(\emph{middle left})  shows that, as predicted by Thm.~\ref{thm:algo.mix.regret.constant}, \algo leverages different representations in different context-action regions and it thus performs significantly better than any \linucb using non-\hls representations.
The superiority of \algo w.r.t.\ the model-selection baselines is evident in this case (Fig.~\ref{fig:all}(\emph{middle right}) ), since only \algo is able to mix representations, whereas model-selection algorithms target the best in a set of ``bad'' representations. Additional experiments in App.~\ref{app:experiments} confirm that \algo consistently outperforms all model-selection algorithms.

\textbf{Jester Dataset.} In the last experiment, we extract multiple linear representations from the Jester dataset \citep{goldberg2001eigentaste}, which consists of joke ratings in a continuous range from $-10$ to $10$ for a total of $100$ jokes and 73421 users. For a subset of 40 jokes and 19181 users rating all these 40 jokes, we build a linear contextual problem as follows. First, we fit a $32\times 32$ neural network to predict the ratings from features extracted via a low-rank factorization of the full matrix. Then, we take the last layer of the network as our ``ground truth'' linear model and fit multiple smaller networks to clone its predictions, while making sure that the resulting misspecification is small. We thus obtain $7$ representations with different dimensions among which, interestingly, we find that $6$ are \hls. Figure~\ref{fig:all}(\emph{right})  reports the comparison between \algo using all representations and \linucb with each single representation on a log-scale. Notably, the ability of \algo to mix representations makes it perform better than the best candidate, while transitioning to constant regret much sooner. Finally, the fact that \hls representations arise so ``naturally'' raises the question of whether this is a more general pattern in context-action features learned from data.

\section{Conclusion}
We provided a complete characterization of ``good'' realizable representations for \linucb, ranging from existence to a sufficient and necessary condition to achieve problem-dependent constant regret. We introduced \algo, a novel algorithm that, given a set of realizable linear representations, is able to adapt to the best one and even leverage their combination to achieve constant regret under the milder mixed-\hls condition.
While we have focused on \linucb, other algorithms (e.g., LinTS~\citep{AbeilleL17}) as well as other settings (e.g., low-rank RL~\citep{jin2020provably}) 
may also benefit from \hls-like assumptions.
We have mentioned an approach for eliminating misspecified representations, but a non-trivial trade-off may exist between the level of misspecification and the goodness of the representation. A slightly imprecise but very informative representation may be preferable to most bad realizable ones. Finally, we believe that moving from selection to representation learning --e.g., provided a class of features such as a neural network-- is an important direction both from a theoretical and practical perspective.

\bibliography{bibliography}
\bibliographystyle{plainnat}

\clearpage
\onecolumn
\appendix

\addcontentsline{toc}{section}{Appendix} 

\part{Appendix}
\parttoc

\section{Comparison to Related Work} \label{app:relatedwork}

In this section, we provide a more detailed review of the literature of model-selection in contextual bandits.
Note that this literature has mainly focused on analyzing the minimax regime. Studying whether it is possible and/or how to leverage a problem-dependent analysis is outside the scope of this work.
We thus present a review of the theoretical results of these algorithms, once applied to the representation selection problem considered in this paper.

In general, a model-selection algorithm has access to a set of $M$ base contextual-bandit algorithms (in the following, simply \emph{bases}), and, at each step, it selects a base and plays the corresponding action.
Our bases are instances of \linucb with different (realizable) representations.

One of the representative algorithms in model selection is \corral~\citep{agarwal2017corral}. However, the stability conditions required by \corral are not satisfied by \linucb. For this reason, we consider the stochastic version in~\citep{pacchiano2020stochcorral}, which can be applied to any base algorithm after applying a \emph{smoothing wrapper}.
This approach requires the knowledge of an upper bound to the regret of the best base algorithm.
For simplicity, we consider the $\wt{\mathcal{O}}(\sqrt{n})$ worst-case bound~\citep[e.g.,][Thm. 3]{abbasi2011improved}. Let $c(\delta)$ be the leading constant from that bound.
Neglecting logarithmic terms and setting the initial learning rate of \corral to $\eta=\sqrt{M/n}/c(\delta)$, according to Theorem 5.3. from~\citet{pacchiano2020stochcorral}:
\begin{equation*}
	R_n \le \widetilde{\mathcal{O}}\left((1+c(\delta))\sqrt{Mn} + \frac{\sqrt{Mn}}{c(\delta)}\right).
\end{equation*}
\citet{pacchiano2020stochcorral} showed a minimax lower bound $\Omega(T)$ for adversarial masters. Despite being a minimax result, whether the inherent $\sqrt{T}$ of adversarial masters can be shaved off through a refined analysis is an open question, as well as how to leveraged more adaptive regret bounds. 
Another major downside of \corral is the $\sqrt{M}$ dependence on the number of representations, compared to the logarithmic dependence of \algo. Clearly, these downsides are compensated by the more general applicability of these algorithms. See~\citep{arora2020corralling,foster2020misspecified} for examples of usage of corralling techniques.
\citet{pacchiano2020stochcorral} also propose to use \expthree as a master algorithm in place of \corral, using the same smoothing wrapper for base algorithms and the same regret oracle. By setting the explicit-exploration parameter of \expthree to $p=n^{-1/3}M^{-2/3}c(\delta)^{2/3}$, according to Theorem 5.3. from~\citet{pacchiano2020stochcorral}: $R_n \le \wt{\mathcal{O}}\left(\sqrt{Mn} + M^{1/3}n^{2/3}c(\delta)^{2/3}\right)$.
The same considerations made for \corral apply to this case.

Recently, techniques based on the idea of regret balancing has been used for model selection~\citep{abbasiyadkori2020regret,pacchiano2020regret}.
Denote by $\mathcal{H}_{ti} = (x_k,a_k,r_k)_{k\in N_i(t)}$ as the history associated to base algorithm $i$, i.e., restricted to rounds where $i$ was selected.
The Regret Balancing algorithm by~\citet{abbasiyadkori2020regret} (\regbal for short) requires access to a (high-probability, possibly data dependent) upper bound $u:\mathcal{H}_{ti^\star}\mapsto\Reals$ on the regret of the best base algorithm $i^\star\in[M]$. The regret of \regbal is then, with high probability:
\begin{equation*}
	R_n \le M \max_{i\in[M]}u(\mathcal{H}_{ni}),
\end{equation*}
that is, bounded like the best base on the worst subjective history. While this bound is very implicit, it shows: \emph{i)} \regbal pays a linear dependence on the number of base algorithms; \emph{ii)} even using a problem dependent bound, the final result will scale as the worse base algorithm. 

The second regret balancing approach is by~\citep{pacchiano2020regret}. Their Regret Balancing and Elimination by~\citet{pacchiano2020regret} (\regbalelim for short)  requires an upper bound $u_i: \mathcal{H}_i\mapsto\Reals$ on the regret of \emph{each} base algorithm (can be different for each algorithm). Their algorithms is able to eliminate the bases for which $u_i$ is incorrectly specified, and compete with the best \emph{correct} upper bound. Their algorithm improves w.r.t.\ \regbal and, by simplifying a lot, may allow to scale as the regret of the best base algorithm. However, \emph{i)} the dependence on $M$ is still linear and \emph{ii)} whether the idea of regret balancing will allow to break the barrier of logarithmic regret is an open question. As shown in their paper (e.g., Table 1), even if the regret upper bound is $O(n^{\beta})$ with $\beta < 1/2$, their bound has a (minimax) $O(\sqrt{n})$ dependence.

Finally, we would like to mention also the literature about learning with expert advice.
In particular, \expfour~\citep{neu2015explore} is an algorithm for bandits with expert advice, itself an improvement over the original \textsc{EXP4}~\citep{auer2002nonstochastic}. At each step, each expert prescribes an action (more in general, a distribution over actions) based on the whole history. \expfour observes the prescriptions of all the experts and decides which action to play. In our case, the experts are the bases. Similarly to \algo, \expfour is not strictly a model-selection algorithm since the action it plays may be different from the one prescribed by any base. 
In our notation, the (worst-case, high-probability) regret bound from~\citep[][Thm. 2]{neu2015explore} is:
\begin{equation*}
	R_n \le 2\sqrt{2Kn\log M} + \left(\sqrt{\frac{2Kn}{\ln n}}+1\right)\ln(2/\delta) + u_n,
\end{equation*}
where $u_n$ is an upper bound on the regret of \linucb with the best representation. Even in the case the best representation is \hls (i.e., $u_n$ is a constant), the regret bound is still $\widetilde{\mathcal{O}}(\sqrt{n})$. Moreover, \expfour has an explicit dependence on the number of actions $K$. An upside is the logarithmic dependence on the number of representations $M$.


\section{Analysis of Several Diversity Conditions}

\subsection{Categorization}
\label{app:rep.structures}
In this Section, we provide a complete proof of Lemma~\ref{prop:rep.asm.properties}.
The (counter)examples we provide to prove that the intersections depicted in Figure~\ref{fig:rep.asm.classification} are nonempty, we believe, can also improve understanding of the diversity conditions. 

In comparing representations, we fix a contextual problem with reward $\mu$ and context distribution $\rho$, and a dimension $d\ge 1$. We use set notation for conciseness.\footnote{For instance, we use \hls to denote the set of $d$-dimensional realizable representations that are \hls for the given contextual problem.}
First, we provide in Table~\ref{tab:diversity.span} equivalent definitions, in terms of the span of particular sets of features, of the properties defined in Figure~\ref{tab:diversity.conditions}. Without loss of generality, we assume that the context distribution $\rho$ is full-support, i.e., $\supp(\rho)=\mathcal{X}$. Otherwise, it is enough to replace $\mathcal{X}$ with $\supp(\rho)$ in the definitions. 

\begin{table}[h]
	\centering
	\begin{tabular}{|c|c|}
		\hline
		Name & Definition \\
		\hline
		\hline
		Non-redundant & $\spann\left\{\phi(x,a)\mid x\in\mathcal{X},a\in\mathcal{A}\right\} = \Reals^{d}$ \\
		\hline
		\cmb & $\forall a$, $\spann\left\{\phi(x,a)\mid x\in\mathcal{X}\right\} = \Reals^{d}$ \\
		\hline
		\bbk & $\forall a, \boldsymbol{u}\in\Reals^{d}$, $\spann\left\{\phi(x,a)\mid x\in\mathcal{X},\phi(x,a)^\transp \boldsymbol{u}\ge 0\right\} = \Reals^{d}$ \\ 
		\hline
		\hls & $\spann\left\{\phi^{\star}(x)\mid x\in\mathcal{X}\right\} = \Reals^{d}$ \\
		\hline
		\wys & $\forall a$, $\spann\left\{\phi(x,a)\mid x\in\X^\star(a)\right\} = \Reals^{d}$ \\
		\hline
		\hline
	\end{tabular}
	\caption{Diversity conditions in terms of the span of particular sets of features.}
	\label{tab:diversity.span}
\end{table}
	
The examples we provide are always $2$-dimensional representations with a corresponding parameter $\theta^{\star}=[1,1]$, and refer to the contextual problem having the resulting reward function and uniform context distribution. Optimal features are \underline{underlined} for convenience.
 
\paragraph{$\cmb \subset \text{non-redundant}$.}
For any $a\in\mathcal{A}$, $\spann\left\{\phi(x,a)\mid x\in\mathcal{X}\right\}\subseteq \spann\left\{\phi(x,a')\mid x\in\mathcal{X},a'\in\mathcal{A}\right\}$. In words, if the features of a single arm span $\Reals^d$, so does the whole set of features. 
Not all non-redundant representations are \cmb, as testified by the following representation:
\begin{align*}
&\underline{\phi(x_1,a_1)} = [1,1] &&\color{darkred}\phi(x_1,a_2)=[\nicefrac{1}{2},\nicefrac{1}{2}]\\
&\phi(x_2,a_1) = [0,1] &&\color{darkred}\underline{\phi(x_2,a_2)}=[1,1]
\end{align*}
which is non-redundant since all features together span $\Reals^2$, but not \cmb since those of arm $a_2$ alone fail to do so.

\paragraph{$\hls \subset \text{Non-redundant}$.}
Since $\phi^{\star}(x)=\phi^{\star}(x,a_{x}^{\star})$ for some $a^{\star}_x\in\mathcal{A}$, $\spann\left\{\phi^{\star}(x)\mid x\in\mathcal{X}\right\}\subseteq \spann\left\{\phi(x,a')\mid x\in\mathcal{X},a'\in\mathcal{A}\right\}$. In words, if optimal features span $\Reals^d$, so does the whole set of features.
Not all non-redundant representations are \hls, as testified by the following representation:
\begin{align*}
&\color{darkred}\underline{\phi(x_1,a_1)} = [1,1] &&\phi(x_1,a_2)=[\nicefrac{1}{2},\nicefrac{1}{2}]\\
&\phi(x_2,a_1) = [0,1] &&\color{darkred}\underline{\phi(x_2,a_2)}=[1,1]
\end{align*}
which is non-redundant since all features together span $\Reals^2$, but not \hls since optimal features alone fail to do so.

\paragraph{$\cmb\nsubseteq \hls$.}
The following representation is \cmb but not \hls:
\begin{align}
&\color{darkred}\underline{\phi(x_1,a_1)} = [1,1] &&\phi(x_1,a_2)=[1,0]\nonumber\\
&\phi(x_2,a_1) = [0,1] &&\color{darkred}\underline{\phi(x_2,a_2)}=[1,1]\label{eq:rep.cmb}
\end{align}
since the features of each arm taken alone span $\Reals^2$, but optimal features fail to do so.

\paragraph{$\hls\nsubseteq \cmb$.}
The following representation is \hls but not \cmb:
\begin{align*}
&\underline{\phi(x_1,a_1)} = [2,0] &&\color{darkred}\phi(x_1,a_2)=[\nicefrac{1}{2},\nicefrac{1}{2}]\\
&\underline{\phi(x_2,a_1)} = [0,2] &&\color{darkred}\phi(x_2,a_2)=[\nicefrac{1}{2},\nicefrac{1}{2}]
\end{align*}
since optimal features span $\Reals^2$, but the features of arm $a_2$ alone fail to do so. 

\paragraph{$\wys \subset \cmb\cap\hls$.}
Assume $\phi$ is \wys. For all $a$, since $\X^\star(a)\subseteq \X$, $\spann\{\phi(x,a)|x\in\X^\star(a)\}\subseteq\spann\{\phi(x,a)|x\in\X\}$. Hence $\phi$ is \cmb. If $x\in\X^\star(a)$, then $\phi(x,a)=\phi^\star(x)$. So, for any $a$, $\spann\{\phi(x,a)|x\in\X^\star(a)\}\subseteq\spann\{\phi^\star(x)|x\in\X\}$. Hence $\phi$ is \hls. The inclusion is strict, as testified by the following representation, which is both \hls and \cmb:
\begin{align*}
&\color{darkred}\underline{\phi(x_1,a_1)} = [2,0] &&\phi(x_1,a_2)=[1,0]\\
&\underline{\phi(x_2,a_1)} = [0,2] &&\phi(x_2,a_2)=[0,1]
\end{align*}
but is not \wys since $a_1$ is only optimal for $x_1$, and $\phi(x_1,a_1)$ cannot span $\Reals^2$ alone.

\paragraph{$\cmb\cap\hls\neq\emptyset$.}
See the previous example.

\paragraph{$\bbk\subset\cmb$.}
For all action $a$ and $\boldsymbol{u}\in\Reals^d$, $\spann\left\{\phi(x,a)\mid x\in\mathcal{X},\phi(x,a)^\transp \boldsymbol{u}\ge 0\right\}\subseteq \spann\left\{\phi(x,a)\mid x\in\mathcal{X}\right\}$. So if a representation is \bbk, it is also \cmb. The converse is not true in general: none of the examples we have provided so far are \bbk, since all features lie in the first quadrant of $\Reals^2$, hence $\spann\left\{\phi(x,a)\mid x\in\mathcal{X},\phi(x,a)^\transp\boldsymbol{u} \ge 0\right\}=\emptyset$ for $\boldsymbol{u}=[-1,-1]$. In particular, \eqref{eq:rep.cmb} is \cmb but not \bbk. This, together with $\hls\nsubseteq\cmb$, also shows that $\hls\nsubseteq\bbk$.

\paragraph{$\bbk\nsubseteq\hls$}
The following representation:
\begin{align}
&\color{darkred}\underline{\phi(x_1,a_1)} = [2,0] &&\phi(x_1,a_2)=[0,1]\nonumber\\
&\phi(x_2,a_1) = [0,1] &&\color{darkred}\underline{\phi(x_2,a_2)}=[2,0]\nonumber\\
&\color{darkred}\underline{\phi(x_2,a_1)} = [-1,0] &&\phi(x_2,a_2)=[0,-2]\nonumber\\
&\phi(x_2,a_1) = [0,-2] &&\color{darkred}\underline{\phi(x_2,a_2)}=[-1,0]\label{eq:rep.bbk.nohls}
\end{align}
is \bbk but not \hls. To show that it is not \hls, we just notice that optimal features fail to span $\Reals^2$. To show that it is \bbk, we can easily check that $\{\phi(x,a)\mid x\in\mathcal{X},\phi(x,a)^\transp\boldsymbol{u} \ge 0\}$ spans $\Reals^2$ for both actions and $\boldsymbol{u}\in\{[1,1],[-1,1],[-1,-1],[1,-1]\}$. Any other vector can be obtained as $A\boldsymbol{u}$ from one of these four vectors, where $A$ is a p.s.d. diagonal matrix. Since all features $\phi$ are aligned with the axes of $\Reals^2$, $\phi^\transp\boldsymbol{u} = \phi_i\times u_i$, for some index $i\in\{1,2\}$. Similarly, $\Reals^d$, $\phi(x,a)^\transp A\boldsymbol{u} = \phi_i\times A_{ii}\times u_i$. Since $A_{ii}\ge 0$, non-negativity of the former guarantees non-negativity of the latter. This example also shows that \bbk is not empty in general.

\paragraph{$\hls\cap\bbk\nsubseteq\wys$.}
The following representation:
\begin{align}
&\underline{\phi(x_1,a_1)} = [2,0] &&\phi(x_1,a_2)=[1,0]\nonumber\\
&\underline{\phi(x_2,a_1)} = [0,2] &&\phi(x_2,a_2)=[0,1]\nonumber\\
&\underline{\phi(x_2,a_1)} = [-1,0] &&\phi(x_2,a_2)=[-2,0]\nonumber\\
&\underline{\phi(x_2,a_1)} = [0,-1] &&\phi(x_2,a_2)=[0,-2]\label{eq:rep.bbk.nowys}
\end{align}
is \hls and \bbk but not \wys. To show that it is \bbk, we can use the same argument used for~\eqref{eq:rep.bbk.nohls}. To show that it is not \wys, we notice that $a_2$ is never optimal, so $\spann\left\{\phi(x,a_2)\mid x\in\X^\star(a_2)\right\}=\emptyset$.
This example also shows that $\hls\cap\bbk$ is not empty in general, and that $\bbk\nsubseteq\wys$.

\paragraph{$\wys\nsubseteq\bbk$}
The following representation:
\begin{align*}
&\underline{\phi(x_1,a_1)} = [2,0] &&\phi(x_1,a_2)=[1,0]\\
&\underline{\phi(x_2,a_1)} = [0,2] &&\phi(x_2,a_2)=[0,1]\\
&\phi(x_2,a_1) = [1,0] &&\underline{\phi(x_2,a_2)}=[2,0]\\
&\phi(x_2,a_1) = [0,1] &&\underline{\phi(x_2,a_2)}=[0,2]
\end{align*}
is \wys since each arm admits two orthogonal optimal features, but not \bbk since all the features lie in the first quadrant of $\Reals^2$. This example also shows that \wys is not empty in general.

\paragraph{$\bbk\cap\hls\cap\wys\neq \emptyset$.}
The following representation:
\begin{align*}
&\underline{\phi(x_1,a_1)} = [2,0] &&\phi(x_1,a_2)=[1,0]\\
&\underline{\phi(x_2,a_1)} = [0,2] &&\phi(x_2,a_2)=[0,1]\\
&\phi(x_2,a_1) = [-2,0] &&\underline{\phi(x_2,a_2)}=[-1,0]\\
&\phi(x_2,a_1) = [0,-2] &&\underline{\phi(x_2,a_2)}=[0,-1]
\end{align*}
is \hls, \bbk and \wys. This can be shown using the arguments employed for the previous examples.

\subsection{Existence of \hls Representations} \label{app:rep.structures.hls.existence}
In this section, we prove a slightly more general version of Theorem~\ref{prop:existence.hls}.

We first prove the existence of an \hls representation in the case of finite contexts.
\begin{lemma}\label{lem:existence.finite}
For any dimension $d\ge 1$, any contextual problem such that:
\begin{itemize}
	\item $\X$ is finite,
	\item $|\supp(\rho)|\ge d$, and
	\item there exists $x\in\supp(\rho)$ with $\mu^\star(x)\neq 0$,
\end{itemize} 
admits a $d$-dimensional realizable \hls representation.
\end{lemma}
\begin{proof}
	Consider a contextual problem $\mathcal{P}$ with the properties stated above. Without loss of generality, $\rho(x)>0$ for $x\in\{x_1,\dots,x_d\}$, and $\mu^\star(x_1)\neq 0$.
	
	In this case, to prove that a $d$-dimensional realizable linear representation $\phi$ is \hls, it suffices to show that $\spann\{\phi^\star(x_i)\}_{i=1}^d=\Reals^d$. Indeed:
	\begin{equation}
	\lambda_{\min}\left(\EV_{x\sim\rho}[\phi^{\star}(x)\phi^{\star}(x)^\transp]\right) \ge \lambda_{\min}\left(\sum_{i=1}^d\rho(x_i)\phi^{\star}(x_i)\phi^{\star}(x_i)^\transp\right) \ge \rho_{\min}\times\lambda_{\min}\left(\sum_{i=1}^d\phi^{\star}(x_i)\phi^{\star}(x_i)^\transp\right),\label{eq:existence.finite.1}
	\end{equation}
	where $\rho_{\min}\coloneqq\min_{i\in[d]}\rho(x_i)>0$. If $\phi^\star(x_1),\dots\phi^\star(x_d)$ span $\Reals^d$, the matrix $\Phi$ having these vectors as rows is full rank, hence so is $\sum_{i=1}^d\phi^{\star}(x_i)\phi^{\star}(x_i)^\transp=\Phi^\transp\Phi$, and the minimum eigenvalue in~\eqref{eq:existence.finite.1} is positive.
	
	We first consider the case $d=1$. A realizable one-dimensional linear representation for $\mathcal{P}$ is $\phi(x,a)=\mu(x,a)$ for all $x,a$, with $\theta^\star=1$. This is \hls since $\phi^\star(x_1)=\mu^\star(x_1)\neq 0$ spans $\Reals$. 
	
	Now consider the case $d>1$. Let $L$ be a constant such that $L>\max_{i\in[d]}|\mu^\star(x_i)|$. Notice that $L>|\mu^\star(x_1)|>0$. Within the scope of this proof, we use $\phi_i$ to denote the $i$-th element of vector $\phi$. Consider the linear representation with $\phi_1(x,a)=\mu(x,a)$ and $\phi_i(x,a)=L\times \indi{x=x_i}$ for $i=2,\dots,d$ and all $x,a$. This is a $d$-dimensional linear representation for $\mathcal{P}$ with $\theta^\star=[1,0\dots,0]^\transp$. The optimal features of the first $d$ contexts, taken as rows, form the following matrix:
	\begin{equation*}
		\begin{bmatrix}
			\mu^*(x_1) & 0 & 0 &\dots & 0 \\
			\mu^*(x_2) & L & 0 &\dots & 0 \\
			\mu^*(x_3) & 0 & L &\dots & 0 \\
			\vdots & \vdots & \vdots &\ddots & \vdots \\
			\mu^*(x_d) & 0 & 0 & \dots & L
		\end{bmatrix}
	\end{equation*}
	which is strictly diagonally dominant, hence full-rank. Equivalently, $\phi^\star(x_1),\dots,\phi^\star(x_d)$ span $\Reals^d$, completing the proof.
\end{proof}

Next, we prove a similar result for the continuous case.

\begin{lemma}\label{lem:existence.borel}
	For any dimension $d\ge 1$, any contextual problem such that:
	\begin{itemize}
		\item $\X$ is a standard Borel space with measure $\rho$ and metric $\partial$,
		\item $x\mapsto\mu^\star(x)$ is continuous on $\X$, and 
		\item there exist $x_1,\dots,x_d\in\supp(\rho)$ with $\mu^\star(x_i)\neq 0$ for $i=1,\dots,d$,
	\end{itemize} 
	admits a $d$-dimensional realizable \hls representation.
\end{lemma}
\begin{proof}
	Consider a contextual problem $\mathcal{P}$ with the properties stated above. Without loss of generality, let $x_1,\dots,x_d\in\supp(\rho)$ and $|\mu^\star(x_i)|>\epsilon_i>0$ for $i=1,\dots,d$.
	
	We first consider the case $d=1$.
	Define the open interval $I_1\coloneqq(\epsilon_1,\infty)$.
	Since $\mu^\star$ is continuous, $|\mu^\star|$ is also continuous. Hence, the inverse image $E_1=|\mu^\star|^{-1}(I_1)$ is an open subset of $\X$. In particular, since $x_1\in E_1$, $E_1$ is an open neighborhood of $x_1$. 
	Since $x_1\in\supp(\rho)$, by definition of support\footnote{In a Borel space, the support of the measure is defined as the smallest set for which every open neighbourhood of every point in the set has positive measure.
	}, $\rho(E_1)>0$. Moreover, by definition of $E_1$, $|\mu^\star(x)|>\epsilon_1$ for all $x\in E_1$. Consider the representation $\phi(x,a)=\mu(x,a)$ for all $x,a$. This is a one-dimensional realizable linear representation for $\mathcal{P}$ with $\theta^\star=1$. We have that:
	\begin{equation*}
		\lambda_{\min}\left(\EV_{x\sim\rho}[\phi^{\star}(x)\phi^{\star}(x)^\transp]\right) = \int_{\mathcal{X}} \mu^\star(x)^2 \mathrm{d}\rho 
		\ge \int_{E_1}\mu^\star(x)^2\mathrm{d}\rho
		> \rho(E_1)\times \epsilon_1^2 > 0,
	\end{equation*}
	and the representation is \hls.
	
	Now consider the case $d>1$.
	For $i=1,\dots,d$, define the open interval $I_i\coloneqq(\epsilon_i,\infty)$.  As in the scalar case, $E_i=|\mu^\star|^{-1}(I_i)$ is an open neighborhood of $x_i$ of positive measure such that $\mu^\star(x)>\epsilon_i$ for all $x\in E_i$. However, these neighborhoods may not be disjoint. 
	
	To fix this, we will use the fact that $\X$ is a metric space. Let $r_0=\min\{\partial(x_i,x_j)\mid i,j=1,\dots,d\}$ be the minimum distance between any of the $d$ points.
	Also, let $r_i$ be the radius of the smallest open ball contained in $E_i$. Let $r=\min\{r_0,r_1,\dots,r_d\}$. Finally, let $F_i=\mathcal{B}_{x_i,r}=\{x\in\X|\partial(x,x_i)<r\}$ be the open ball of radius $r$ centered in $x_i$.
	The $F_1,\dots,F_d$ defined in this way are \emph{disjoint} open neighborhoods of $x_1,\dots,x_d$ of positive measure such that, for all $x\in F_i$, $|\mu^\star(x)|>\epsilon_i$.
	
	Let $L$ be a constant such that $L>\max_{i\in[d]}\epsilon_i$. Notice that $L>0$. Within the scope of this proof, we use $\phi_i$ to denote the $i$-th element of vector $\phi$. Consider the linear representation with $\phi_1(x,a)=\mu(x,a)$ and $\phi_i(x,a)=L\times \indi{x\in F_i}$ for $i=2,\dots,d$ and all $x,a$. This is a $d$-dimensional realizable linear representation for $\mathcal{P}$, with $\theta^\star=[1,0\dots,0]^\transp$. 
	Notice that, for all $x\in F_i$, $\norm{\phi{^\star}(x)}^2 = \mu^\star(x)^2+L^2>\epsilon_i^2+L^2$. 
	For $i=1,\dots,d$, define vector $\boldsymbol{u}(i)$ as follows: $\boldsymbol{u}_1(i)=\epsilon_i$ and $\boldsymbol{u}_j(i)=L\times \indi{x\in F_j}$ for $j=1,\dots,d$.  
	For any vector $\boldsymbol{u}\in\Reals^d$, the only nonzero eigenvalue of $\boldsymbol{u}\boldsymbol{u}^\transp$ is $\norm{\boldsymbol{u}}^2$. 
	So if $x\in F_i$, $\phi^\star(x)\phi^\star(x)^\transp\succ \boldsymbol{u}(i)\boldsymbol{u}(i)^\transp$. Finally, we have that:
	\begin{align}
	\lambda_{\min}\left(\EV_{x\sim\rho}[\phi^{\star}(x)\phi^{\star}(x)^\transp]\right) &= \lambda_{\min}\left(\int_{\mathcal{X}} \phi^\star(x)\phi^\star(x)^\transp \mathrm{d}\rho\right)\nonumber\\ 
	&\ge \lambda_{\min}\left(\sum_{i=1}^d\int_{F_i} \phi^\star(x)\phi^\star(x)^\transp \mathrm{d}\rho\right)\label{eq:existence.borel.disjoint}\\
	&\ge  \lambda_{\min}\left(\sum_{i=1}^d\int_{F_i} \boldsymbol{u}(i)\boldsymbol{u}(i)^\transp \mathrm{d}\rho\right) \nonumber\\
	&= \lambda_{\min}\left(\sum_{i=1}^d\rho(F_i)\boldsymbol{u}(i)\boldsymbol{u}(i)^\transp\right) \nonumber\\
	&\ge \rho_{\min}\times\lambda_{\min}\left(\sum_{i=1}^d\boldsymbol{u}(i)\boldsymbol{u}(i)^\transp\right),\label{eq:existence.borel.final}
	\end{align}
	where $\rho_{\min}=\min_{i\in[d]}\rho(F_i)>0$. In~\eqref{eq:existence.borel.disjoint}, we used the fact that the $F_i$ are disjoint. It remains to show that the matrix $U$ obtained by taking $\boldsymbol{u}(1),\dots,\boldsymbol{u}(i)$ as rows is full-rank. But this yields the following matrix:
	\begin{equation}
		\begin{bmatrix}
		\epsilon_1 & 0 & 0 &\dots & 0 \\
		\epsilon_2 & L & 0 &\dots & 0 \\
		\epsilon_3 & 0 & L &\dots & 0 \\
		\vdots & \vdots & \vdots &\ddots & \vdots \\
		\epsilon_d & 0 & 0 & \dots & L
		\end{bmatrix}
	\end{equation}
	which, by definition of $L$, is strictly diagonally dominant, hence full-rank. So, the matrix $U^\transp U=\sum_{i=1}^d\boldsymbol{u}(i)\boldsymbol{u}(i)^\transp$ in~\eqref{eq:existence.borel.final} also is full-rank, its minimum eigenvalue is positive, and $\phi$ is \hls.
\end{proof}

The following lemma implies that, once there exist an \hls representation for a problem, there exist infinite equivalent \hls representations (at least one for each $d\times d$ invertible matrix):

\begin{lemma}\label{lem:hls.properties.linear}
	Any invertible linear transformation of an \hls representation yields an \hls representation.
\end{lemma}
\begin{proof}
	Let $\phi:\mathcal{X}\times[K]\to \Reals^d$ be a realizable linear representation for a problem $\mathcal{P}$ with context space $\mathcal{X}$ and $K$ arms, with corresponding parameter $\theta^\star$. Any invertible matrix $A\in\Reals^{d\times d}$ defines an invertible linear transformation $T_A$ in the following sense: 
	\begin{equation}
		T_{A}(\phi)(x,a) = A^\transp\phi(x,a).
	\end{equation}
	First, note that $\wt{\phi}=T_{A}(\phi)$ is still a ($d$-dimensional) realizable linear representation for the same problem with parameter $\wt{\theta}^\star=A^{-1}\theta^\star$, since for all $x,a$:
	\begin{equation}
		\mu(x,a) = \phi(x,a)^\transp\theta^\star = \phi(x,a)^\transp AA^{-1}\theta^\star = \wt{\phi}(x,a)^\transp\wt{\theta}^\star.
	\end{equation}
	Now note that:
	\begin{align*}
		\lambda_{\min}\left(\EV_{x\sim\rho}[\wt{\phi}^{\star}(x)\wt{\phi}^{\star}(x)^\transp]\right) &=
		\lambda_{\min}\left(\EV_{x\sim\rho}[A^\transp\phi^{\star}(x)\phi^{\star}(x)^\transp A]\right).
	\end{align*}
	If the matrix in the LHS is full-rank, so is the matrix in the RHS, because $A$ is full-rank and the product of full-rank \emph{square} matrices is always full-rank. Hence if $\phi$ is HLS, so is $\wt{\phi}$.
\end{proof}

These results are summarized in Lemma~\ref{prop:existence.hls} in the paper. For simplicity, we assumed that $\mu^\star(x)\neq 0$ for all $x\in\X$. In the continuous case, we also assumed that $\rho$ is non-degenerate, which means $\supp(\rho)$ contains an open ball. This two assumptions together guarantee there exist $d$ (in fact, infinite) points in the support of $\rho$ with nonzero optimal reward.

Finally, we can remove all assumptions on nonzero optimal rewards by introducing a concept of equivalence between contextual problems. We say two reward functions $\mu$ and $\mu'$ are equivalent provided $\mu'(x,a)=\mu(x,a)+C$ for all $x,a$ and some $C\in\Reals$. We say two contextual problems are equivalent if they have the same context space, the same arm set, the same context distribution, and equivalent reward functions. So, if the assumptions on nonzero optimal rewards from Lemma~\ref{prop:existence.hls},~\ref{lem:existence.finite}, or~\ref{lem:existence.borel} are not satisfied, we can always find an \emph{equivalent} problem that admits an \hls representation by adding an appropriate constant offset to all rewards.

\subsection{Additional Properties of \hls Representations}
In this section, we prove additional properties of \hls representations that will be useful to construct illustrative representation-selection problems (App.~\ref{app.exp.toy}).

\begin{lemma}\label{lem:hls.properties.derank}
	Under the same assumptions of Lemma~\ref{lem:existence.finite}, for every $d>1$, every contextual problem, assuming \emph{non-optimal} features span $\Reals^d$, admits an infinite number of $d$-dimensional non-redundant representations that are not \hls.
\end{lemma}
\begin{proof}
	Lemma~\ref{lem:existence.finite} shows that every contextual problem admits a $d$-dimensional \hls (hence, non-redundant) representation. We will show how to turn an \hls representation $\phi$ of dimension $d>1$ into an equivalent representation that is still non-redundant but not \hls, having $\rank(M_\phi)\le k<d$ for a $k$ of choice, where $M_\phi=\EV_{x\sim\rho}[\phi^\star(x)\phi^\star(x)^\transp]$.
	As shown in the proof of Lemma~\ref{lem:hls.properties.linear}, an invertible linear transformation does not affect the rank of $M_\phi$. So, one can easily obtain an infinite number of equivalent representations with the same property.  
	
	Let $\Phi^\star\in\Reals^{N\times d}$ be a matrix having as rows the optimal features, after removing contexts that are not in the support of $\rho$. Let $q=N-k+1$, where $k$ is the desired rank, and $\Phi^\star_{q}$ denotes $\Phi^\star$ with all the $d$ columns but only the first $q$ rows. We will only modify $\Phi^\star_q$, leaving all other features unchanged. Let $\mu^\star\in\Reals^N$ be the vector of optimal rewards, and $\mu^\star_q$ be the sub-array of its first $q$ elements. The modified features are:
	\begin{equation*}
		\wt{\Phi}^\star_q = \frac{\mu^\star_q(\mu^\star_q)^\transp}{\norm{\mu^\star_q}^2} \Phi^\star_q.
	\end{equation*}
	The new representation $\wt{\phi}$ obtained in this way is equivalent to $\phi$ (with $\wt{\theta}^\star=\theta^\star$) since:
	\begin{equation*}
		\wt{\Phi}^\star_q \theta^\star = \frac{\mu^\star_q(\mu^\star_q)^\transp}{\norm{\mu^\star_q}^2} \Phi^\star_q\theta^\star = \frac{\mu^\star_q(\mu^\star_q)^\transp\mu^\star_q}{\norm{\mu^\star_q}^2} = \mu^\star_q = \Phi^\star_q\theta^\star,
	\end{equation*}
	and all other features are unchanged. However, note that $\wt{\Phi}^\star_q$ has rank one, since it is obtained by multiplication with the rank-one matrix $\mu^\star_q(\mu^\star_q)^\transp$. This means that the first $q=N-k+1$ optimal features of $\wt{\phi}$ are all linearly dependent. With the remaining $k-1$ features, the rank of $\wt{\Phi}^\star$ is at most $k<d$, so $\wt{\phi}$ is not \hls. Since we only modified optimal features, assuming non-optimal features of $\phi$ span $\Reals^d$, $\wt{\phi}$ is still non-redundant.
\end{proof}

\begin{lemma}\label{lem:hls.properties.random}
	Let $\Phi\in\Reals^{NK\times d}$ be a random matrix whose elements are sampled i.i.d. from a non-degenerate distribution (e.g., a standard normal). Let $\theta^\star\in\Reals^d$ be any vector and consider the contextual problem with context set $[N]$, action set $[K]$, uniform context distribution $\rho$, and $\mu:(x,a)\mapsto \Phi[xK + a]^\transp\theta^\star$, where $[\cdot]$ selects rows. Then, almost surely, the representation $\phi:(x,a)\mapsto\Phi[xK+a]$ is \hls for this problem.
\end{lemma}
\begin{proof}
	Let $\Phi^\star$ be the submatrix of $\Phi$ with only rows $xK+a$ such that $a$ is an optimal action for $x$. Notice that $\phi$ is \hls if and only if $\Phi^\star$ is full-rank. But the rows of $\Phi^\star$ are sampled i.i.d. from a non degenerate $d$-variate distribution (e.g., a $d$-variate Gaussian with positive variance in all directions). So, $\Phi^\star$ is full-rank almost surely~\citep{eaton1973non}.
\end{proof}

\section{Constant Regret with a \hls Representation}\label{app:hls.regret}

\subsection{Preliminary Results}\label{app:hls.regret.prelim}

\begin{proposition}[\citealp{abbasi2011improved}]\label{prop:oful.good}
	Consider any linear contextual bandit problem with noise standard deviation $\sigma$.
	Run \linucb with confidence parameter $\delta$, regularization parameter $\lambda$, and a $d_{\phi}$-dimensional realizable linear representation such that $\sup_{x,a}\|\phi(x,a)\|_2\leq L_{\phi}$,
	and $\|\theta^\star_{\phi}\|_2 \leq S_{\phi}$.
	The following \emph{good event} holds with probability at least $1-\delta$:
	\begin{equation}
		\mathcal{G}_{\phi}(\delta) \coloneqq \left\{
			\forall t\ge 1,
			\norm{\theta_{t\phi} - \theta^\star_{\phi}}_{V_{t\phi}} \le \beta_{t\phi}(\delta) 
		\right\},
	\end{equation}
	where
	\begin{align}
	\beta_{t\phi}(\delta) &:= \sigma\sqrt{2\ln\left(\frac{\det(V_{t\phi})^{1/2}\det(\lambda I_{d_{\phi}})^{-1/2}}{\delta}\right)} + \sqrt{\lambda} S_{\phi} \\
	&\le \sigma\sqrt{2\log(1/\delta)+d_{\phi}\log(1+(t-1)L_{\phi}^2/(\lambda d_{\phi}))} + \sqrt{\lambda}S_{\phi}.\label{eq:beta.bound}
	\end{align}
\end{proposition}

\begin{proposition}[\citealp{abbasi2011improved}]\label{prop:oful.instreg}
Under the same assumptions of Proposition~\ref{prop:oful.good}, assuming the good event $\mathcal{G}_{\phi}(\delta)$ holds, for all $t\ge 1$, the instantaneous regret of \linucb is bounded as:
\begin{equation*}
	r_{t} \le 2\beta_{t\phi}(\delta)\norm{\phi(x_t,a_t)}_{V_{t\phi}^{-1}}.
\end{equation*}
	
\end{proposition}

\paragraph{Proof of Proposition~\ref{prop:oful.log.regret}.}
\begin{proof}
	This is a variant of Theorem 5 by~\citet{abbasi2011improved} where we do not assume the existence of a unique optimal feature vector over $\X$.
	The proof is easier, and the result has a worse dependence on the feature dimension.
	
	The assumptions of this proposition are the same as Prop.~\ref{prop:oful.good}. 
	We will show that the logarithmic regret bound holds under the good event $\mathcal{G}_{\phi}(\delta)$, which in turn holds with probability at least $1-\delta$.
	Here and in the rest of the paper, we further assume that $\Delta_{\max}\ge1$, $S_\phi\ge 1, \lambda\ge 1$, and $\sigma\ge1$.
	These are technical assumptions that can be easily removed by properly clipping the constants in the regret bound.
	
	From Prop.~\ref{prop:oful.instreg} and $\beta_{t\phi}(\delta)\ge 1$:
	\begin{align*}
		r_t &\le 2\beta_{t\phi}(\delta)\norm{\phi(x_t,a_t)}_{V_{t\phi}^{-1}} \\
		&\le \min\left\{2\beta_{t\phi}(\delta)\norm{\phi(x_t,a_t)}_{V_{t\phi}^{-1}}, \Delta_{\max}\right\} \\
		&\le 2\Delta_{\max}\beta_{t\phi}(\delta)\min\left\{\norm{\phi(x_t,a_t)}_{V_{t\phi}^{-1}},1\right\}.
	\end{align*}
	From the fact that $\beta_{t\phi}(\delta)$ is increasing and the Elliptical Potential Lemma~\citep[e.g.,][Lemma 11]{abbasi2011improved}:
	\begin{align*}
		\sum_{k=1}^tr_k^2 
		&\le 4\Delta_{\max}^2\beta_{t\phi}(\delta)^2 \sum_{k=1}^t\min\left\{\norm{\phi(x_t,a_t)}_{V_{t\phi}^{-1}}^2,1\right\}\\
		&\le 8\Delta_{\max}^2\beta_{t\phi}(\delta)^2d_\phi\log(1+tL_\phi^2/(\lambda d_\phi)).
	\end{align*}
	Since either $r_t\ge\Delta$ or $r_t=0$:
	\begin{align*}
		R_t &= \sum_{k=1}^tr_k 
		\le \sum_{k=1}^t\frac{r_k^2}{\Delta} \le \frac{8\Delta_{\max}^2\beta_{t\phi}(\delta)^2d_\phi\log(1+tL_\phi^2/(\lambda d_\phi))}{\Delta} \\
		&\le \frac{8\Delta_{\max}^2\left(\sigma\sqrt{2\log(1/\delta)+d_{\phi}\log(1+(t-1)L_{\phi}^2/(\lambda d_{\phi}))} + \sqrt{\lambda}S_{\phi}\right)^2d_\phi\log(1+L_\phi^2/(\lambda d_\phi))}{\Delta} \\
		&\le \frac{32\Delta_{\max}^2\lambda S_\phi^2\sigma^2\left(2\log(1/\delta)+d_\phi \log(1+tL_\phi^2/(\lambda d_\phi))\right)^2}{\Delta}.
	\end{align*} 
\end{proof}

\begin{proposition}[\citealp{abbasi2011improved}]\label{prop:pulls.basic}
Under the same assumptions of Proposition~\ref{prop:oful.good}, assuming the good event $\mathcal{G}_{\phi}(\delta)$ holds, for all $t\ge 1$, the number of suboptimal pulls of \linucb up to time $t$ is at most:
\begin{equation*}
	g_{t\phi}(\delta)\coloneqq\frac{32\Delta_{\max}^2\lambda S_\phi^2\sigma^2\left(2\log(1/\delta)+d_\phi \log(1+tL_\phi^2/(\lambda d_\phi))\right)^2}{\Delta^2}.
\end{equation*}
\end{proposition}
\begin{proof}
	A similar bound can be found in the proof of Theorem 5 by~\citet{abbasi2011improved}. Again, we do not assume a unique optimal feature vector.
	
	Let $n_t$ be the number of suboptimal pulls up to time $t$. Since $R_t\ge n_t\Delta$:
	\begin{align*}
		n_t \le \frac{R_t}{\Delta}.
	\end{align*}
	The regret bound from Prop.~\ref{prop:oful.log.regret}, which holds under $\mathcal{G}_\phi(\delta)$, completes the proof.
\end{proof}

\begin{lemma}\label{lem:mineig.basic}
Under the same assumptions of Proposition~\ref{prop:oful.good}, assuming the good event $\mathcal{G}_{\phi}(\delta)$ holds, with probability $1-\delta$, for all $t\ge 1$:
\begin{equation}
\lambda_{\min}(V_{t+1,\phi}) \ge \lambda + t\lambda_{\phi,\hls} - 8L_\phi^2\sqrt{t\log(2d_\phi t/\delta)} - L_\phi^2 g_{t\phi}(\delta),
\end{equation}
where $g_{t\phi}(\delta)$ is from Prop.~\ref{prop:pulls.basic}.
\end{lemma}
\begin{proof}
	Since $V_{t+1,\phi}$ is symmetric:
	\begin{align}
		V_{t+1,\phi} 
		&=\lambda I_{d_\phi} + \sum_{k=1}^{t}\phi(x_k,a_k)\phi(x_k,a_k)^\transp \nonumber\\
		&= \lambda I_{d_\phi} + \sum_{k=1}^{t}\indi{a_k=a^\star_{x_k}}\phi(x_k,a_k)\phi(x_k,a_k)^\transp + \sum_{k=1}^{t}\indi{a_k\neq a^\star_{x_k}}\phi(x_k,a_k)\phi(x_k,a_k)^\transp \nonumber\\
		&\succeq \lambda I_{d_\phi} + \sum_{k=1}^{t}\indi{a_k=a^\star_{x_k}}\phi^\star(x_k)\phi^\star(x_k)^\transp \nonumber\\
		&=\lambda I_{d_\phi} + \sum_{k=1}^{t}\phi^\star(x_k)\phi^\star(x_k)^\transp
		-\sum_{k=1}^{t}\indi{a_k\neq a^\star_{x_k}}\phi^\star(x_k)\phi^\star(x_k)^\transp \nonumber\\
		&\succeq \lambda I_{d_\phi} + \sum_{k=1}^{t}\phi^\star(x_k)\phi^\star(x_k)^\transp
		-g_{t\phi}(\delta)L_\phi^2I_{d_\phi} \nonumber\\
		&= \lambda I_{d_\phi} +
		t\EV_{x\sim\rho}\left[\phi^\star(x)\phi^\star(x)^\transp\right] - 
		\sum_{k=1}^{t}\left(\EV_{x\sim\rho}\left[\phi^\star(x)\phi^\star(x)^\transp\right]-\phi^\star(x_k)\phi^\star(x_k)^\transp\right)
		-L_\phi^2g_{t\phi}(\delta)I_{d_\phi} \nonumber\\
		&=\lambda I_{d_\phi} +
		t\EV_{x\sim\rho}\left[\phi^\star(x)\phi^\star(x)^\transp\right] - 
		\sum_{k=1}^{t}X_k
		-L_\phi^2g_{t\phi}(\delta)I_{d_\phi},\label{eq:mineig.basic.matrix}
	\end{align}
	where $g_{t\phi}(\delta)$ is the upper bound on suboptimal pulls from Prop.~\ref{prop:pulls.basic}, which holds under the good event $\mathcal{G}_{\phi}(\delta)$, and $X_k=\EV_{x\sim\rho}\left[\phi^\star(x)\phi^\star(x)^\transp\right]-\phi^\star(x_k)\phi^\star(x_k)^\transp$. Since the matrix in~\eqref{eq:mineig.basic.matrix} is still symmetric, by definition of $\lambda_{\phi,\hls}$:
	\begin{align*}
		\lambda_{\min}(V_{t+1,\phi}) \ge \lambda + t\lambda_{\phi,\hls} -\lambda_{\max}\left(\sum_{k=1}^{t}X_k\right)- L_\phi^2g_{t\phi}(\delta).
	\end{align*}
	We bound the third term using a matrix Azuma inequality by~\citet{tropp2012user}. First, notice that $\EV_k[X_k] = 0$. Also, since $X_k$ is symmetric:\[X_k^2 \preceq \lambda_{\max}(X_k^2)I_{d_\phi} \preceq \norm{X_k}^2I_{d_\phi} \preceq 4L_\phi^4I_{d_\phi}.\]
	Hence, from Prop.~\ref{prop:mazuma}, with probability at least $1-\delta'_t$, for all $t\ge1$:\footnote{Notice that this is true regardless of $\mathcal{G}_\phi(\delta)$. This will be useful to bound expected regret in App.~\ref{app:hls.regret.expected}.}
	\begin{equation*}
		\lambda_{\max}\left(\sum_{k=1}^{t}X_k\right) \le 4L_\phi^2\sqrt{2t\log(d_\phi/\delta'_t)}.
	\end{equation*}
	We set $\delta'_t=\delta/(2t^2)$ and perform a union bound over time. Finally, with probability at least $1-\delta$, for all $t\ge1$:
	\begin{equation*}
		\lambda_{\max}\left(\sum_{k=1}^{t}X_k\right) \le 4L_\phi^2\sqrt{2t\log(2d_\phi t^2/\delta)}\le 8L_\phi^2\sqrt{t\log(2d_\phi t/\delta)}.
	\end{equation*}
\end{proof}

\subsection{High-Probability Regret Bound}
\begin{lemma}\label{lem:zeroregret.basic}
	Consider a contextual bandit problem with realizable linear representation $\phi$ satisfying the \hls condition. Assume $\Delta > 0$, $\max_{x,a}\|\phi(x,a)\|_2 \leq L_\phi$ and $\|\theta^\star_\phi\|_2 \leq S_\phi$. Then, assuming the good event $\mathcal{G}_\phi(\delta)$ holds, with probability at least $1-\delta$, there exists a (constant) time $\tau_\phi$ such that, for all $t\ge\tau_\phi$, the instantaneous regret of \linucb run with confidence parameter $\delta$ is $r_{t+1}=0$.
\end{lemma}
\begin{proof}
	From Prop.~\ref{prop:oful.instreg}, under the good event $\mathcal{G}_{\phi}(\delta)$:
	\begin{align*}
	r_{t+1} 
	&\le 2\beta_{t+1,\phi}(\delta)\norm{\phi(x_{t+1},a_{t+1})}_{V_{t+1,\phi}^{-1}} 
	\\
	&\le 2\beta_{t+1,\phi}(\delta)L_\phi\sqrt{\lambda_{\max}(V_{t+1,\phi}^{-1})} \\
	&\le 2\beta_{t+1,\phi}(\delta)\frac{L_\phi}{\sqrt{\lambda_{\min}(V_{t+1,\phi})}}.
	\end{align*}
	Since either $r_{t+1}\ge\Delta$ or $r_{t+1}=0$, we just need to show that, for all $t\ge\tau_\phi$:
	\begin{equation*}
		2\beta_{t+1,\phi}(\delta)\frac{L_\phi}{\sqrt{\lambda_{\min}(V_{t+1,\phi})}} < \Delta.
	\end{equation*}
	From Lemma~\ref{lem:mineig.basic}, which holds with probability at least $1-\delta$ under $\mathcal{G}_\phi(\delta)$, rearranging:
	\begin{equation*}
		t\lambda_{\phi,\hls} > \frac{4\beta_{t+1,\phi}(\delta)^2L_\phi^2}{\Delta^2}  + 8L_\phi^2\sqrt{t\log(2d_\phi t/\delta)} + L_\phi^2 g_{t\phi}(\delta)- \lambda.
	\end{equation*}
	By replacing $\beta_{t+1,\phi}(\delta)$ with its bound from~\ref{eq:beta.bound} and $g_{t\phi}(\delta)$ with its definition from Prop.~\ref{prop:pulls.basic} a sufficient condition is:
	\begin{align}
		t\lambda_{\phi,\hls} &> \frac{4\left(
			\sigma\sqrt{2\log(1/\delta)+d_{\phi}\log(1+tL_{\phi}^2/(\lambda d_{\phi}))} + \sqrt{\lambda}S_{\phi}	
		\right)^2L_\phi^2}{\Delta^2} + 8L_\phi^2\sqrt{t\log(2d_\phi t/\delta)} 
		\nonumber\\&\qquad
		+  
			\frac{32\Delta_{\max}^2\lambda L_\phi^2 S_\phi^2\sigma^2\left(2\log(1/\delta)+d_\phi \log(1+tL_\phi^2/(\lambda d_\phi))\right)^2}{\Delta^2}
		- \lambda.\label{eq:zeroregret.final}
	\end{align} 
	Since $\phi$ is \hls, $\lambda_{\phi,\hls}>0$ and the LHS is linear in $t$, while the RHS is sublinear. This means we can find a sufficiently large constant $\tau_\phi$ such that~\eqref{eq:zeroregret.final} holds for all $t\ge\tau_\phi$. One such time is derived explicitly in Appendix~\ref{sec:tau}.
\end{proof}

\paragraph{Proof of Lemma~\ref{prop:hls.regret}.}
\begin{proof}
	First assume the good event $\mathcal{G}_\phi(\delta)$ holds.
	From Lemma~\ref{lem:zeroregret.basic}, with probability $1-\delta$, the instantaneous regret is zero after $\tau_\phi$. So, we can replace $n$ with $\tau_{\phi}$ in the anytime regret upper bound from Prop.~\ref{prop:oful.log.regret}, which also holds under $\mathcal{G}_\phi(\delta)$. More precisely, if $n\ge\tau_{\phi}$:
	\begin{equation*}
		R_n = \sum_{t=1}^nr_t = \sum_{t=1}^{\tau_\phi}r_t + \underbrace{\sum_{t=\tau_{\phi}+1}^n r_t}_{0}  = \sum_{t=1}^{\tau_\phi}r_t.
	\end{equation*}
	So, in any case:
	\begin{equation*}
		R_n \le \sum_{t=1}^{\min\{n,\tau_\phi\}}r_t = R_{\min\{n,\tau_\phi\}},
	\end{equation*}
	which can be bounded with Prop.~\ref{prop:oful.log.regret} to obtain:
		\begin{align*}
	R_n \leq 
	&\frac{32\Delta_{\max}^2\lambda S_{\phi}^2\sigma^2}{\Delta} 
	\left(2\ln\left(\frac{1}{\delta}\right)
	+d_{\phi} \ln \left( 1+\frac{\min\{\tau_{\phi},n\} L_{\phi}^2}{\lambda d_{\phi}} \right) \right)^2.
	\end{align*}
	
	From Prop.~\ref{prop:oful.good}, the good event may fail with probability at most $\delta$. Lemma~\ref{lem:zeroregret.basic} may fail anyway with probability at most $\delta$. From a union bound, the overall failure probability is at most $2\delta$.
\end{proof}

\subsection{Expected Regret Bound}\label{app:hls.regret.expected}
To establish an expected-regret guarantee, we need to consider a slight variant of \linucb that employs an adaptive confidence parameter schedule $(\delta_t)_{t=1}^\infty$. For now, we just assume $\delta_t\in(0,1)$ for all $t$ and is a decreasing function of $t$.

We will need the following good events:
\begin{itemize}
	\item $\mathcal{G}_{t\phi} \coloneqq \left\{\norm{\theta_{t\phi} - \theta^\star_{\phi}}_{V_{t\phi}} \le \beta_{t\phi}(\delta_t)\right\}$. From~\citep{abbasi2011improved} we know that this holds with probability at least $1-\delta_t$. Under this event, we have the instantaneous regret upper bound from Prop.~\ref{prop:oful.instreg} restricted to time $t$.
	\item $\mathcal{F}_{t\phi}\coloneqq \left\{n_t \le g_{t\phi}(\delta_t)\right\}$, where $n_t$ is the number of suboptimal pulls of \linucb up to time $t$ and $g_{t\phi}$ is defined in Prop.~\ref{prop:pulls.basic}. This event holds if $\mathcal{G}_{k\phi}$ holds for all $k\le t$. Indeed:
	\begin{align}
		n_t 
		&\le \frac{R_t}{\Delta}
		\le\sum_{k=1}^t\frac{r_k^2}{\Delta} \nonumber\\
		&\le\sum_{k=1}^{t}\frac{4\Delta_{\max}^2\beta_{k\phi}(\delta_k)^2\min\left\{\norm{\phi(x_k,a_k)}_{V_{k\phi}^{-1}}^2,1\right\}}{\Delta}\nonumber\\
		&\le 4\Delta_{\max}^2\beta_{t\phi}(\delta_t)^2\sum_{k=1}^{t}\frac{\min\left\{\norm{\phi(x_k,a_k)}_{V_{k\phi}^{-1}}^2,1\right\}}{\Delta}\label{eq:expected.events.1}\\
		&\le g_{t\phi}(\delta_t),\nonumber 
	\end{align}
	where~\eqref{eq:expected.events.1} uses the fact that $\delta_t$ is decreasing in $t$ and $\beta_{t\phi}(\delta)$ is increasing in $t$ and decreasing in $\delta$, and the last inequality uses the same algebraic manipulations of the proof of Prop.~\ref{prop:pulls.basic}. Hence, from a union bound, $\Prob(\mathcal{F}_{t\phi})\ge 1-\sum_{k=1}^t\delta_k$.
	\item $\mathcal{E}_{t\phi}\coloneqq \left\{\lambda_{\max}\left(\sum_{k=1}^{t}\EV_{x\sim\rho}\left[\phi^\star(x)\phi^\star(x)^\transp\right]-\phi^\star(x_k)\phi^\star(x_k)^\transp\right) \le 4L_\phi^2\sqrt{2t\log(d_\phi/\delta'_t)}\right\}$. This is the event we used to concentrate contexts in the proof of Lemma~\ref{lem:mineig.basic}. From Prop.~\ref{prop:mazuma}, it holds with probability $\delta_t'$. Since it is independent from the behavior of \linucb, we can set $\delta'_t$ to any convenient value regardless of the confidence schedule $\delta_t$.
\end{itemize}

\begin{lemma}\label{prop:hls.regret.expected}
	Consider a contextual bandit problem with realizable linear representation $\phi$ satisfying the \hls condition. Assume $\Delta > 0$, $\max_{x,a}\|\phi(x,a)\|_2 \leq L$ and $\|\theta^\star_\phi\|_2 \leq S$. Then, the expected regret of \oful run with adaptive confidence schedule $\delta_t=1/t^3$ after $n \geq 1$ steps is at most:
	\begin{equation*}
		\EV R_n \le\frac{32\Delta_{\max}^2\lambda S_{\phi}^2\sigma^2}{\Delta} 
		\left(6\ln(\min\{\wt{\tau}_\phi,n\})
		+d_{\phi} \ln \left( 1+\frac{\min\{\wt{\tau}_\phi,n\} L_{\phi}^2}{\lambda d_{\phi}} \right) \right)^2 + 26,
	\end{equation*}
	where $\wt{\tau}_\phi$ is a constant independent from $n$.
\end{lemma}
\begin{proof}
	Fix a time $t$ and assume the three good events $\mathcal{G}_{t+1,\phi}$, $\mathcal{F}_{t\phi}$, and $\mathcal{E}_{t\phi}$ hold. Following the proof of Lemma~\ref{lem:zeroregret.basic} we can show that, thanks to $\mathcal{G}_{t+1,\phi}$:
		\begin{align*}
	r_{t+1} 
	\le 2\beta_{t+1,\phi}(\delta_{t+1})\frac{L_\phi}{\sqrt{\lambda_{\min}(V_{t+1,\phi})}}.
	\end{align*}
	Following the proof of Lemma~\ref{lem:mineig.basic} we can show that, thanks to $\mathcal{G}_{t+1,\phi}$ and $\mathcal{E}_{t\phi}$:
	\begin{equation*}
		\lambda_{\min}(V_{t+1,\phi}) \ge \lambda + t\lambda_{\phi,\hls} - 4L_\phi^2\sqrt{2t\log(d_\phi/\delta'_t)} - L_\phi^2 n_t,
	\end{equation*}
	and $n_t\le g_{t\phi}(\delta_t)$ thanks to $\mathcal{F}_{t\phi}$. Putting all together, as observed in the proof of Lemma~\ref{lem:zeroregret.basic}, a sufficient condition for $r_{t+1}=0$ is:
	\begin{equation*}
	t\lambda_{\phi,\hls} > \frac{4\beta_{t+1,\phi}(\delta_{t+1})^2L_\phi^2}{\Delta^2}  + 4L_\phi^2\sqrt{2t\log(d_\phi/\delta'_t)} + L_\phi^2 g_{t\phi}(\delta_t)- \lambda.
	\end{equation*}
	With our choice of $\delta_t$, by setting $\delta_t'=1/t^2$, this becomes:
		\begin{align}
	t\lambda_{\phi,\hls} &> \frac{4\left(
		\sigma\sqrt{6\log(t+1)+d_{\phi}\log(1+tL_{\phi}^2/(\lambda d_{\phi}))} + \sqrt{\lambda}S_{\phi}	
		\right)^2L_\phi^2}{\Delta^2} + 8L_\phi^2\sqrt{2t\log(d_\phi t)}
	\nonumber\\&\qquad
	+  
	\frac{32\Delta_{\max}^2\lambda L_\phi^2 S_\phi^2\sigma^2\left(6\log(t)+d_\phi \log(1+tL_\phi^2/(\lambda d_\phi))\right)^2}{\Delta^2}
	- \lambda.\label{eq:zeroregret.expected.final}
	\end{align}
	The RHS is still sublinear in $t$, so, if $\phi$ is \hls, we can find a sufficiently large constant $\wt{\tau}_\phi$ such that~\eqref{eq:zeroregret.expected.final} is always satisfied for $t\ge\wt{\tau_{\phi}}$.
	
	If $n\ge\wt{\tau}_{\phi}$, we can decompose the expected regret as follows:
	\begin{align}
		\EV R_n 
		&\le \underbrace{\EV\sum_{t=1}^{\wt{\tau}_{\phi}}r_t}_{(a)} + \underbrace{\EV\sum_{t=\wt{\tau}_{\phi}+1}^{n}r_t}_{(b)}.
	\end{align}
	The first summation is just the expected regret of \linucb up to time $\wt{\tau}_\phi$:
	\begin{align*}
		(a) &= \EV\sum_{t=1}^{\wt{\tau}_{\phi}}\indi{\mathcal{G}_{t\phi}}r_t + \EV\sum_{t=1}^{\wt{\tau}_{\phi}}\indi{\neg\mathcal{G}_{t\phi}}r_t \\
		&\le \sum_{t=1}^{\wt{\tau}_{\phi}}\beta_{k\phi}(\delta_k)\norm{\phi(x_k,a_k)}_{V_k^{-1}} + 2\sum_{t=1}^{\wt{\tau}_{\phi}}\delta_t\\
		&\le \beta_{\wt{\tau}_{\phi},\phi}(\delta_{\wt{\tau}_{\phi}})\sum_{t=1}^{\wt{\tau}_\phi}\norm{\phi(x_k,a_k)}_{V_k^{-1}} + 2\sum_{t=1}^{\infty}\delta_t\\
		&\le \frac{32\Delta_{\max}^2\lambda S_{\phi}^2\sigma^2}{\Delta} 
		\left(6\ln(\wt{\tau}_{\phi})
		+d_{\phi} \ln \left( 1+\frac{\wt{\tau}_{\phi} L_{\phi}^2}{\lambda d_{\phi}} \right) \right)^2 + 2\sum_{t=1}^{\infty}t^{-3}\\
		&\le\frac{32\Delta_{\max}^2\lambda S_{\phi}^2\sigma^2}{\Delta} 
		\left(6\ln(\wt{\tau}_{\phi})
		+d_{\phi} \ln \left( 1+\frac{\wt{\tau}_{\phi} L_{\phi}^2}{\lambda d_{\phi}} \right) \right)^2 + 3,
	\end{align*}
	where we used the same algebraic manipulations used to prove Prop.~\ref{prop:oful.log.regret}.
	For the second summation:
	\begin{align*}
		(b) 
		&= \EV\sum_{t=\wt{\tau}_{\phi}}^{n-1}r_{t+1} \\
		&= \EV\sum_{t=\wt{\tau}_{\phi}}^{n-1}\underbrace{\indi{\mathcal{G}_{t+1,\phi}\cap\mathcal{F}_{t\phi}\cap\mathcal{E}_{t\phi}}r_{t+1}}_{=0} + 
		\EV\sum_{t=\wt{\tau}_{\phi}}^{n-1}\indi{\neg\mathcal{G}_{t+1,\phi}\cup\neg\mathcal{F}_{t\phi}\cup\neg\mathcal{E}_{t\phi}}r_{t+1} \\
		&\le 2\sum_{t=1}^{\infty}\left(\delta_{t} + \sum_{k=1}^tk^{-3} + t^{-2}\right)\\
		&\le 2\sum_{t=1}^{\infty}t^{-3} + 2\sum_{t=1}^{\infty}\sum_{k=1}^t\frac{1}{k^3} + 4
		\le 2\sum_{t=1}^{\infty}\sum_{k=t}^\infty k^{-3} + 7
		\le 2\sum_{t=1}^{\infty}t^{-3/2}\sum_{k=t}^\infty k^{-3/2} + 7
		\le 6\sum_{t=1}^{\infty}t^{-3/2} + 7 \le 23.
	\end{align*}
	In any case, from the same argument used to bound $(a)$:
	\begin{equation*}
	\EV R_n \le\frac{32\Delta_{\max}^2\lambda S_{\phi}^2\sigma^2}{\Delta} 
	\left(6\ln(n)
	+d_{\phi} \ln \left( 1+\frac{n L_{\phi}^2}{\lambda d_{\phi}} \right) \right)^2 + 3.
	\end{equation*}
	Putting everything together, we obtain our statement.
\end{proof}

\subsection{Margin Condition (Proof of Theorem \ref{th:hls.margin})}\label{app:hls.regret.margin}

Before proving Th.~\ref{th:hls.margin}, we need to generalize some of the previous results for the case where the minimum gap can be arbitrarily small but the margin condition (Asm.~\ref{asm:margin.gap}) holds. First, we show two immediate results that bound the regret suffered by \linucb on rounds where the minimum gap is above a given value.

\begin{lemma}\label{lem:epsreg}
Under the good event $G_\phi(\delta)$, for all $\epsilon>0$ and $t > 0$, the $t$-step regret of \linucb on rounds where the minimum gap is at least $\epsilon$ can be bounded as
    \begin{equation}
        R_t^{\epsilon} := \sum_{k=1}^t\indi{\Delta(x_k)\ge\epsilon}r_k \le \frac{32\lambda(\Delta_{\max}S_\phi\sigma)^2}{\epsilon}\left(2\log(1/\delta)+d_\phi\log(1+tL_\phi^2/(\lambda d_\phi))\right)^2.
    \end{equation}
    Moreover, the number of sub-optimal pulls performed by \linucb on rounds where the minimum gap is at least $\epsilon$ can be bounded as
    \begin{equation}
    g_{t\phi}^{\epsilon}(\delta) := \sum_{k=1}^t\indi{\Delta(x_k)>\epsilon, r_k>0} \le \frac{32\lambda(\Delta_{\max}S_\phi\sigma)^2}{\epsilon^2}\left(2\log(1/\delta)+d_\phi\log(1+tL_\phi^2/(\lambda d_\phi))\right)^2
\end{equation}
\end{lemma}
\begin{proof}
It is easy to see that
\begin{align}
    R_t^{\epsilon} = \sum_{k=1}^t\indi{\Delta(x_k)\ge\epsilon}r_k \le \sum_{k=1}^t\indi{r_k>0, \Delta(x_k)\ge\epsilon}\frac{r_k^2}{\epsilon} \le \frac{1}{\epsilon}\sum_{k=1}^tr_k^2.
\end{align}
From here the first result follows by reproducing the proof of the regret bound in Prop.~\ref{prop:oful.log.regret}. The second result is immediate from
\begin{align}
        R_t^{\epsilon} = \sum_{k=1}^t\indi{\Delta(x_k)\ge\epsilon, r_k>0}r_k \ge \epsilon g_{t\phi}^\epsilon(\delta)
    \implies g_{t\phi}^\epsilon(\delta) \le \frac{R_t^\epsilon}{\epsilon}.
    \end{align}
\end{proof}

Next, we generalize Lem.~\ref{lem:mineig.basic}.

\begin{lemma}\label{lem:mineig.margin}
Under Asm.~\ref{asm:margin.gap}, for any $\epsilon>0$ and $\delta \in (0,1)$, with probability at least $1-\delta$, for every $t>0$,
\begin{equation}
		\lambda_{\min}(V_{t+1,\phi}) \ge \lambda + t(\lambda_{\phi,\hls} - L_\phi^2C\epsilon^{\alpha}) -4(1+C\epsilon^\alpha)L_\phi^2\sqrt{t\log(2d_\phi t/\delta)} - L_\phi^2g_{t\phi}^\epsilon(\delta).
	\end{equation}
\end{lemma}
\begin{proof}
	Using a similar decomposition as in the proof of Lem.~\ref{lem:mineig.basic},
	\begin{align}
		V_{t+1,\phi} 
		&=\lambda I_{d_\phi} + \sum_{k=1}^{t}\phi(x_k,a_k)\phi(x_k,a_k)^\transp \nonumber\\
		&\succeq \lambda I_{d_\phi} + \sum_{k=1}^{t}\indi{\Delta(x_k)\geq\epsilon,r_k=0}\phi^\star(x_k)\phi^\star(x_k)^\transp \nonumber\\
		&=\lambda I_{d_\phi} + \sum_{k=1}^{t}\indi{\Delta(x_k)\geq\epsilon}\phi^\star(x_k)\phi^\star(x_k)^\transp
		-\sum_{k=1}^{t}\indi{\Delta(x_k)\geq\epsilon,r_k>0}\phi^\star(x_k)\phi^\star(x_k)^\transp \nonumber\\
		&\succeq \lambda I_{d_\phi} + \sum_{k=1}^{t}\indi{\Delta(x_k)\geq\epsilon}\phi^\star(x_k)\phi^\star(x_k)^\transp
		-g_{t\phi}^\epsilon(\delta)L_\phi^2I_{d_\phi} \nonumber\\
		&= \lambda I_{d_\phi} +
		t\EV_{x\sim\rho}\left[\indi{\Delta(x)\geq\epsilon}\phi^\star(x)\phi^\star(x)^\transp\right] - 
		\sum_{k=1}^{t}X_k
		-L_\phi^2g_{t\phi}^\epsilon(\delta)I_{d_\phi},\label{eq:mineig.basic.matrix}
	\end{align}
	where $g_{t\phi}^\epsilon(\delta)$ is given in Lem.~\ref{lem:epsreg} and $X_k :=\EV_{x\sim\rho}\left[\indi{\Delta(x)\geq\epsilon}\phi^\star(x)\phi^\star(x)^\transp\right]-\indi{\Delta(x_k)\geq\epsilon}\phi^\star(x_k)\phi^\star(x_k)^\transp$. Using the margin condition,
	\begin{align*}
    \lambda_{\min}\left(\EV_{x\sim \rho}\left[\indi{\Delta(x)\geq\epsilon}\phi^\star(x)\phi^\star(x)^T\right]\right) &\ge \lambda_{\phi,\hls} - \lambda_{\max}\left(\EV_{x\sim \rho}\left[\indi{\Delta(x)<\epsilon}\phi^\star(x)\phi^\star(x)^T\right]\right) \\
    &\ge \lambda_{\phi,\hls} -L_\phi^2 \mathbb{P}(\Delta(x) < \epsilon) \ge \lambda_{\phi,\hls} - L_\phi^2C\epsilon^{\alpha}.
\end{align*}
Thus,
	\begin{align*}
		\lambda_{\min}(V_{t+1,\phi}) \ge \lambda + t(\lambda_{\phi,\hls} - L_\phi^2C\epsilon^{\alpha}) -\lambda_{\max}\left(\sum_{k=1}^{t}X_k\right)- L_\phi^2g_{t\phi}^\epsilon(\delta).
	\end{align*}
	As in Lem.~\ref{lem:mineig.basic}, we can bound the third term using a matrix Azuma inequality by~\citet{tropp2012user}. Note that $\EV_k[X_k] = 0$ and $\norm{X_k}\le (1+\mathbb{P}(\Delta(x) < \epsilon))L_\phi^2 \le (1+C\epsilon^\alpha)L_\phi^2$. Hence, from Prop.~\ref{prop:mazuma}, with probability at least $1-\delta'_t$, for all $t\ge1$:
	\begin{equation*}
		\lambda_{\max}\left(\sum_{k=1}^{t}X_k\right) \le 2(1+C\epsilon^\alpha)L_\phi^2\sqrt{t\log(d_\phi/\delta'_t)}.
	\end{equation*}
	Setting $\delta'_t=\delta/(2t^2)$ and taking a union bound over time, we have that, with probability at least $1-\delta$, for all $t\ge1$,
	\begin{equation*}
		\lambda_{\max}\left(\sum_{k=1}^{t}X_k\right) \le 2(1+C\epsilon^\alpha)L_\phi^2\sqrt{t\log(2d_\phi t^2/\delta)}\le 4(1+C\epsilon^\alpha)L_\phi^2\sqrt{t\log(2d_\phi t/\delta)}.
	\end{equation*}
\end{proof}

Finally, we need the following upper bound on the regret suffered by \linucb on (possibly random) subsets of rounds.

\begin{lemma}[Cf. the proof of Lemma 13 in~\citet{tirinzoni2020asymptotically}]\label{lem:filter}
Let $\{E_t\}_{t\geq 1}$ be any sequence of events. Then, under the good event $G_\phi(\delta)$, the $n$-step regret of \linucb on rounds where the corresponding event holds can be upper bounded by
\begin{align}
    \sum_{t=1}^n\indi{E_t}r_t \le \sqrt{8\Delta_{\max}^2N_n d_\phi\log(1+N_n L_\phi^2/(\lambda d_\phi))}\left(\sqrt{\lambda}S_\phi +\sigma\sqrt{2\log(1/\delta)+d_\phi\log(1+n L_\phi^2/(\lambda d_\phi))}\right),
\end{align}
where $N_n := \sum_{t=1}^n\indi{E_t}$.
\end{lemma}
\begin{proof}
Using the standard regret decomposition for \linucb, 
    \begin{align*}
         \sum_{t=1}^n\indi{E_t}r_t &\le \sqrt{N_n\sum_{t=1}^n \indi{E_t} r_t^2} \le \sqrt{4N_n\Delta_{\max}^2\beta_{n\phi}^2(\delta)\sum_{t=1}^n \indi{E_t} \min\{\norm{\phi(x_t,a_t)}^2_{V_{t\phi}^{-1}}, 1\}} \\
        &\le \sqrt{4N_n\Delta_{\max}^2\beta_{n\phi}^2(\delta)\sum_{t=1}^n \indi{E_t} \min\{\norm{\phi(x_t,a_t)}^2_{\wt{V}_{t\phi}^{-1}}, 1\}} \\
        &\le \sqrt{8N_n\Delta_{\max}^2\beta_{n\phi}^2(\delta)d_\phi\log(1+N_n L_\phi^2/(\lambda d_\phi))},
    \end{align*}
    where $\wt{V}_{t\phi} := \lambda + \sum_{k=1}^t\indi{E_k}\phi(x_k,a_k)\phi(x_k,a_k)^T \preceq {V}_{t\phi}$ since ${V}_{t\phi} - \wt{V}_{t\phi} \succeq 0$. This implies $\wt{V}_{t\phi}^{-1}\succeq {V}_{t\phi}^{-1}$, which implies the norm inequality. The last inequality is from the elliptical potential lemma \citep{abbasi2011improved}.
\end{proof}

\paragraph{Proof of Theorem \ref{th:hls.margin}.}

We begin by proving the regret bound for \linucb with margin condition and without a \hls representation (first statement of Th.~\ref{th:hls.margin}). Then, we prove the regret bound with \hls representation (second statement of Th.~\ref{th:hls.margin}).

\paragraph{Regret bound for \linucb without \hls condition.}

Let $\epsilon > 0$ to be chosen later. We start by splitting the regret into rounds where the minimum gap is above $\epsilon$,
\begin{align}
    R_n = \underbrace{\sum_{t=1}^{n}\indi{\Delta(x_t)\geq\epsilon}r_t}_{(a)} 
    +\underbrace{\sum_{t=1}^{n}\indi{\Delta(x_t)<\epsilon}r_t}_{(b)}.
\end{align}
From Lem.~\ref{lem:epsreg},
\begin{align}
    (a) \le \frac{32\lambda(\Delta_{\max}S_\phi\sigma)^2}{\epsilon}\left(2\log(1/\delta)+d_\phi\log(1+nL_\phi^2/(\lambda d_\phi))\right)^2.
\end{align}
Term $(b)$ can be bounded using Lem.~\ref{lem:filter} with the sequence of events $\{\Delta(x_t)<\epsilon\}_{t\geq 1}$, 
\begin{align}
(b) \leq \sqrt{8\Delta_{\max}^2N_n d_\phi\log(1+N_n L_\phi^2/(\lambda d_\phi))}\left(\sqrt{\lambda}S_\phi +\sigma\sqrt{2\log(1/\delta)+d_\phi\log(1+n L_\phi^2/(\lambda d_\phi))}\right).
\end{align}
It only remains to bound $N_n$, i.e., the count of these events. We have,
\begin{align}
    N_n := \sum_{t=1}^n\indi{\Delta(x_t)<\epsilon} &\le \mathbb{P}(\Delta(x) < \epsilon)n + 2\sqrt{\mathbb{P}(\Delta(x) < \epsilon)(1-\mathbb{P}(\Delta(x) < \epsilon))n\log(2n/\delta)} +\frac{2}{3}\log(2n/\delta)\\
    &\le C\epsilon^\alpha n + 2\sqrt{C\epsilon^\alpha(1-C\epsilon^\alpha)n\log(2n/\delta)} +\frac{2}{3}\log(2n/\delta),
\end{align}
where the first inequality is from (any-time) Bernstein inequality and the second one is from the margin condition. The proof of the first result is concluded by setting $\epsilon = n^{-1/\alpha}$, for which the regret has the stated order. The final $1-3\delta$ probability is the union bound of three events: the good event for \linucb, concentration of contexts for Lem.~\ref{lem:mineig.margin}, and concentration of ``low-gap'' contexts to bound $N_n$.

\paragraph{Regret bound for \linucb with \hls condition.}

We follow similar steps as before, while considering a decreasing sequence $\epsilon_t = t^{-1/\alpha}$ instead of a fixed $\epsilon$. We split the immediate regret at time $t+1$ as
\begin{align*}
    r_{t+1} = \underbrace{\indi{\Delta(x_{t+1}) \geq \epsilon_{t+1}}r_{t+1}}_{(a)} + 
    +\underbrace{\indi{\Delta(x_{t+1})<\epsilon_{t+1})}r_{t+1}}_{(b)}.
\end{align*}
Let us start from (a). Since the minimum gap is larger than $\epsilon_{t+1}$, a sufficient condition for the immediate regret to be zero is (cf. proof of Lem.~\ref{lem:zeroregret.basic})
    \begin{align}
        \frac{2L_\phi\beta_{t+1,\phi}(\delta)}{\sqrt{\lambda_{\min}(V_{t+1,\phi})}}
        < \epsilon_{t+1}.
    \end{align}
    Rearranging and expanding the definition of $\epsilon_t$,
        \begin{align}
        \lambda_{\min}(V_{t+1,\phi}) >  4L_\phi^2 \beta_{t+1,\phi}^2(\delta) (t+1)^{2/\alpha}.
    \end{align}
Using Lem.~\ref{lem:mineig.margin} with $\epsilon = \left(\frac{\lambda_{\phi,\hls}}{2L_\phi^2 C}\right)^{1/\alpha}$,
\begin{equation}
		\lambda_{\min}(V_{t+1,\phi}) \ge \lambda + \frac{\lambda_{\phi,\hls}}{2}t -6L_\phi^2\sqrt{t\log(2d_\phi t/\delta)} - L_\phi^2g_{t\phi}^\epsilon(\delta),
	\end{equation}
where $g_{t\phi}^\epsilon(\delta) \leq {O}((C/\lambda_{\phi,\hls})^{2/\alpha}(\log t)^2)$ by our choice of $\epsilon$.
Therefore, since (i) $\alpha > 2$ and (ii) the minimum eigenvalue of $V_{t+1,\phi}$ grows linearly with $t$, there exists a constant time $\tau_\phi$ after which $r_{t} = 0$. Moreover, it is easy to see that $\tau_{\phi} \propto (\lambda_{\phi,\hls})^{\frac{\alpha}{2 - \alpha}}$. Therefore, summing (a) from $t=1$ to $t=n$ yields a total regret bounded by $\Delta_{\max}\tau_\phi$. It only remains to characterize the contribution of (b) to the regret. Using Lem.~\ref{lem:filter} on the sequence of events $\{\Delta(x_t)<\epsilon_t\}$,
\begin{align}
\sum_{t=1}^n \indi{\Delta(x_{t})<\epsilon_{t})}r_{t} \leq \sqrt{8\Delta_{\max}^2N_n d_\phi\log(1+N_n L_\phi^2/(\lambda d_\phi))}\left(\sqrt{\lambda}S_\phi +\sigma\sqrt{2\log(1/\delta)+d_\phi\log(1+n L_\phi^2/(\lambda d_\phi))}\right).
\end{align}
To bound $N_n$, note that $\{\Delta(x_t)<\epsilon_t\}$ is a sequence of independent random variables. Moreover $\mathbb{E}[N_n] = \sum_{t=1}^n \mathbb{P}\{\Delta(x_t)<\epsilon_t\} \leq C\sum_{t=1}^n \epsilon_t^\alpha$ and $\mathbb{V}[N_n] = \sum_{t=1}^n \mathbb{P}\{\Delta(x_t)<\epsilon_t\}(1-\mathbb{P}\{\Delta(x_t)<\epsilon_t\}) \leq C\sum_{t=1}^n \epsilon_t^\alpha$. Thus, using Bernstein's inequality,
\begin{align}
    N_n 
    \le \mathbb{E}[N_n] + 2\sqrt{\mathbb{V}[N_n]\log(2n/\delta)} +\frac{2}{3}\log(2n/\delta)\le C\sum_{t=1}^n \epsilon_t^\alpha + 2\sqrt{C\sum_{t=1}^n \epsilon_t^\alpha\log(2n/\delta)} +\frac{2}{3}\log(2n/\delta).
\end{align}
Since $\sum_{t=1}^n \epsilon_t^\alpha = \sum_{t=1}^n 1/t \leq \log n + 1$, we obtain the stated order of regret.
The final $1-3\delta$ probability is the union bound of the same three events considered in the proof of the first part of the theorem.

\subsection{\hls is Necessary and Sufficient}\label{app:hls.regret.iif}

Prop. \ref{prop:hls.regret} shows that \linucb suffers constant regret on any linear contextual bandit problem with strictly-positive minimum gap when the algorithm is run with a \hls realizable representation $\phi$. In this section, we prove that the \hls condition is also necessary for achieving constant regret. We do so by leveraging the asymptotic problem-dependent regret lower-bound for linear contextual bandits \citep{lattimore2017end,hao2020adaptive,tirinzoni2020asymptotically}, from which we show that any consistent algorithm (like \linucb) must suffer logarithmic regret if the representation is not \hls. We note that \citet{hao2020adaptive} proved that the asymptotic lower bound is sub-logarithmic when representation $\phi$ is \hls, while we could not find a formal proof that the asymptotic regret is logarithmic when $\phi$ is not \hls.

In addition to those in the main paper, we consider the same assumptions used to derive the asymptotic lower-bound for linear contextual bandits \citep{hao2020adaptive,tirinzoni2020asymptotically}:
\begin{enumerate}
\item The set of contexts $\mathcal{X}$ is finite;
\item The context distribution $\rho$ is full-support, i.e., $\rho(x) > 0$ for each $x\in\X$;
\item The reward noise $\eta_t$ is i.i.d. from $\mathcal{N}(0,1)$.
\end{enumerate}
We note that the finiteness of $\mathcal{X}$ implies that the minimum-gap assumption holds (i.e., $\Delta > 0$). We believe that the asymptotic lower-bound (and thus our results) could be generalized to infinite contexts with an assumption on the minimum positive gap, though this is out of the scope of this work.

We start by stating the asymptotic lower bound on the expected regret of any consistent bandit algorithm. Formally, we call an algorithm consistent if it suffers $o(n^\alpha)$ regret for any $\alpha \in (0,1)$ in any linear contextual problem satisfying the assumptions above.

\begin{proposition}[\citealp{hao2020adaptive}]\label{p:lower.bound}
Consider any consistent bandit algorithm and any linear contextual bandit problem satisfying the assumptions above. Then,
\begin{equation}\label{eq:regret.lower.bound}
\liminf_{n \rightarrow \infty}\frac{\mathbb{E} \big[ R_n \big]}{\log(n)} \geq v^\star(\theta^\star),
\end{equation} 
where $v^\star$ is the value of the optimization problem
	\begin{equation}\label{eq:optim-lb}\tag{P}
	\begin{aligned}
    &\underset{\eta(x,a) \geq 0}{\inf}&& \sum_{x\in\X}\sum_{a\in\A}\eta(x,a)\Delta(x,a)
    \quad \mathrm{s.t.} \quad
    \|\phi(x,a)\|_{V_\eta^{-1}}^2 \leq \frac{\Delta(x,a)^2}{2} \quad \forall x\in\X,a\in\A : \Delta(x,a) > 0,
	\end{aligned}
	\end{equation}
where $V_\eta := \sum_x\sum_a\eta(x,a)\phi(x,a)\phi(x,a)^T$.
\end{proposition}

We now show that, if $\phi$ is not \hls, then there exists a positive constant $c > 0$ such that $v^\star \geq c$. This implies that the asymptotic regret of any consistent strategy (including \linucb) grows at rate at least $c\log(n)$, hence proving our main claim.

\begin{proposition}
Consider any consistent bandit algorithm and any linear contextual bandit problem satisfying the assumptions above with a realizable representation $\phi$ that is not \hls. Then, there exists a positive constant $c > 0$ such that 
\begin{equation}\label{eq:regret.lower.bound}
\liminf_{n \rightarrow \infty}\frac{\mathbb{E} \big[ R_n \big]}{\log(n)} \geq c.
\end{equation} 
\end{proposition}
\begin{proof}
We start by noting that, if $\phi$ is not \hls, then there exists at least one context $\bar{x}$ and one sub-optimal arm $\bar{a}$ such that $\phi(\bar{x},\bar{a}) \notin \mathrm{span}(\{\phi^\star(x)\}_{x\in\X})$. Clearly, for any $\lambda \geq 0$, $v^\star$ is larger than the value of the following optimization problem
	\begin{equation}
	\begin{aligned}
    &\underset{\eta(x,a) \geq 0}{\inf}&& \sum_{x\in\X}\sum_{a\in\A}\eta(x,a)\Delta(x,a)
    \quad \mathrm{s.t.} \quad
    \|\phi(\bar{x},\bar{a})\|_{(\lambda I + V_\eta)^{-1}}^2 \leq \frac{\Delta(\bar{x},\bar{a})^2}{2}.
	\end{aligned}
	\end{equation}
	This holds because we removed all constraints except the one for $\bar{x},\bar{a}$ and added a small regularization to the matrix $V_\eta$ (which can only decrease the norm). Let $\eta^\star(x,a)$ be an optimal solution of this optimization problem. We now prove that there exists a positive constant $c' > 0$ such that, for some context $x$ and sub-optimal arm $a$, $\eta^\star(x,a) \geq c'$. Let us proceed by contradiction. Suppose that all sub-optimal arms have $\eta^\star(x,a) = 0$. Since $\eta^\star$ is feasible, we must have that 
	\begin{align}
	\|\phi(\bar{x},\bar{a})\|_{\left(\lambda I + \sum_{x\in\X}\eta(x,a^\star(x))\phi^\star(x)\phi^\star(x)^T\right)^{-1}}^2 \leq \frac{\Delta(\bar{x},\bar{a})^2}{2}.
	\end{align}
Let $\lambda_i, u_i$ denote the eigenvalues/eigenvectors of the matrix $\sum_{x\in\X}\eta(x,a^\star(x))\phi^\star(x)\phi^\star(x)^T$. Note that at least one of the eigenvalues must be zero since the optimal features do not span $\mathbb{R}^d$. Moreover, since $\phi(\bar{x},\bar{a})$ is not in the span of the optimal arms, Lem. \ref{lemma:span-eig} ensures that there exists some $\epsilon > 0$ such that $|\phi(\bar{x},\bar{a})^T u_i| \geq \epsilon$ for at least one eigenvector $u_i$ associated with a zero eigenvalue. Then, the norm can be re-written as
\begin{align}
\|\phi(\bar{x},\bar{a})\|_{\left(\lambda I + \sum_{x\in\X}\eta(x,a^\star(x))\phi^\star(x)\phi^\star(x)^T\right)^{-1}}^2 = \sum_{i\in[d]} \frac{(\phi(\bar{x},\bar{a})^T u_i)^2}{\lambda + \lambda_i} \geq \frac{\epsilon^2}{\lambda}.
\end{align}
Since $\lambda$ was arbitrary, we can set it to any value $\lambda < \frac{2\epsilon^2}{\Delta(\bar{x},\bar{a})^2}$, for which we have a contradiction. Thus, we have proved that there always exists a positive constant $c' > 0$ such that, for some context $x$ and sub-optimal arm $a$, $\eta^\star(x,a) \geq c'$. This means that the value of the original optimization problem is at least $v^\star \geq c'\Delta > 0$. This concludes the proof.
\end{proof}

\subsection{An Explicit Bound on $\tau$ from Lemma 2}\label{sec:tau}
We need to find a $t$ that satisfies:
\begin{align*}
    t > \frac{1}{\lambda_{HLS}}\left(8L^2\sqrt{t\log({2dt}/{\delta})} + L^2g_t(\delta)+\frac{4L^2\beta_t(\delta)}{\Delta^2}-\lambda\right).
\end{align*}

Note that:
\begin{align*}
    \frac{4L^2\beta_t(\delta)}{\Delta^2} < \frac{4L^2}{\Delta^2} \left( \lambda S^2 + \sigma^2\log\left(\frac{d\lambda+tL^2}{\delta d} \right) \right) < L^2 g_t(\delta),
\end{align*}
thus a sufficient condition is:
\begin{align}
    t > \frac{1}{\lambda_{HLS}}\left(\underbrace{8L^2\sqrt{t\log({2dt}/{\delta})}}_{(A)} + \underbrace{2L^2g_t(\delta)}_{(B)}\right).\label{eq:tau.twocases}
\end{align}
We consider two cases.

\paragraph{Case 1.}
We first consider the case $(A)\ge (B)$. In this case, we just need to find a $t$ such that:
\begin{align*}
    &t > \frac{1}{\lambda_{HLS}} 16L^2\sqrt{t\log({2dt}/{\delta})}
    &\iff&& t > \frac{256 L^4}{\lambda_{\hls}^2} \log({2dt}/{\delta}).
\end{align*}
Inequalities of the form $at \geq \log(bt)$ (with $a, b>0$) can be solved using the Lambert $W$ function as follows:
\begin{align*}
    &e^{\ln(bt) - at} \leq 1\\
    \iff & bt e^{-at} \leq 1\\
    \iff & -at e^{-at} \geq -a/b := x.
\end{align*}
We get real solutions only if $x > -e^{-1}$. If $x \in (e^{-1}, 0)$, then $-at = W_{-1}(x)$. If $x>0$, then $-at = W_0(x)$.
We can make the bound more explicit by noting that $-1-\sqrt{2u}-u\leq W_{-1}(-e^{-u-1})\leq -1 - \sqrt{2u} - 2u/3$ for $u>0$~\citep{Chatzigeorgiou16} and, when $x > e$, $\ln(x) - \ln\ln(x) \leq W_0(x) \leq \ln(x) - \frac{1}{2} \ln\ln(x)$~\citep{hoorfar2008inequalities}.

Since $a = \frac{\lambda_{\hls}^2}{256 L^4}$ and $b=2d/\delta$, we have that $0 > -a/b = \frac{-\delta \lambda_{\hls}^2}{512 d L^4} > e^{-1}$ for reasonable values of the constants. Then $t \geq \frac{-W_{-1}(-a/b)}{a}$, or $t \geq \frac{1+\sqrt{2u}+u}{a}$ with $u=\ln(b/a) + 1$. Then a valid $t$ is:
\begin{align*}
    t \geq \frac{768 L^4}{\lambda_{\hls}^2}\ln\left(\frac{512 d L^4}{\delta \lambda_{\hls}^2}\right).
\end{align*}


\paragraph{Case 2:}
We now consider the case $(A)< (B)$ in~\eqref{eq:tau.twocases}. We seek a $t$ such that:
\begin{align*}
    t > \frac{4L^2g_t(\delta)}{\lambda_{\hls}},
\end{align*}
but notice that:
\begin{align*}
    &L^2g_t(\delta) < \frac{64L^2\sigma^2\lambda S^2}{\Delta^2} 4d^2 \log^2\left(\frac{d\lambda+tL^2}{\delta d} \right),
\end{align*}
so equivalently:
\begin{equation*}
	t \geq \frac{256 L^2\sigma^2\lambda S^2}{\lambda_{\hls}\Delta^2} 4d^2 \log^2\left(\frac{d\lambda+tL^2}{\delta d} \right).
\end{equation*}
We have an inequality of the kind  $\sqrt{ct} = \ln(bt)$. Let $y=\sqrt{t}$ and $a=\sqrt{c}$, then $ay = 2\ln(by)$, then:
\begin{align*}
    0 = \ln(by) - \frac{ay}{2}\iff -\frac{a}{2b} = \frac{-ay}{2} e^{-ay/2}.
\end{align*}
We have that $a = \sqrt{\frac{\lambda_{\hls}\Delta^2}{1024 d^2L^2\sigma^2\lambda S^2}}$, and since $\frac{d\lambda+tL^2}{\delta d} < \frac{2d L^2 t}{\delta}$, then $b = \frac{2d L^2}{\delta}$.
Note that $0 > -\frac{a}{2b} = -\frac{\Delta\delta\sqrt{\lambda_{\hls}}}{64 d^2L^3\sigma S \sqrt{\lambda}} > e^{-1}$ for reasonable values of the constants.
Then $-\frac{ay}{2} = W_{-1}\left(-\frac{a}{2b}\right)$ and:
\begin{align*}
&y\geq \frac{6}{a} \ln(2b/a) = \frac{384 dLS\sigma \sqrt{\lambda}}{\sqrt{\lambda_{\hls}}\Delta} \ln\left(\frac{64 d^2L^3\sigma S \sqrt{\lambda}}{\sqrt{\lambda_{\hls}}\Delta\delta} \right),
\end{align*}
so a valid $t$ is:
\begin{equation*}
t \geq \left(\frac{384 dLS\sigma \sqrt{\lambda}}{\lambda_{\hls}\Delta} \ln\left(\frac{64 d^2L^3\sigma S \sqrt{\lambda}}{\sqrt{\lambda_{\hls}}\Delta\delta} \right) \right)^2.
\end{equation*}

In conclusion we have that:
\begin{align}
    t \geq \max\left\{\frac{384^2 d^2L^2S^2\sigma^2 \lambda}{\lambda_{\hls}\Delta^2} \ln^2\left(\frac{64 d^2L^3\sigma S \sqrt{\lambda}}{\sqrt{\lambda_{\hls}}\Delta\delta} \right), \frac{768 L^4}{\lambda_{\hls}^2}\ln\left(\frac{512 d L^4}{\delta \lambda_{\hls}^2}\right) \right\}.
\end{align}

\section{\hls Representations and Best-Arm Identification}\label{app:hls.best-arm}

In this section, we show that the \hls condition also enables \linucb to solve best-arm identification (BAI) problems in the linear contextual bandit setting. More precisely, we show that, given a \hls representation, \linucb equipped with a generalized likelihood-ratio test~\citep[e.g.,][]{hao2020adaptive,tirinzoni2020asymptotically} stops after finite number of steps and returns the optimal arms of each context. In fact, thanks to the \hls condition, \linucb collects sufficient information about $\theta^\star_\phi$ by pulling the optimal arms alone, hence enabling us to prove that the algorithm stops after a finite time. 

For each parameter $\theta\in\mathbb{R}^d$, let $a^\star_x(\theta):= \argmax_{a\in\A}\phi(x,a)^T\theta$. We assume that the true optimal arms $a^\star_x(\theta^\star)$ are unique for each context $x\in\X$, with corresponding optimal feature vector $\phi^\star(x)$. Moreover, we shall remove the subscripts $\phi$ to simplify the notation.

Consider the following variant of the \linucb algorithm. At each step $t$, before choosing the next arm to pull, we perform the following test:
\begin{align}\label{eq:glrt}
\inf_{\theta \in \bar{\Theta}_{t}} \| {\theta}_{t} - \theta \|_{V_{t}} > \beta_{t}(\delta),
\end{align}
where 
\begin{align}
\bar{\Theta}_{t} := \left\{\theta \in \mathbb{R}^{d} \big|\ \exists x\in\X : a^\star_x(\theta) \neq a^\star_x(\theta_t)\right\}.
\end{align}
is the set of alternative parameters with respect to ${\theta}_{t}$, i.e., those where the optimal arm of at least one context differs from the one of ${\theta}_{t}$. This is the standard generalized likelihood-ratio test adopted for best-arm identification~\citep[e.g.,][]{degenne2020gamification} and in works focusing on asymptotic optimality~\citep[e.g.,][]{hao2020adaptive,tirinzoni2020asymptotically}. For finite contexts, it is known that the test can be re-written in the convenient form (Cf. (79) from \citet{tirinzoni2020asymptotically}):
\begin{align}\label{eq:glr-closed-form}
\min_{x\in\mathcal{X}}\min_{a\neq a^\star_x({\theta}_{t})} \frac{\big(\phi(x,a^\star_x(\theta_t)) - \phi(x,a)\big)^T\theta_t}{\|\phi(x, a^\star_x({\theta}_{t})) - \phi(s,a)\|_{V_{t}^{-1}}} > \beta_{t}(\delta).
\end{align}
Then, if the test triggers, we simply return the current least-square estimate $\theta_t$. Otherwise, we keep running \linucb in its original form.

Using the standard confidence set derived for \linucb \cite{abbasi2011improved}, it is easy to prove the following result.
\begin{lemma}[\linucb for BAI is $\delta$-correct]
If \eqref{eq:glrt} holds at time $t$, then, with probability at least $1-\delta$, for all $x\in\mathcal{X}$, $a^\star_x(\theta^\star) = a^\star_x(\theta_t)$.
\end{lemma}
\begin{proof}
By contradiction, suppose the statement does not hold. This means that, for some context $x\in\mathcal{X}$, the true optimal arm $a^\star_x(\theta^\star)$ is sub-optimal for ${\theta}_{t}$. By definition, this implies that $\theta^\star \in \bar{\Theta}_{t}$, so that,
\begin{align*}
\| {\theta}_{t} - \theta^\star \|_{V_t} \geq \inf_{\theta \in \bar{\Theta}_{t}} \| {\theta}_{t} - \theta \|_{V_{t}} > \beta_{t}(\delta),
\end{align*}
which holds with probability at most $\delta$ since $\theta^\star$ is contained in the confidence ellipsoid with probability at least $1-\delta$.
\end{proof}

The following result shows that \linucb run with a \hls representation and the generalized likelihood-ratio test stops in finite time and retrieves the true optimal arms with high probability.
\begin{lemma}
Let $\phi$ be \hls with $\lambda_{\mathrm{HLS}} := \lambda_{\min}(\EV[\phi^\star(x)\phi^\star(x)^T])>0$ and $\delta\in(0,1)$. Let $\tau\geq 1$ be such that, for all $t\geq \tau - 1$,
\begin{align*}
\lambda + t\lambda_{\hls} - 8L^2\sqrt{t\log(2d t/\delta)} - L^2 g_{t}(\delta) \geq \frac{16L^2\beta_{t+1}^2(\delta)}{\Delta^2}.
\end{align*}
Then, with probability $1-\delta$, \linucb for BAI stops in at most $\tau$ steps and returns a parameter whose optimal arms match the true ones.
\end{lemma}
\begin{proof}
The fact that the returned parameter is correct was already proved in the previous lemma, so we only focus on showing that \linucb eventually stops. In order to achieve so, we lower bound the left-hand side in \eqref{eq:glr-closed-form} by a function that grows linealy in time. Fix any time step $t$ such that $t+1\geq \tau$, $x\in\mathcal{X}$, $a\neq a^\star_{{\theta}_{t+1}}(s)$. Then, using Lem.~\ref{lem:mineig.basic},
\begin{align*}
\|\phi(x, a^\star_x({\theta}_{t+1})) - \phi(x,a)\|_{V_{t+1}^{-1}}^2 \leq \frac{4L^2}{\lambda_{\min}(V_{t+1})} \leq \frac{4L^2}{\lambda + t\lambda_{\hls} - 8L^2\sqrt{t\log(2d t/\delta)} - L^2 g_{t}(\delta)}.
\end{align*}
Similarly,
\begin{align*}
\big(\phi(x,a^\star_x(\theta_{t+1})) - \phi(x,a)\big)^T\theta_{t+1} &\geq \Delta(s,a) - 2{\beta_{t+1}(\delta)}\max_{a'\in[K]}\|\phi(x,a')\|_{V_{t}^{-1}} \geq \Delta - \frac{2L{\beta_{t+1}(\delta)}}{\sqrt{\lambda_{\min}(V_{t+1})}}\\ &\geq \Delta - \frac{2L{\beta_{t+1}(\delta)}}{\sqrt{ \lambda + t\lambda_{\hls} - 8L^2\sqrt{t\log(2d t/\delta)} - L^2 g_{t}(\delta) }} \geq \frac{\Delta}{2},
\end{align*}
where we used Lem.~\ref{lem:mineig.basic} and the fact that $t+1 \geq \tau$. Then, the left-hand side of \eqref{eq:glr-closed-form} is at least,
\begin{align*}
\frac{\Delta}{4L}\sqrt{\lambda + t\lambda_{\hls} - 8L^2\sqrt{t\log(2d t/\delta)} - L^2 g_{t}(\delta)} \geq \beta_{t+1}(\delta),
\end{align*}
where the last inequality holds for $t+1 \geq \tau$. This means that the test triggers for $t+1\geq\tau$, which concludes the proof.
\end{proof}
\section{Representation Selection}
\label{app:rep.selection}
In this section, we study the properties of \algo when a set of $M$ \emph{realizable} representations is provided.
Denote by $a_t$ the action selected by \algo at time $t$. The the instantaneous regret at time $t$ is
\begin{equation}
    r_t = \mu^\star(x_t) - \mu(x_t,a_t).
\end{equation}
Note that since all the representations are realizable, we have that $r_t = \langle \theta_i^\star, \phi_i(x_t, a^\star) - \phi_i(x_t,a_t)\rangle$ for any $i \in [M]$.\footnote{Recall that we abbreviate $\phi_i$ as just $i$ in subscripts.}
We will address misspecified representations in App.~\ref{app:elimination}.

\subsection{Leveraging Good Representations}\label{app:selection.hls}
The following lemma establishes a key property of \algo. We will use it to leverage access to \hls representations, but the same argument could be used to leverage other nice properties of candidate representations, or even to combine them (see App.~\ref{app:rep.selection.mixing}).

\begin{lemma}\label{lem:selection.main}
	Consider a contextual bandit problem with reward $\mu$, context distribution $\rho$ and $\Delta >0$.
	Let $(\phi_i)$ be a set of $M$ realizable linear representations such that $\max_{x,a} \|\phi_i(x,a)\|_2 \leq L_i$ and $\|\theta^\star_i\|_i \leq S_i$. 
	Then, with probability $1-\delta$, for all $t \geq 1$, the instantaneous regret of \algo run with confidence parameter $\delta$ is:
	\begin{equation*}
		r_t \le \min_{i\in[M]}\left\{
			2\beta_{ti}(\delta/M)\norm{\phi_i(x_t,a_t)}_{V_{ti}^{-1}}
		\right\}.
	\end{equation*}
\end{lemma}
\begin{proof}
	Notice that \linucb is entirely off-policy, in the sense that the quality of its parameter estimates is not affected by the fact of observing feedback from actions selected by an external rule. Since \algo updates the estimates of all representations at each step, we can affirm that it runs $M$ parallel \linucb instances. Notice that each \linucb instance is run with confidence parameter $\delta/M$, where $\delta$ is the global confidence parameter of \algo. For this reason, the good event $\mathcal{G}_i(\delta/M)$ defined in Prop.~\ref{prop:oful.good} holds with probability $1-\delta/M$ for any $i\in[M]$. From a union bound, the intersection of this $M$ events: 
	\begin{equation}
		\mathcal{G}(\delta) \coloneqq \left\{\forall i \in [M], \forall t\ge 1, \norm{\theta_{ti}-\theta_i^\star}_{V_{ti}}\le\beta_{ti}(\delta/M)\right\},\label{eq:selection.good}
	\end{equation}
	holds with probability at least $1-\delta$. 
	Recall that $\mathcal{C}_{ti}(\delta)=\{\theta\in\Reals^{d_i}|\norm{\theta_{ti}-\theta}_{V_{ti}}\le\beta_{ti}(\delta)\}$. So $\theta_i^\star\in\mathcal{C}_{ti}(\delta/M)$ for all $i\in[M]$ under $\mathcal{G}(\delta)$.
	
	We proceed to bound the instantaneous regret at any time $t$ under $\mathcal{G}(\delta)$. 
	Since all the representations are realizable, for all $a\in[K]$, $\mu(x_t,a) = \phi_j(x_t,a)^\transp\theta^\star_j$ for any $j \in [M]$. So it is also true that, for any $i\in[M]$:
	\begin{align}
		\phi_i^\star(x_t)^\transp\theta_i^\star = \mu^\star(x_t) = \max_{a\in[K]}\mu(x,a)
		&=\max_{a\in[K]}\min_{j\in[M]}
			\phi_j(x_t,a)^\transp\theta^\star_j \nonumber\\
		&\le \max_{a\in[K]}\min_{j\in[M]} \max_{\theta\in\mathcal{C}_{tj}(\delta/M)} \phi_j(x_t,a)^\transp\theta \label{eq:selection.hls.1}\\
		&= \min_{j\in[M]} \max_{\theta\in\mathcal{C}_{tj}(\delta/M)} \phi_j(x_t,a_t)^\transp\theta \label{eq:selection.hls.2}, \\
		&\le \max_{\theta\in\mathcal{C}_{ti}(\delta/M)} \phi_i(x_t,a_t)^\transp\theta \nonumber\\
		&\le \phi_i(x_t,a_t)^\transp\wt{\theta}_{ti}\label{eq:selection.hls.3},
	\end{align}
	where~\eqref{eq:selection.hls.1} is from $\mathcal{G}(\delta)$, \eqref{eq:selection.hls.2} is from the arm-selection rule of \algo, and $\wt{\theta}_{ti}\coloneqq\arg\max_{\theta\in\mathcal{C}_{ti}(\delta/M)} \phi_i(x_t,a_t)^\transp\theta$. Finally, for any $i\in[M]$:
	\begin{align}
		r_t &= \phi_i^\star(x_t)^\transp\theta^\star_i -\phi_i(x_t,a_t)^\transp\theta^\star_i \nonumber\\
		&\le \phi_i(x_t,a_t)^\transp\left(\wt{\theta}_{ti}-\theta_i^\star\right) \label{eq:selection.hls.4}\\
		&\le \norm{\wt{\theta}_{ti}-\theta_i^\star}_{V_{ti}}\norm{\phi(x_t,a_t)}_{V_{ti}^{-1}} \nonumber\\
		&\le 2\beta_{ti}(\delta/M)\norm{\phi(x_t,a_t)}_{V_{ti}^{-1}}\label{eq:selection.hls.final},
	\end{align}
	where~\eqref{eq:selection.hls.4} is from~\eqref{eq:selection.hls.3}, and the last inequality is from $\mathcal{G}(\delta)$ and  $\wt{\theta}_{ti}\in\mathcal{C}_{ti}(\delta/M)$. Since~\eqref{eq:selection.hls.final} holds for all $i\in[M]$ with overall probability $1-\delta$, we can take a minimum over representations to conclude the proof. 
\end{proof}

\paragraph{Remark.} Inequality~\eqref{eq:selection.hls.3} clarifies in which sense \algo is optimistic: every representation $i\in[M]$ can overestimate the optimal reward of $x_t$ with the reward of the action $a_t$ selected by \algo, using an optimistic parameter $\wt{\theta}_{ti}$ from its own confidence ellipsoid $\mathcal{C}_{ti}(\delta/M)$. In this sense, $a_t$ is optimistic according to all representations, even if each representation alone could prescribe a different optimistic action, since $\wt{\theta}_{ti}$ may differ from $\arg\max_{\theta\in\mathcal{C}_{ti}(\delta/M)}\max_{a\in[K]}\phi_i(x_t,a)^\transp\theta$, the optimistic parameter of \linucb. The fact that these prescriptions may all differ from $a_t$ implies that all representations may concur to the final action selection.

\paragraph{Proof of Theorem~\ref{thm:algo.regret.positivegap}.}
The regret upper bound for \algo established by Thm.~\ref{thm:algo.regret.positivegap} is a simple consequence of Lemma~\ref{lem:selection.main} and the results established for the single-representation case in previous sections.
\begin{proof}
	From Lemma~\ref{lem:selection.main}, under the good event $\mathcal{G}$ defined in~\eqref{eq:selection.good}:
	\begin{equation}
		R_n = \sum_{t=1}^{n}r_t \le \sum_{t=1}^{n}
		\min_{i\in[M]}
		\beta_{ti}(\delta/M)\norm{\phi_i(x_t,a_t)}_{V_{ti}^{-1}}
		\le \min_{i\in[M]}\sum_{t=1}^{n}
		\beta_{ti}(\delta/M)\norm{\phi_i(x_t,a_t)}_{V_{ti}^{-1}}.
	\end{equation}
	We can then bound $\beta_{ti}(\delta/M)\norm{\phi_i(x_t,a_t)}_{V_{ti}^{-1}}$ individually for each representation $i\in[M]$. First of all, we can proceed as in the proof of Proposition~\ref{prop:oful.log.regret}. In this case, the regret of \algo is bounded by the smallest of the regret bounds (according to Proposition~\ref{prop:oful.log.regret}) of each \linucb instance if it was run alone with confidence parameter $\delta/M$. For \hls representations, we can proceed as in the proof of Lemma~\ref{prop:hls.regret} instead. For these representations, the time $\tau_i$ after which $\beta_{ti}(\delta/M)\norm{\phi_i(x_t,a_t)}_{V_{ti}^{-1}}=0$ is the same defined in Lemma~\ref{prop:hls.regret}, with $\delta$ replaced by $\delta/M$. The final statement unifies the two cases by defining $\tau_i=\infty$ for representations that are not \hls. The overall probability is $1-2\delta$, from a union bound. The first $1-\delta$ is for $\mathcal{G}$. The additional $1-\delta$ is to apply Lemma~\ref{lem:mineig.basic}, instantiated with $\delta\gets\delta/M$, to each \hls representation.
\end{proof}

\paragraph{Remark.}
By using the same arguments used for the single-representation case, one can easily obtain an expected-regret version of Theorem~\ref{thm:algo.regret.positivegap} (cf. App.~\ref{app:hls.regret.expected}) and relax the zero-gap assumption into a margin condition (cf.~App.~\ref{app:hls.regret.margin}).

\subsection{Mixing representations}\label{app:rep.selection.mixing}
To prove Theorem~\ref{thm:algo.mix.regret.constant}, we first need to establish variants of the preliminary results from App.~\ref{app:hls.regret.prelim} that hold without the \hls condition.

\begin{lemma}[Generalization of Lemma~\ref{lem:mineig.basic} in Loewner ordering]\label{lem:mixing.loewner}
	Under the same assumptions of Proposition~\ref{prop:oful.good}, assuming the good event $\mathcal{G}_{i}(\delta)$ holds, with probability $1-\delta$, for all $t\ge 1$:
	\begin{equation*}
	V_{t+1,i} \succeq \lambda I_{d_i} + t\EV_{x\sim \rho}[\phi_i^\star(x)\phi_i^\star(x)^\transp] - \left(8L_i^2\sqrt{t\log(2d_it/\delta)} + L_i^2 g_{ti}(\delta)\right)I_{d_i},
	\end{equation*}
	where $\succeq$ denotes the Loewner ordering.
\end{lemma}
\begin{proof}
	First, we can rewrite $V_{t+1,i}$ as
	\begin{align*}
	V_{t+1,i} &= \lambda I_{d_i} + \sum_{k=1}^t \phi_i(x_k,a_k)\phi_i(x_k,a_k)^\transp \\
	&= \lambda I_{d_i}  + \sum_{k=1}^t\indi{a_k= a^\star_{s_k}} \phi_i^\star(x_k) \phi_i^\star(x_k)^\transp +\sum_{k=1}^t\indi{a_k\neq a^\star_{s_k}}\phi_i(x_k,a_k) \phi_i(x_k,a_k)^\transp\\
	& \succeq \lambda I_{d_i}  + \sum_{k=1}^t \phi_i^\star(x_k) \phi_i^\star(x_k)^\transp - \sum_{k=1}^t\indi{a_k\neq a^\star_{s_k}} \phi_i^\star(x_k) \phi_i^\star(x_k)^\transp,
	\end{align*}
	where the last inequality is due to the positive-semidefiniteness of $\sum_{k=1}^t\indi{a_k\neq a^\star_{s_k}}\phi_i(x_k,a_k) \phi_i(x_k,a_k)^\transp$. Note that, for two positive semidefinite matrices $A,B$, we have $A - B \succeq A - \lambda_{\max}(B) I$. To see this, note that a sufficient condition for $A - B$ to be greater than $A - \lambda_{\max}(B) I$ in Loewner ordering is $\lambda_{\min}(A - B - A + \lambda_{\max}(B)I) \geq 0$, which is clearly true. Therefore,
	\begin{align*}
	V_{t+1,i} &\succeq \lambda I_{d_i}  + \sum_{k=1}^t \phi_i^\star(x_k) \phi_i^\star(x_k)^\transp - \lambda_{\max}\left(\sum_{k=1}^t\indi{a_k\neq a^\star_{s_k}} \phi_i^\star(x_k) \phi_i^\star(x_k)^\transp\right) I_{d_i} \\
	&\succeq \lambda I_{d_i}  + \sum_{k=1}^t \phi_i^\star(x_k) \phi_i^\star(x_k)^\transp - L_i^2g_{ti}(\delta) I_{d_i},
	\end{align*}
	where $g_{ti}$ is the upper bound on suboptimal pulls from Prop.~\ref{prop:pulls.basic}, which is valid under $\mathcal{G}_i$.
	We can now apply the same reasoning to lower-bound the sum of outer products of optimal feature vectors:
	\begin{align*}
	\sum_{k=1}^t \phi_i^\star(x_k) \phi_i^\star(x_k)^\transp &= \sum_{k=1}^t \phi_i^\star(x_k) \phi_i^\star(x_k)^\transp \pm \sum_{k=1}^t\EV_{s\sim\rho}[\phi_i^\star(s)\phi_i^\star(s)^\transp] \\
	&\succeq t\EV_{s\sim\rho}[\phi_i^\star(s)\phi_i^\star(s)^\transp] - \lambda_{\max}\left(\sum_{k=1}^t\underbrace{\EV_{s\sim\rho}[\phi_i^\star(s)\phi_i^\star(s)^\transp] - \phi_i^\star(x_k)\phi_i^\star(x_k)^\transp}_{X_k}\right)I_{d_i}.
	\end{align*}
	Clearly $\EV[X_k]=0$ since $s\sim\rho$. Also $X_k^2 \preceq \lambda_{\max}(X_k^2)I_{d_i} \preceq \norm{X_k}^2I \preceq 4L_i^4I$ since $X_k$ is symmetric. Hence, from matrix Azuma (Prop.~\ref{prop:mazuma}), w.p. $1-\delta'_t$:
	\begin{align*}
	\lambda_{\max}\left(\sum_{k=1}^t X_k\right) \leq 4L_i^2\sqrt{2t\log(d/\delta'_t)}.
	\end{align*}
	The proof is completed by union bound over time with $\delta'_t=\delta/(2t^2)$.
\end{proof}

The proof of the following key lemma provides some insights on how \algo is able to mix representations.

\begin{lemma}\label{lem:mixing.zeroregret}
	Make the same assumptions of Theorem~\ref{thm:algo.mix.regret.constant} and
	assume the good event $\mathcal{G}(\delta)$ defined in~\eqref{eq:selection.good} holds. With probability at least $1-\delta$, for each representation $i\in[M]$, there exists a constant $\check{\tau}_i$ such that, for all $t\ge\check{\tau}_i$, whenever $(x_{t+1},a_{t+1})\in Z_i$, the instantaneous regret of \algo is $r_{t+1}=0$. Moreover, for all $t\ge\max_{i\in[M]}{\check{\tau}_i}$, $r_{t+1}=0$.
\end{lemma}
\begin{proof}
	Recall from Definition~\ref{def:hls.mix} that $M_i=\EV_{x\sim\rho}[\phi_i^\star(x)\phi_i^\star(x)^\transp]$, and $Z_i=\{(x,a)\in\mathcal{X}\times\mathcal{A}|\phi_i(x,a)\in\Imm(M_i)\}$, where $\Imm(M_i)$ denotes the column space of $M_i$. We will fix a representation $i\in[M]$, assume $(x_{t+1},a_{t+1})\in Z_i$, and show that $r_{t+1}=0$ if $t\ge\check{\tau}_i$. Since the $M$ representations together satisfy the mixed-\hls condition, $\bigcup_{i\in[M]}Z_i$ covers $\mathcal{X}\times\mathcal{A}$ and $r_{t+1}=0$ after the maximal $\check{\tau}_i$.
	
	From Lemma~\ref{lem:mixing.loewner} (instantiated with $\delta\gets\delta/M$ for each representation, from which the $1-\delta$ overall probability):
	\begin{equation}
	V_{t+1,i} \succeq \lambda I_{d_i} + tM_i - \left(8L_i^2\sqrt{t\log(2Md_it/\delta)} + L_i^2 g_{ti}(\delta/M)\right)I_{d_i} \coloneqq B_{ti}.
	\end{equation}
	Note that $B_{ti}$ is invertible for all $t\ge T_i$ for a sufficiently large constant $T_i$. To see this, interpret the RHS as a time-dependent affine transformation of $M_i$, and note that those eigenvalues of $M_i$ that are initially zero keep decreasing with time, while those that are initially nonzero start increasing forever after a finite time. 
	Without loss of generality, we assume that $T_i\le\check{\tau}_i$. Otherwise, it suffices to replace $\check{\tau}_i$ with $\max\{\check{\tau}_i,T_i\}$.
	
	Let $\boldsymbol{v}\coloneqq\phi_i(x_{t+1},a_{t+1})/\norm{\phi_i(x_{t+1},a_{t+1})}$. From the Loewner ordering and the invertibility:
	\begin{equation}
	\boldsymbol{v}^\transp V_{t+1,i}^{-1}\boldsymbol{v} \le \boldsymbol{v}^\transp B_{ti}^{-1}\boldsymbol{v}.
	\end{equation}
	Since by assumption $(x_{t+1},a_{t+1})\in Z_i$, $\boldsymbol{v}\in\Imm(M_i)$, hence $\boldsymbol{v}$ belongs to the span of $k$ orthonormal eigenvectors of $M_i$, where $k$ is the number of nonzero eigenvalues of $M_i$. Note that the orthonormal eigenvectors of $B_{ti}$ coincide with the orthonormal eigenvectors of $M_i$ since $B_{ti}$ is an affine transformation of $M_i$.
	This means that $\boldsymbol{v}$ belongs to the span of $k$ eigenvectors of $B_{ti}$.
	Also note that, for $t\ge T_i$, the nonzero eigenvalues of $M_i$ correspond to \emph{positive} eigenvalues of $B_{ti}$. This means that $\boldsymbol{v}$ belongs to the span of $k$ eigenvectors of $B_{ti}$ having positive corresponding eigenvalues. The smallest of these eigenvalues is:
	\begin{equation}
	\lambda + t\lambda^+_i - \left(8L_i^2\sqrt{t\log(2Md_it/\delta)} + L_i^2 g_{ti}(\delta/M)\right),
	\end{equation}
	where $\lambda^+_i$ denotes the smallest nonzero eigenvalue of $M_i$,
	and all of these eigenvalues are upper-bounded by:
	\begin{equation}
	\lambda + tL_i^2- \left(8L_i^2\sqrt{t\log(2Md_it/\delta)} + L_i^2 g_{ti}(\delta/M)\right).
	\end{equation}
	Now from Lemma~\ref{lem:kanto}:
	\begin{equation}
	\sqrt{\boldsymbol{v}^\transp V_{ti}^{-1}\boldsymbol{v}} \le \sqrt{\boldsymbol{v}^\transp B_{ti}^{-1}\boldsymbol{v}} \le \frac{\lambda + tL_i^2- \left(8L_i^2\sqrt{t\log(2Md_it/\delta)} + L_i^2 g_{ti}(\delta/M)\right)}{\lambda + t\lambda^+_i - \left(8L_i^2\sqrt{t\log(2Md_it/\delta)} + L_i^2 g_{ti}(\delta/M)\right)} \frac{1}{\sqrt{\boldsymbol{v}^\transp B_{ti}\boldsymbol{v}}}.
	\end{equation}
	Finally, since $\boldsymbol{v}$ is orthogonal to all the eigenvectors of $B_{ti}$ that correspond to zero eigenvalues, from Lemma~\ref{lem:minposeig}:
	\begin{align}
	\boldsymbol{v}^\transp B_{ti}\boldsymbol{v} &\ge \lambda + t\boldsymbol{v}^\transp M_i\boldsymbol{v} - \left(8L_i^2\sqrt{t\log(2Md_it/\delta)} + L_i^2 g_{ti}(\delta/M)\right) \\
	&\ge \lambda + t\min_{\substack{\boldsymbol{v}\in\Imm(M_i)\\\norm{\boldsymbol{v}}=1}}\boldsymbol{v}^\transp M_i\boldsymbol{v} - \left(8L_i^2\sqrt{t\log(2Md_it/\delta)} + L_i^2 g_{ti}(\delta/M)\right) \\
	&\ge \lambda + t\lambda^+_i - \left(8L_i^2\sqrt{t\log(2Md_it/\delta)} + L_i^2 g_{ti}(\delta/M)\right).
	\end{align}
	Finally, from Lemma~\ref{lem:selection.main}, under the same good event $\mathcal{G}(\delta)$:
	\begin{align}
	r_{t+1} &\le 2\sqrt{\beta_{t+1,i}}\norm{\phi(x_{t+1},a_{t+1})}_{V_{t+1}^{-1}} \nonumber\\
	&\le 
	2\sqrt{\beta_{t+1,i}}L_i\sqrt{\boldsymbol{v}^\transp V_{ti}^{-1}\boldsymbol{v}} \nonumber\\
	&\le
	\frac{2L_i\sqrt{\beta_{t+1,i}}\left[\lambda + tL_i^2- \left(8L_i^2\sqrt{t\log(2Md_it/\delta)} + L_i^2 g_{ti}(\delta/M)\right)\right]}
	{\left[\lambda + t\lambda^+_i - \left(8L_i^2\sqrt{t\log(2Md_it/\delta)} + L_i^2 g_{ti}(\delta/M)\right)\right]^{3/2}}\label{eq:mixing.deftau}	\\
	&=  \frac{\widetilde{O}(t)}{\Omega(t\sqrt{t})-\widetilde{O}(t)} = \widetilde{O}\left(\frac{1}{\sqrt{t}}\right).\nonumber
	\end{align}
	Hence we can find $\check{\tau}_i$ such that, for all $t\ge\check{\tau}_i$, $r_{t+1}<\Delta$. By definition of $\Delta$, $r_{t+1}=0$.
\end{proof}

\paragraph{Remark.}
Although we do not provide an explicit value for $\check{\tau}_i$ here, by examining~\eqref{eq:mixing.deftau} we can conclude that $\sqrt{\check{\tau}_i}\propto L_i^3/(\Delta(\lambda^+_i)^{3/2})\ge L_i^2/(\Delta\sqrt{\lambda^+_i})$ hence $\check{\tau}_i\propto L_i^4/(\Delta^2\lambda^+_i)$. So $\lambda^+_i$ (which is always nonzero for non-degenerate representations) has effectively replaced $\lambda_{i,\hls}$ for the context-action pairs belonging to $Z_i$. In this sense, we can say that $\phi_i$ is always "locally \hls" w.r.t. $Z_i$.

\paragraph{Proof of Theorem~\ref{thm:algo.mix.regret.constant}.}
\begin{proof}
	Assume the good event $\mathcal{G}(\delta)$ from~\eqref{eq:selection.good} holds.
	Let $\tau=\max_{i\in[M]}\check{\tau}_i$, where $\check{\tau}_i$ is defined as in Lemma~\ref{lem:mixing.zeroregret}. From the same lemma, with probability $1-\delta$, $r_{t+1}=0$ if $t\ge\tau$. Hence, if $n\ge\tau$:
	\begin{equation*}
		R_n = \sum_{t=1}^{n}r_t 
		\le \sum_{t=1}^{\tau}r_t.
	\end{equation*}
	So, in any case, from Lemma~\ref{lem:selection.main}:
	\begin{align*}
		R_{n} 
		\le \sum_{t=1}^{\min\{\tau,n\}}r_t
		\le 2\sum_{t=1}^{\min\{\tau,n\}}\min_{i\in[M]}\beta_{ti}(\delta/M)\norm{\phi_i(x_t,a_t)}_{V_{ti}^{-1}}
		\le 2\min_{i\in[M]}\sum_{t=1}^{\min\{\tau,n\}}\beta_{ti}(\delta/M)\norm{\phi_i(x_t,a_t)}_{V_{ti}^{-1}}.
	\end{align*}
	We can bound $\beta_{ti}(\delta/M)\norm{\phi_i(x_t,a_t)}_{V_{ti}^{-1}}$ individually for each representation as in Proposition~\ref{prop:oful.log.regret}, obtaining the desired statement. The overall probability is $1-2\delta$, from a union bound. The first $1-\delta$ is for $\mathcal{G}$. The additional $1-\delta$ was required to apply Lemma~\ref{lem:mixing.zeroregret}.
\end{proof}

A corollary of Theorem~\ref{thm:algo.mix.regret.constant} provides a generalization of Lemma~\ref{prop:hls.regret} to the case of redundant representations, which is stated by~\cite{hao2020adaptive} without an explicit proof.

\begin{corollary}\label{cor:mixing.redundant}
	Consider a contextual bandit problem with realizable linear representation $\phi_i$ such that, for all $x\in\X$ and $a\in\A$:
	\begin{equation}
		\phi_i(x,a) \in \spann\{\phi_i^\star(x)\mid x\in\supp(\rho)\}.\label{eq:mixing.redundant.condition}
	\end{equation}
	Assume $\Delta > 0$, $\max_{x,a}\|\phi_i(x,a)\|_2 \leq L_i$ and $\|\theta^\star_i\|_2 \leq S_i$. Then, with probability at least $1-2\delta$, the regret of \linucb after $n \geq 1$ steps is at most	
		\begin{align*}
		R_n \leq 
		&\frac{32\lambda\Delta_{\max}^2 S_{\phi}^2\sigma^2}{\Delta} 
		\left(2\ln\left(\frac{1}{\delta}\right)
		+d_{\phi} \ln \left( 1+\frac{{\color{darkred}\check{\tau}_{i}} L_{\phi}^2}{\lambda d_{\phi}} \right) \right)^2,
		\end{align*}
		where $\check{\tau}_i$ is defined as in the proof of Lemma~\ref{lem:mixing.zeroregret}.
\end{corollary}
\begin{proof}
	We will show that $\phi$ by itself satisfies the mixed-\hls condition. Then, this is just Theorem~\ref{thm:algo.mix.regret.constant} in the special case $M=1$.
	
	Let $M_i$ and $Z_i$ be as in Definition~\ref{def:hls.mix}. Note that $\Imm(M_i)$ is precisely $\spann\{\phi_i^\star(x)\mid x\in\supp(\rho)\}$. Hence, condition~\eqref{eq:mixing.redundant.condition} is saying that $Z_i=\X\times\A$. But this means that $\phi_i$ by itself satisfies the mixed-\hls condition. 
\end{proof}

Corollary~\ref{cor:mixing.redundant} that, even if $\phi_i$ is not \hls, but it spans the same subspace spanned by the whole set of (redundant) features, \linucb (hence, also \algo) can achieve constant regret. 

Finally, we provide some examples of sets of representations that have the mixed-\hls property.

\paragraph{Example 1.} The following are two equivalent (non-\hls) $2$-dimensional representations, with $\theta^\star_1=\theta^\star_2=[1,1]$, for a problem with $2$ contexts, $2$ arms, and uniform context distribution. Optimal features are underlined.
\begin{align*}
&\underline{\phi_1(x_1,a_1)} = [2,0] &\phi_1(x_1,a_2)=[1,0]
\hfill
&&\underline{\phi_2(x_1,a_1)} = [2,0]
&&\color{darkred}\phi_2(x_1,a_2) = [0,1] 
\\
&\underline{\phi_1(x_2,a_1)} = [2,0] &\color{darkred}\phi_1(x_2,a_2)=[0,1]
\hfill
&&\underline{\phi_2(x_2,a_1)}=[2,0]
&&\phi_2(x_2,a_2)=[1,0].
\end{align*}
The first representation is \hls restricted to context $x_1$, i.e., $Z_1=\{(x_1,a_1),(x_1,a_2),(x_2,a_1)\}$. The second is \hls restricted to context $x_2$, i.e., $Z_2=\{(x_1,a_1),(x_2,a_1),(x_2,a_2)\}$. Since $Z_1\bigcup Z_2=\X\times\A$, $\{\phi_1,\phi_2\}$ is mixed-\hls.

\paragraph{Example 2.} The following are three equivalent (non-\hls) $2$-dimensional representations, with $\theta^\star_1=\theta^\star_2==\theta^\star_3=[1,1]$, for a problem with $2$ contexts, $3$ arms, and uniform context distribution. Optimal features are underlined.
{\small\begin{align*}
&\underline{\phi_1(x_1,a_1)} = [2,0] &\phi_1(x_1,a_2)=[1,0]
&&\color{darkred}\phi_1(x_1,a_3) = [0,1]
\hfill
&&\underline{\phi_2(x_1,a_1)} = [0,2]
&&&\color{darkred}\phi_2(x_1,a_2)=[1,0]
&&&\phi_2(x_1,a_3)=[0,1]
\\
&\underline{\phi_1(x_2,a_1)} = [2,0] &\color{darkred}\phi_1(x_2,a_2)=[0,1]
&&\phi_1(x_2,a_3)=[1,0]
\hfill
&&\underline{\phi_2(x_2,a_1)}=[0,2]
&&&\phi_2(x_2,a_2)=[0,1]
&&&\color{darkred}\phi_2(x_2,a_3)=[1,0].
\end{align*}
}
In this case $Z_1=\{(x_1,a_1),(x_1,a_2),(x_2,a_1),(x_2,a_3)\}$ and $Z_2=\{(x_1,a_1),(x_1,a_3),(x_2,a_1),(x_2,a_2)\}$. Since $Z_1\bigcup Z_2=\X\times\A$, $\{\phi_1,\phi_2\}$ is mixed-\hls.

\paragraph{Example 3.}
Consider the following $2$-dimensional representation, with $\theta^\star_1=[1,1]$, for a problem with $2$ contexts, $2$ arms, and uniform context distribution. 
\begin{align*}
&\underline{\phi_1(x_1,a_1)} = [2,0] &&\phi_1(x_1,a_2)=[1,0]\\
&\underline{\phi_1(x_2,a_1)} = [2,0] &&\phi_1(x_2,a_2)=[1,0].
\end{align*}
It is redundant since the features only span $\Reals^1$. For the same reason, it is not \hls. However, optimal features also span $\Reals^1$. So, according to Corollary~\ref{cor:mixing.redundant}, \linucb and \algo can achieve constant regret with access to this distribution. Indeed, $Z_1=\X\times\A$ and $\phi_1$ by itself satisfies the mixed-\hls condition. Redundancy is still a problem: the regret will scale with $d=2$ even if all features lie in a one-dimensional subspace. Redundant representations are the subject of study of \emph{sparse bandits}~\citep[e.g.,][]{abbasi2012online}.

\section{Dealing with Misspecified Representations}\label{app:elimination}

In this section, we introduce a variant of \algo, called E-\algo, that handles misspecified representations. More precisely, we consider the realizable setting \citep{agarwal2012contextual} where we are given a set of $M$ representations such that at least one is realizable. Then, we combine \algo with a statistical test that, under certain conditions, allows eliminating misspecified representations. Finally, we analyze the regret of E-\algo and show that, under a minimum misspecification assumption, the algorithm essentially retains the same guarantees of \algo for the subset of realizable representations, while suffering only constant regret for eliminating the misspecified ones. While we use a similar test as the one designed by \citet{agarwal2012contextual} based on the mean squared error, our setting and analysis involve several additional complications. In fact, \citet{agarwal2012contextual} consider a finite set of representations, each being a direct mapping from context and actions to rewards. On the other hand, in our case this mapping is given only after pairing some features $\phi$ (i.e., our concept of representation) with some parameter $\theta$. Since these parameters are unknown and lie in a continuous space, our setting basically involves infinite possible representations (according to the definition of \citet{agarwal2012contextual}).

\subsection{Formal Setting and Assumptions}

We assume to have a set of representations $\{\phi_i\}_{i\in[M]}$ such that
\begin{align}
    \mu(x,a) &= \phi_i(x,a)^T\theta^\star_i + f_i(x,a) \quad \forall i\in[M],x\in\X,a\in[K].
\end{align}
As for the realizable setting, we suppose that $\|\theta^\star_i\|_2 \in \mathcal{B}_i := \{\theta\in\mathbb{R}^{d_i} : \|\theta\|_2 \leq S_i\}$ and $\|\phi_i(x,a)\|_2\le L_i$ for all $i\in[M]$. Moreover, we suppose that the non-linear component $f_i$ (i.e., the misspecification) satisfies $|f_i(x,a)| \leq \|f_i\|_{\infty} := \epsilon_i \leq \epsilon$. The learner knows the upper bounds $S_i$, $L_i$, while $f_i,\epsilon_i,\epsilon$ are all unknown.

The following are our two key assumptions.
\begin{assumption}[Realizability]\label{ass:realizability}
There exists a subset $\mathcal{I}^\star \in [M]$ with $|\mathcal{I}^\star| \geq 1$ such that $\epsilon_{i} = 0$ for each $i\in\mathcal{I}^\star$
\end{assumption}

\begin{assumption}[Minimum misspecification]\label{ass:min-missp}
For each misspecified representation $i\in[M]\setminus\mathcal{I}^\star$, there exists $\gamma_i > 0$ such that
\begin{align}
\min_{\theta \in \mathcal{B}_i}\min_{\pi\in\Pi} \mathbb{E}_{x\sim\rho}\left[\left(\phi_i(x,\pi(x))^T \theta - \mu(x,\pi(x))\right)^2 \right] \geq \gamma_i,
\end{align}
where $\Pi$ is the set of all mappings from $\X$ to $[K]$.
\end{assumption}
Intuitively, Asm.~\ref{ass:min-missp} requires that, no matter what actions the algorithm takes, in expectation over contexts no linear predictor can perfectly approximate the reward function. A sufficient condition for this to hold is that there exists some contexts which occur with positive probability and for which all actions are misspecified. Finally, as a technical assumption, we suppose that $|\mu(x,a)| \leq 1$ and $|y_t| \leq 1$ almost surely. This simplifies the analysis and the notation, though it can be easily relaxed.

\paragraph{Additional Notation.} 

Finally, introduce some additional terms which will be recurrent in this section.

\begin{itemize}
\item $E_t(\phi,\theta) := \frac{1}{t} \sum_{k=1}^t \left(\phi(x_k,a_k)^T \theta - y_k\right)^2$: mean squared error (MSE) of representation $(\phi,\theta)$ at time $t$;
\item $\wt{\theta}_{ti} := \argmin_{\theta \in \mathcal{B}_i} E_{t-1}(\phi_i,\theta)$: least-squares estimate projected onto the ball of valid parameters;
\item $\mathbb{E}_k$ and $\mathbb{V}_k$: expectation and variance conditioned on the full history up to (and not including) round $k$.
\end{itemize}

\subsection{Algorithm}

The pseudo-code of elimination-based \algo (E-\algo), which combines \algo with a statistical test for eliminating misspecified representations, is provided in Alg.~\ref{alg:elimination}.

\begin{algorithm}[tb]
    \begin{small}
    \begin{algorithmic}
        \STATE {\bfseries Input:} representations $(\phi_i)_{i \in [M]}$ with values $(L_i,S_i)_{i\in [M]}$, regularization factor $\lambda \geq 1$, confidence level $\delta \in (0,1)$.
        \caption{E-\algo}\label{alg:elimination}
        \STATE Initialize $V_{1i} = \lambda I_{d_i}$, $\theta_{1i} = 0_{d_i}$ for each $i \in [M]$ and $\Phi_1 = [M]$
        \FOR{$t=1, \ldots$}
            \STATE Observe context $x_t$
            \STATE Pull action
            $a_t \in \argmax_{a \in [K]} \min_{i \in \Phi_t} \{U_{ti}(x_t,a)\}$
            \STATE Observe reward $r_t$ and, for each $i\in [M]$\\
                $V_{t+1,i}= V_{ti} + \phi_i(x_t,a_t)  \phi_i(x_t,a_t)^\transp$ \\
                $\theta_{t+1,i} =V_{t+1,i}^{-1} \sum_{l=1}^t \phi_i(x_l,a_l) r_l$\\
                $\wt{\theta}_{t+1,i} = \argmin_{\theta \in \mathcal{B}_i} \sum_{l=1}^t \left(\phi(x_l,a_l)^T \theta - r_l\right)^2$\\
                ${\Phi}_{t+1} := \left\{ i\in\Phi_{t} : E_{t}(\phi_{i}, \wt{\theta}_{t+1,i}) \leq \min_{j\in[M]}\min_{\theta\in\mathcal{B}_j} \left( E_{t}(\phi_j,\theta) + \alpha_{t+1,i} \right) \right\}$
        \ENDFOR
    \end{algorithmic}
    \end{small}
\end{algorithm}

The algorithm maintains a set of active representations $\Phi_t$ defined as $\Phi_1 := [M]$ and, for $t \geq 2$,
\begin{align}
{\Phi}_t := \left\{ i\in\Phi_{t-1} : E_{t-1}(\phi_{i}, \wt{\theta}_{ti}) \leq \min_{j\in[M]}\min_{\theta\in\mathcal{B}_j} \left( E_{t-1}(\phi_j,\theta) + \alpha_{tj} \right) \right\},
\end{align}
where
\begin{align}
\alpha_{tj} := \frac{20}{t-1}\log\frac{8M^2(12L_jS_j(t-1))^{d_j}(t-1)^3}{\delta} + \frac{1}{t-1}.
\end{align}
Note that E-\algo needs to compute a projected (onto the ball of valid parameters) version of the least-squares estimates to define the set $\Phi_t$. While we need this to simplify the analysis, in practice one can use the standard regularized least-squares estimates $\theta_{ti}$ instead of $\wt{\theta}_{ti}$. Finally, E-\algo proceeds exactly as \algo, by selecting, at each time step, the action with the minimum UCB across all active representations. Computationally, the main difference between the two algorithms is that E-\algo needs to maintain estimates of the mean squares error $E_t$. These estimates can be updated incrementally when using regularized least squares, which makes E-\algo as efficient as \algo.

\subsection{Analysis}

\subsubsection{Realizable representations are never eliminated}

The first three lemmas show that any realizable representation $i \in \mathcal{I}^\star$ is never eliminated from $\Phi_t$ with high probability.

\begin{lemma}\label{lemma:mse-single}
Let $\phi,\theta\in\mathbb{R}^d$ be such that $|\phi(s,a)^T\theta| \leq 1$. Then, for each $i\in\mathcal{I}^\star$, $t\geq 1$, and $\delta \in (0,1)$,
\begin{align}
\mathbb{P}\left( E_t(\phi_{i}, \theta_{i}^\star) > E_t(\phi,\theta) + \frac{20}{t}\log\frac{4t}{\delta}\right) \leq \delta.
\end{align}
\end{lemma}
\begin{proof}
This result can be obtained by reproducing the first part of the proof of Lemma 4.1 in \cite{agarwal2012contextual}. We report the main steps for completeness. 

Define $g(x,a) := \phi(x,a)^T\theta$, $g^\star(x,a) := \phi_{i}(x,a)^T\theta_{i}^\star$, and $Z_k := (g(x_k,a_k) - y_k)^2 - (g^\star(x_k,a_k) - y_k)^2$. Note that $\{Z_k\}$ is a martingale difference sequence
with $|Z_k|\leq 4$. 
Then, using Freedman's inequality (Lem.~\ref{lemma:freedman}), with probability at least $1-\delta$,
\begin{align*}
\sum_{k=1}^t \mathbb{E}_k[Z_k] - \sum_{k=1}^t Z_k \leq 2\sqrt{\sum_{k=1}^t \mathbb{V}_k[Z_k]\log\frac{4t}{\delta}} + 16\log \frac{4t}{\delta}.
\end{align*}
Using Lemma 4.2 of \cite{agarwal2012contextual}, we have that $\mathbb{V}_k[Z_k] \leq 4 \mathbb{E}_k[Z_k]$. Then, after a simple manipulation, the inequality above yields $-\sum_{k=1}^t Z_k \leq 20\log\frac{4t}{\delta}$. 
The proof is concluded by noting that $\sum_{k=1}^t Z_k = t(E_t(\phi,\theta) - E_t(\phi_{i}, \theta_{i}^\star))$.
\end{proof}
\begin{lemma}\label{lemma:mse-multi}
For each $\delta \in (0,1)$,
\begin{align}
\mathbb{P}\left(\exists t\geq 1, i\in\mathcal{I}^\star,j\in[M], \theta\in\mathcal{B}_j : E_t(\phi_{i}, \theta_{i}^\star) > E_t(\phi_j,\theta) + \frac{20}{t}\log\frac{8M^2(12L_jS_jt)^{d_j}t^3}{\delta} + \frac{1}{t}\right) \leq \delta.
\end{align}
\end{lemma}
\begin{proof}
We shall use a covering argument for each representation $j\in[M]$. First note that, for any $\xi >0$, there always exists a finite set $\mathcal{C}_j \subset \mathbb{R}^{d_j}$ of size at most $(3S_j/\xi)^{d_j}$ such that, for each $\theta \in \mathcal{B}_j$, there exists ${\theta'}\in\mathcal{C}_j$ with $\|\theta-{\theta'}\|_2 \leq \xi$ (see e.g. Lemma 20.1 of \cite{lattimore2020bandit}). Moreover, suppose that all vectors in $\mathcal{C}_j$ have $\ell_2$-norm bounded by $S_j$ (otherwise we can always remove vectors with large norm). Now take any two vectors $\theta,{\theta'}\in\mathcal{B}_j$ with $\|\theta-{\theta'}\|_2 \leq \xi$. We have
\begin{align*}
E_t(\phi_j,\theta) &= \frac{1}{t} \sum_{k=1}^t \left(\phi_j(x_k,a_k)^T \theta \pm \phi_j(x_k,a_k)^T {\theta}' - y_k\right)^2\\ &= \frac{1}{t} \sum_{k=1}^t \left(\phi_j(x_k,a_k)^T (\theta - {\theta}')\right)^2 + \frac{1}{t} \sum_{k=1}^t \left(\phi_j(x_k,a_k)^T {\theta}' - y_k\right)^2 + \frac{2}{t} \sum_{k=1}^t \left(\phi_j(x_k,a_k)^T (\theta - {\theta}')\right)\left(\phi_j(x_k,a_k)^T {\theta}' - y_k\right)\\ &\geq E_t(\phi_j,{\theta}') + \frac{2}{t} \sum_{k=1}^t \left(\phi_j(x_k,a_k)^T (\theta - {\theta}')\right)\underbrace{\left(\phi_j(x_k,a_k)^T {\theta}' - y_k\right)}_{|\cdot|\leq 2}\\ &\geq E_t(\phi_j,{\theta}') - \frac{4}{t} \sum_{k=1}^t \|\phi_j(x_k,a_k)\|_{2} \|\theta - {\theta}'\|_2 \geq  E_t(\phi_j,{\theta}') - 4L_j\xi.
\end{align*}
Using $\xi = \frac{1}{4L_jt}$,
\begin{align*}
&\mathbb{P}\left(\exists t\geq 1, i\in\mathcal{I}^\star, j\in[M], \theta\in\mathcal{B}_j : E_t(\phi_{i}, \theta_{i}^\star) > E_t(\phi_j,\theta) + \frac{20}{t}\log\frac{4t^3}{\delta'_t} + \frac{1}{t}\right)\\ &\leq \sum_{t=1}^\infty\sum_{i\in\mathcal{I}^\star}\sum_{j\in[M]}\mathbb{P}\left(\exists \theta\in\mathcal{B}_j : E_t(\phi_{i}, \theta_{i}^\star) > E_t(\phi_j,\theta) + \frac{20}{t}\log\frac{4t^3}{\delta'_t} + \frac{1}{t}\right)\\ &\leq \sum_{t=1}^\infty\sum_{i\in\mathcal{I}^\star}\sum_{j\in[M]}\mathbb{P}\left(\exists {\theta}'\in\mathcal{C}_j : E_t(\phi_{i}, \theta_{i}^\star) > E_t(\phi_j,{\theta'}) - \frac{1}{t} + \frac{20}{t}\log\frac{4t^3}{\delta'_t} + \frac{1}{t}\right)\\ &\leq \sum_{t=1}^\infty\sum_{i\in\mathcal{I}^\star}\sum_{j\in[M]}\sum_{{\theta}'\in\mathcal{C}_j}\mathbb{P}\left( E_t(\phi_{i}, \theta_{i}^\star) > E_t(\phi_j,{\theta}') + \frac{20}{t}\log\frac{4t^3}{\delta'_t}\right) \leq \sum_{t=1}^\infty\sum_{i\in\mathcal{I}^\star}\sum_{j\in[M]}\sum_{{\theta}'\in\mathcal{C}_j} \frac{\delta'_t}{t^2}\\ &\leq M^2\sum_{t=1}^\infty \frac{\delta'_t}{t^2}(12L_jS_jt)^{d_j}.
\end{align*}
Here the first inequality is from the union bound, the second one follows by relating $\theta$ with its closest vector in the cover as above, the third one is from another union bound, the fourth one uses Lemma \ref{lemma:mse-single}, and the last one is from the maximum size of the cover. The result follows by setting $\delta'_t = \frac{\delta}{2M^2(12L_jS_jt)^{d_j}}$.
\end{proof}
\begin{lemma}\label{lemma:mse-full}[Validity of representation set]
With probability at least $1-\delta$, for each $i\in\mathcal{I}^\star$ and time $t\geq 1$, $i\in\Phi_t$.
\end{lemma}
\begin{proof}
Clearly $i\in\Phi_1$ by definition.
From Lemma \ref{lemma:mse-multi}, with probability at least $1-\delta$, for any $i\in\mathcal{I}^\star$ and $t\geq 1$,
\begin{align*}
E_t(\phi_{i}, \theta_{i}^\star) &\leq \min_{j\in[M]}\min_{\theta\in\mathcal{B}_j} \left( E_t(\phi_j,\theta) + \frac{20}{t}\log\frac{8M^2(12L_jS_jt)^{d_j}t^3}{\delta} + \frac{1}{t} \right).
\end{align*}
By definition, $\wt{\theta}_{t+1,i}$ is the vector in $\mathcal{B}_i$ minimizing $E_t(\phi_{i}, \theta)$. Therefore, since $\theta_i^\star\in\mathcal{B}_i$,
\begin{align*}
E_t(\phi_{i}, \wt{\theta}_{t+1,i}) \leq E_t(\phi_{i}, \theta_{i}^\star),
\end{align*}
which concludes the proof.
\end{proof}

\subsubsection{Misspecified representations are eventually eliminated}

Next, we show that, thanks to Asm.~\ref{ass:min-missp}, all the misspecified representations are eliminated from $\Phi_t$ at some point.

\begin{lemma}\label{lemma:martingale-cover}
Let $Z_k(\phi,\theta) := \left(\phi(x_k,a_k)^T \theta - y_k\right)^2$. Define $v_t := \sum_{k=1}^t\mathbb{E}_k\left[\mathbb{V}_k[y_k | x_k,a_k]\right]$ and $b_t(\phi,\theta) := \sum_{k=1}^t \mathbb{E}_k\left[\left(\phi(x_k,a_k)^T \theta - \mu(x_k,a_k)\right)^2\right]$ as the sum of conditional variances and biases, respectively. Then, for any $\delta\in(0,1)$,
\begin{align}
\mathbb{P}\left( \exists t\geq 1, i\in[M], \theta\in\mathcal{B}_i : \left| \sum_{k=1}^t Z_k(\phi_i,\theta) - b_t(\phi_i,\theta) - v_t \right| > 4\sqrt{t \log\frac{4Mt(12tS_iL_i)^{d_i}}{\delta}} + 2\right) \leq \delta
\end{align}
\end{lemma}
\begin{proof}
We follow a covering argument analogous to the one used in Lem.~\ref{lemma:mse-multi}. It is easy to see that, if two parameters $\theta,\theta'\in\mathcal{B}_i$ are $\xi$-close in $\ell_2$-norm, then
\begin{align*}
\sum_{k=1}^t \left| Z_k(\phi_i,\theta) - Z_k(\phi_i,\theta') \right| \leq \sum_{k=1}^t \underbrace{\big| \phi_i(x_k,a_k)^T (\theta+\theta') - 2y_k\big|}_{\leq 4}\underbrace{ \big|\phi_i(x_k,a_k)^T (\theta-\theta')\big|}_{\leq L_i\xi} \leq 4L_i\xi t.
\end{align*}
For each time $t\geq 1$ and $i\in[M]$, we build a $\xi_{ti}$-cover $\mathcal{C}_{ti}$ of $\mathcal{B}_i$, where $\xi_{ti} = 1/(4L_it)$. Therefore,
\begin{align*}
&\mathbb{P}\left( \exists t\geq 1, i\in[M], \theta\in\mathcal{B}_i : \left| \sum_{k=1}^t Z_k(\phi_i,\theta) - \sum_{k=1}^t  \mathbb{E}_k[Z_k(\phi_i,\theta)] \right| > 4\sqrt{t \log(2t/\delta'_t)} + 2\right)\\ & \qquad \leq \sum_{t\geq 1}\sum_{i\in[M]}\sum_{\theta\in\mathcal{C}_{ti}}\mathbb{P}\left( \left| \sum_{k=1}^t Z_k(\phi_i,\theta) - \sum_{k=1}^t  \mathbb{E}_k[Z_k(\phi_i,\theta)] \right| > 4\sqrt{t \log(2t/\delta'_t)} + 2 - 8L_i\xi_{ti}t \right)\\ & \qquad \leq M\sum_{t\geq 1} \delta'_t(12tS_iL_i)^{d_i},
\end{align*}
where the first inequality is from a union bound and a reduction to the cover $\mathcal{C}_{ti}$ (for both terms inside the absolute value)
while the second one uses Azuma's inequality (Lem.~\ref{lemma:azuma}). The proof is concluded by setting $\delta'_t = \frac{\delta}{2M(12tS_iL_i)^{d_i}}$ and noting that
\begin{align*}
\mathbb{E}_k[Z_k(\phi_i,\theta)] = \mathbb{E}_k\left[\mathbb{V}_k[y_k | x_k,a_k] + \left(\phi_i(x_k,a_k)^T \theta - \mu(x_k,a_k)\right)^2\right].
\end{align*}
\end{proof}

\begin{lemma}[Elimination]\label{lemma:elim-misp}
Let $i\in[M]$ be any misspecified representation. Then, with probability at least $1-2\delta$, we have $i\notin \Phi_{t+1}$ for all $t$ such that
\begin{align*}
t \geq \frac{1}{\gamma_i} \min_{j\in\mathcal{I}^\star} \left\{ 5 + 4\sqrt{t\log\frac{4Mt(12tS_iL_i)^{d_i}}{\delta}} + 4\sqrt{t\log\frac{4Mt(12tS_jL_j)^{d_j}}{\delta}} + 20\log\frac{8M^2(12L_jS_jt)^{d_j}t^3}{\delta} \right\}.
\end{align*}
\end{lemma}
\begin{proof}
Suppose that the misspecified representation $i\in[M]$ is active at time $t+1$ (i.e., $i \in \Phi_{t+1}$). Let $j\in\mathcal{I}^\star$ be any realizable representation. Then,
\begin{align*}
E_{t}(\phi_{i}, \wt{\theta}_{t+1,i}) &\leq \min_{j\in\Phi_{t+1}}\min_{\theta\in\mathcal{B}_j} \left( E_{t}(\phi_j,\theta) + \alpha_{t+1,i}\right) \leq E_t(\phi_{j},\theta_{j}^\star) + \alpha_{t+1,j}\\ &\leq \frac{v_t}{t} + 4\sqrt{\frac{1}{t}\log\frac{4Mt(12tS_jL_j)^{d_j}}{\delta}} + \frac{2}{t} + \alpha_{t+1,j},
\end{align*}
where the first inequality uses the definition of $\Phi_{t+1}$, the second one uses that $j$ is active at any time, and the last one uses Lem.~\ref{lemma:martingale-cover} together with $b_t(\phi_j,\theta_j^\star)$ = 0 since $j$ is realizable.
Similarly, by another application of Lem.~\ref{lemma:martingale-cover} to the left-hand side,
\begin{align*}
E_{t}(\phi_{i}, \wt{\theta}_{t+1,i}) &\geq \frac{b_t(\phi_i,\wt{\theta}_{t+1,i}) + v_t}{t} - 4\sqrt{\frac{1}{t}\log\frac{4Mt(12tS_iL_i)^{d_i}}{\delta}} - \frac{2}{t}.
\end{align*}
Combining these two and expanding the definition of $\alpha_{t+1,j}$,
\begin{align*}
b_t(\phi_i,\wt{\theta}_{t+1,i}) \leq 5 + 4\sqrt{t\log\frac{4Mt(12tS_iL_i)^{d_i}}{\delta}} + 4\sqrt{t\log\frac{4Mt(12tS_jL_j)^{d_j}}{\delta}} + 20\log\frac{8M^2(12L_jS_jt)^{d_j}t^3}{\delta}.
\end{align*}
Finally, by Asm.~\ref{ass:min-missp}
\begin{align*}
b_t(\phi_i,\wt{\theta}_{t+1,i}) &\geq \min_{\theta \in \mathcal{B}_i} b_t(\phi_i,\theta) = \min_{\theta \in \mathcal{B}_i} \mathbb{E}_{x\sim\rho}\left[\sum_{k=1}^t \left(\phi_i(x,\pi_k(x))^T \theta - \mu(x,\pi_k(x))\right)^2 \right]\\ &\geq t\min_{\theta \in \mathcal{B}_i}\min_{\pi\in\Pi} \mathbb{E}_{x\sim\rho}\left[\left(\phi_i(x,\pi(x))^T \theta - \mu(x,\pi(x))\right)^2 \right] \geq t\gamma_i.
\end{align*}
Relating the last two inequalities, where in the former we can take a minimum over all realizable representations, we obtain an inequality of the form $t\gamma_i \leq o(t)$, from which we can find $t$ such that $i$ is eliminated. The result holds with probability $1-2\delta$ due to the combination of two union bounds: one for the validity of the representation set, and the other for using Lem.~\ref{lemma:martingale-cover}.

\end{proof}

\subsubsection{Regret Analysis}

Given that all misspecified representations are eliminated at some point, the analysis follows straightforwardly from the one of \algo. In particular, if we let $\Gamma_i$ be the first (deterministic) time step such that the inequality in Lem.~\ref{lemma:elim-misp} holds, we can bound the immediate regret for each $t \leq \Gamma_i$ with its maximum value (i.e., $2$). Then, the total regret of E-\algo is
\begin{align*}
R_n = \sum_{t=1}^{\lceil\Gamma\rceil} r_t + \sum_{t=\lceil\Gamma\rceil+1}^n r_t \leq 2 \lceil\Gamma\rceil + \sum_{t=\lceil\Gamma\rceil+1}^n r_t,
\end{align*}
where $\Gamma := \max_{i\in[M]\setminus\mathcal{I}^\star}\Gamma_i$. Then, the second term above can be bounded using exactly the same proof as \algo since only the realizable representations remain active after $\lceil \Gamma \rceil$. Hence, E-\algo suffers only constant regret for eliminating misspecified representations. Moreover, if at least one realizable representation is \hls, or the mixing \hls condition holds, the second term above is also constant, and E-\algo suffers constant regret.

\subsection{Experiments}

We finally report some numerical simulations where we compare E-\algo with other model selection baselines (the one that performed better) on a toy problem with misspecification and on feature representations extracted from real data. 

\paragraph{Toy problem.} For the toy problem, we modified the experiment with varying dimensions of App.~\ref{app:exp.toy.vardim} . In addition to the representations considered in such experiment, we added four misspecified features, one with half the dimensionality of the base representation $\phi_{\mathrm{orig}}$ (see App.~\ref{app:exp.toy.vardim}), one with one third of the same, one generated randomly with dimension 3, and one generated randomly with dimension 9.

\paragraph{Last.fm dataset.} For the experiment with real data, we use the Last.fm dataset \citep{Cantador:RecSys2011}, which is a list of users and music artists, together with information about the artists listened by each user.  This dataset has been obtained from Last.fm (\url{https://www.last.fm/}) online music system.  We first preprocessed the dataset by keeping only artists listened by at least 70 users and users that listened at least to 10 different artists.
We thus obtained a dataset of 1322 users (which we treat as contexts) and 220 artists (which we treat as arms) reporting the number of times each user listened to each artist (which we treat as reward). Then, we generated multiple linear representations as follows. First, we extracted context-arm features via a low-rank factorization (taking only the highest 150 singular values) of the full matrix. Then, for each of $20$ representations we generated a random neural network architecture by sampling the number of hidden layers in $\{1,2\}$, the number of hidden neurons uniformly in $[50,200]$, and the size of the output layer uniformly in $[5,50]$. We obtained $20$ neural networks with $R^2$ score on a test set ranging from $0.5$ to $0.85$. For each of them, we took the last layer of the network as our linear model, thus obtaining $20$ linear representations with varying dimension and misspecification. We found that $12$ of these representations satisfy the \hls condition and took the best (in terms of $R^2$) of these as our ground-truth model. Finally, we normalized all these representations to have $\|\theta^\star\|_2=1$ by properly re-scaling $\phi$.

\paragraph{Results.} The results of running E-\algo and the other model-selection baselines on the two problems described above are shown in Fig.~\ref{fig:lastfm}. We compare to the baselines that perform best in the experiments of the main paper, i.e., \regbal and \regbalelim. For these baselines, we used $d\sqrt{n}$ as the oracle upper bound to the $n$-step regret of \linucb with a $d$-dimensional representation instead of the theoretical upper bound. Fig.~\ref{fig:lastfm} shows that E-\algo outperforms the other baselines and quickly transitions to constant regret. \regbalelim also transitions to constant regret since it eventually eliminates all misspecified representations but it takes an order of magnitude more steps than E-\algo to achieve so.

\begin{figure*}[t]
\centering
    \includegraphics{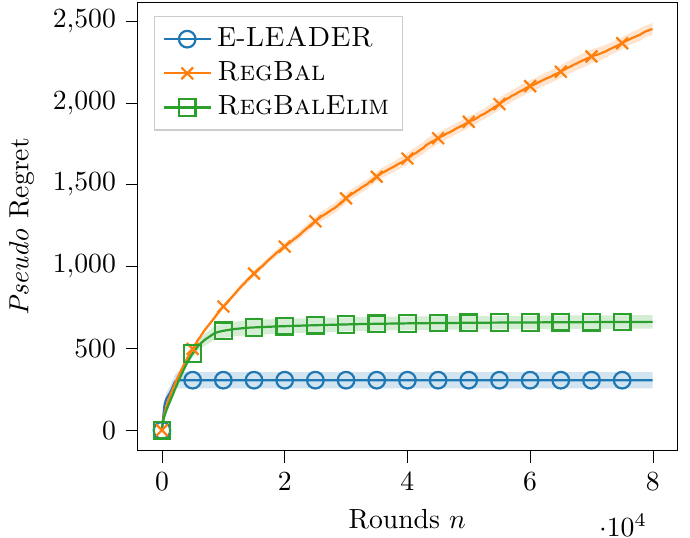}
	\includegraphics{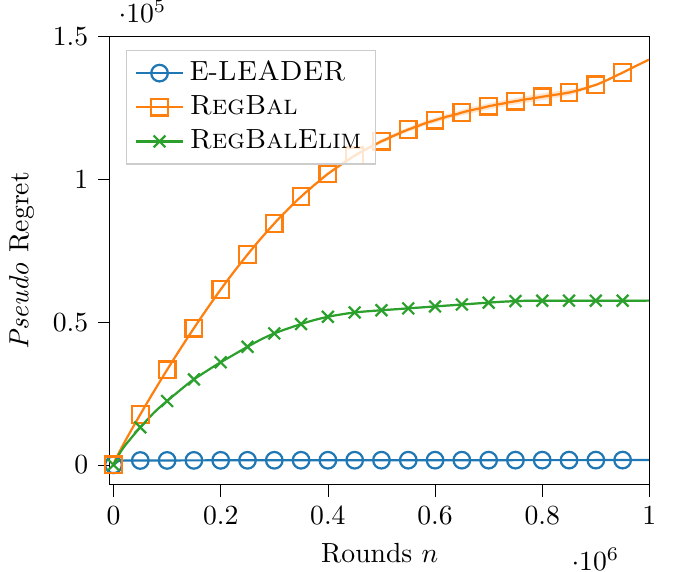}
	\caption{Regret of \algo and model-selection baselines on two misspecified problems. (left) Toy problem with varying dimensions and misspecified representations. (right) $20$ representations extracted from the Last.fm dataset.}
	\label{fig:lastfm}
\end{figure*}

\section{Experimental Details}\label{app:experiments}
In this section we provide further details on how the experiments from Section~\ref{sec:experiments} were designed, the experimental setup, and additional results.

\subsection{Synthetic Problems with Finite Contexts}\label{app.exp.toy}
The two synthetic experiments reported in Section~\ref{sec:experiments} and the motivating example from Figure~\ref{fig:example_intro} all pertain to the same randomly-generated contextual bandit problem with $20$, $5$ actions, uniform context distribution $\rho$, and noise standard deviation $\sigma=0.3$. All the representations we consider are normalized so that $S_i=\norm{\theta^\star_i}=1$, while $L_i$ can change.

To construct the reward function $\mu$ for our problem, we first generate a $6$-dimensional random feature map $\phi_{\text{orig}}:\X\times\A\to\Reals^6$. Specifically, each of the $20\times 5\times 6$ elements of $\phi_{\text{orig}}$ is sampled independently from a standard normal distribution.
We separately sample a $6$-dimensional parameter $\theta^\star_{\text{orig}}$, each element from a uniform distribution over the interval $[-1,1]$. 
We then normalize the parameter to have $S_{\text{orig}}=1$.
The reward function $\mu$ is defined a posteriori as $\mu(x,a)=\phi_{\text{orig}}(x,a)^\transp\theta^\star_{\text{orig}}$, so $\phi_{\text{orig}}$ is linearly realizable by construction.

From Lemma~\ref{lem:hls.properties.random}, $\phi_{\text{orig}}^\star$ is almost surely \hls. However, we verify this by checking that $\lambda_{{\text{orig}},\hls}$ is positive. All the representations we consider in the following are derived from and equivalent to $\phi_{\text{orig}}$.

\begin{figure*}
	\includegraphics{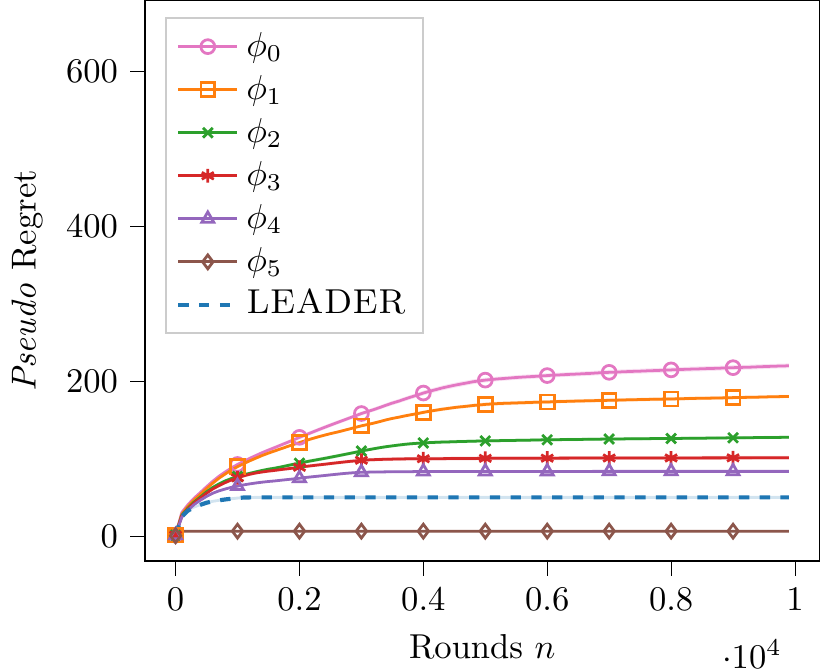}
	\includegraphics{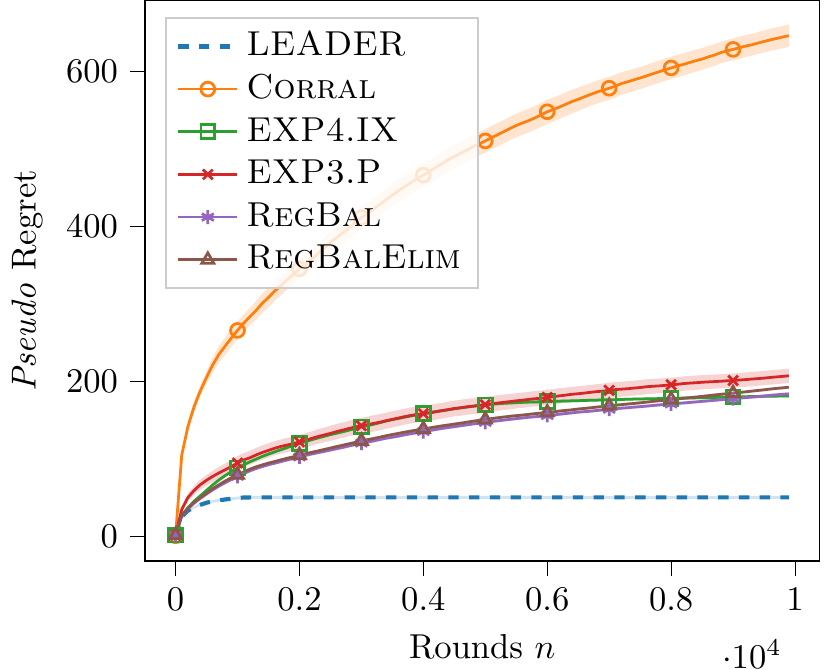}
	\caption{Regret of \algo, \linucb (left) with different representations and model-selection baselines (right) on the motivating representation-selection problem (App.~\ref{app:exp.toy.hls}).}
	\label{fig:app.hls}
\end{figure*}

\begin{figure*}
	\includegraphics{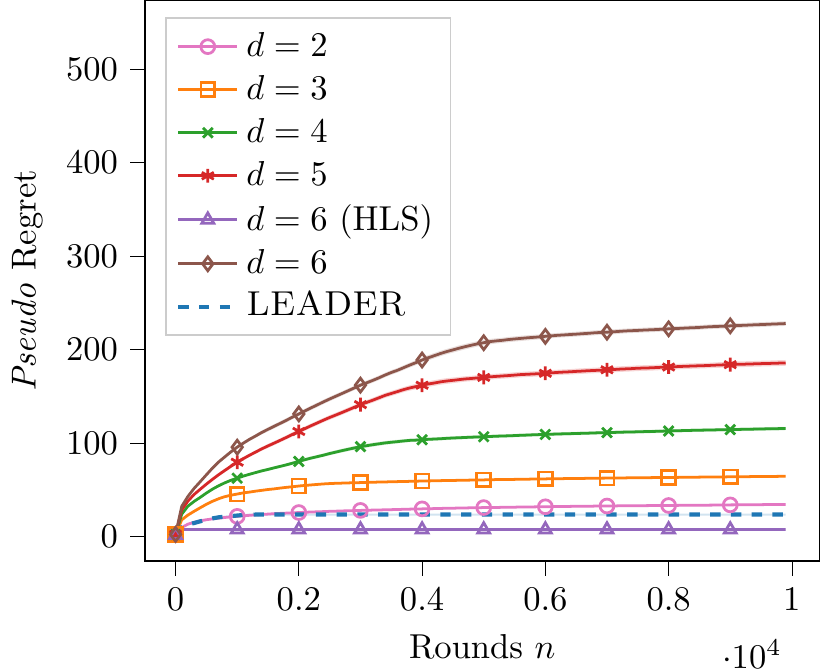}
	\includegraphics{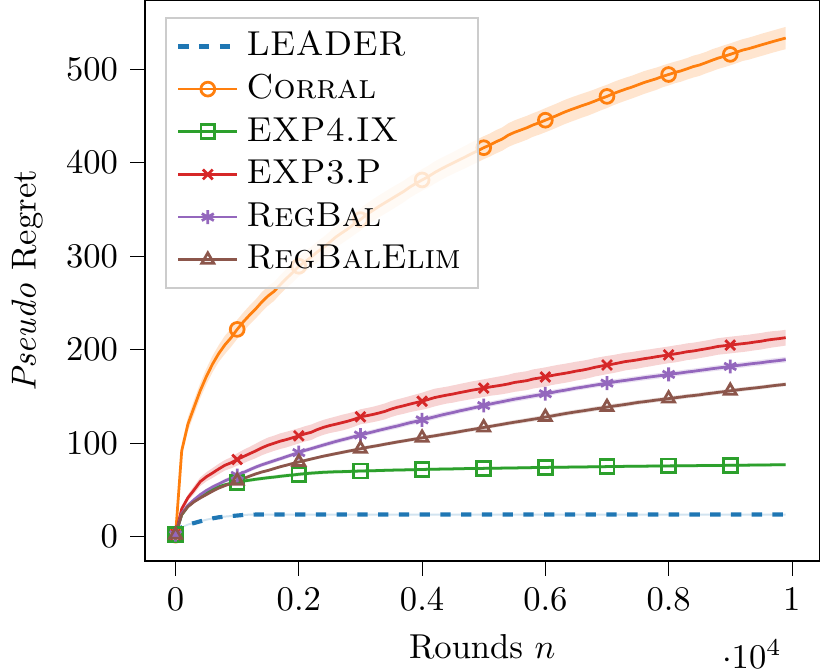}
	\caption{Regret of \algo, \linucb (left) with different representations and model-selection baselines (right) on the representation-selection problem with varying dimension (App.~\ref{app:exp.toy.vardim}).}
	\label{fig:app.vardim}
\end{figure*}

\begin{figure*}
	\includegraphics{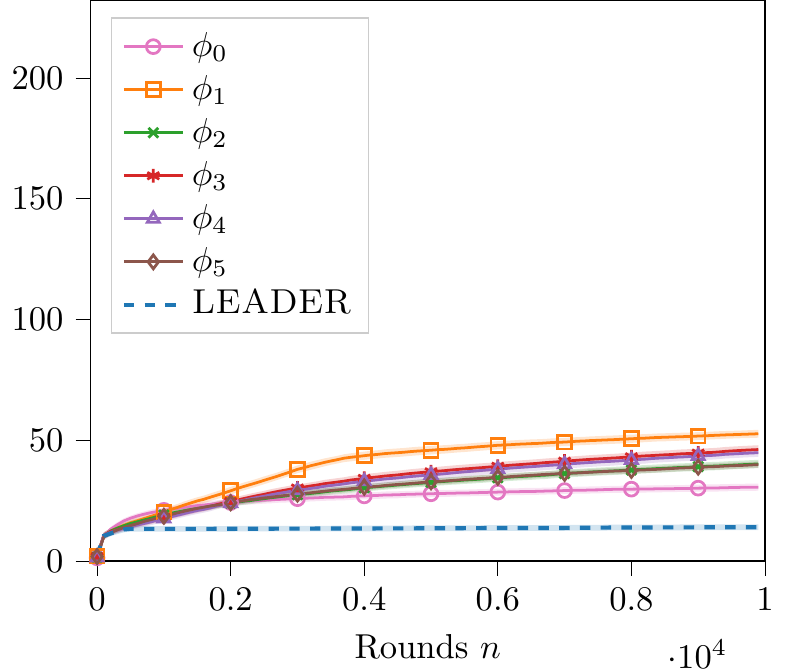}
	\includegraphics{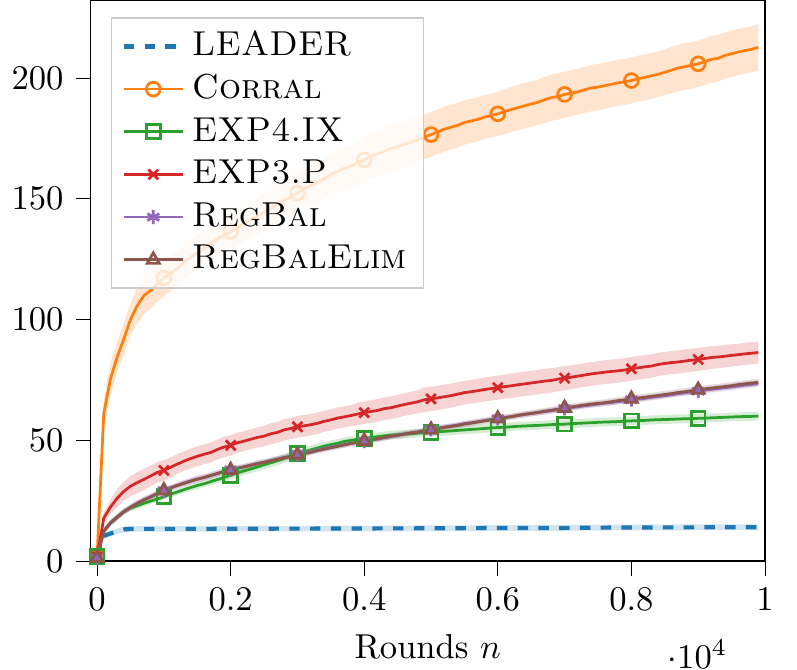}
	\caption{Regret of \algo, \linucb (left) with different representations and model-selection baselines (right) on the problem of mixing representations (App.~\ref{app:exp.toy.mixing}).}
	\label{fig:app.mixing}
\end{figure*}

\subsubsection{Motivating example}\label{app:exp.toy.hls}
The motivating example from Figure~\ref{fig:example_intro} is obtained from $\phi_{\text{orig}}$ ($\phi_5$ in the figure) by reducing the rank of $\EV_{x\sim\rho}[\phi_i^\star(x)\phi_i^\star(x)^\transp]$ from $6$ (full-rank), to $5,4,3,2$ and $1$, using the procedure described in Lemma~\ref{lem:hls.properties.derank}. 
A random linear transformation is then applied to each representation, which is also normalized to have $S_i=1$. Notice that this operations do not affect the rank of $\EV_{x\sim\rho}[\phi_i^\star(x)\phi_i^\star(x)^\transp]$ (Lemma~\ref{lem:hls.properties.linear}).
So we have a total of six representations, where only $\phi_{\text{orig}}$ is \hls. 
The experiment is reproduced in Figure~\ref{fig:app.hls}, where we also report the regret of the model-selection baselines.
The results are in line with the theory. Only \linucb with the \hls representation and \algo achieve constant regret, significantly outperforming the other \linucb instances.

\subsubsection{Varying dimension}\label{app:exp.toy.vardim}
Starting from $\phi_{\text{orig}}$, we construct five equivalent representations with dimension ranging from $2$ to $6$. To reduce the dimension, we just replace a subset of features of $\phi_{\text{orig}}$ (that is, elements of $\phi_{\text{orig}}(x,a)$ for all $x,a$) with a single feature that is a linear combination weighted by the corresponding elements of $\theta^\star_{\text{orig}}$. To obtain an equivalent representation, it is enough to set the corresponding element of the new $\theta^\star_i$ to $1$. We then apply the procedure from Lemma~\ref{lem:hls.properties.derank} to obtain $\rank(\EV_{x\sim\rho}[\phi^\star_i(x)\phi^\star_i(x)^\transp])=1$. Each representation is then passed through a random linear transformation and normalized to have $S_i=1$. So, of the seven representations, only the original $\phi_{\text{orig}}$ is \hls.
The experiment is reproduced in Figure~\ref{fig:app.vardim}, where we also report the regret of the model-selection baselines.

\subsubsection{Mixing representations}\label{app:exp.toy.mixing}
For this experiment, we derive from $\phi_{\text{orig}}$ six non-\hls $6$-dimensional representations with the mixed-\hls property. The original \hls representation itself is \emph{not} included in the set of candidates, to make mixing representations necessary.

To obtain the desired property, we remove two optimal features from each representation, replacing them with two copies of a new one that is an average of the two, similarly to what we did for the previous experiment, but without affecting the dimensionality. This operation makes the new representation non-\hls. However, since we only modify features of optimal arms, the new representation is still non-redundant. We pick the features to remove in a way that ensures the six representations together still have all the original features. Since the original $\phi_{\text{orig}}$ was \hls, it is easy to verify that the set of representations constructed in this way has the mixed-\hls property from Definition~\ref{def:hls.mix}. Again, each representation is passed through a random linear transformation and normalized to have $S_i=1$.
The experiment is reproduced in Figure~\ref{fig:app.mixing}. 
Differently from the experiment reported in Figure~\ref{fig:all} in the main paper, here \emph{all} the base algorithms of \regbal and \regbalelim are updated at each time step for more fair comparison.

\subsection{Synthetic Problem with Continuous Contexts}\label{app:exp.cont}
In this section, we test \algo on a synthetic contextual problem with continuous contexts. This example also clarifies the relationship between feature maps and context distributions in determining the goodness of a representation, showing how a change representation can sometimes correct for a "bad" distribution.

The continuous context space is $\X=\{x\in\Reals^d\mid\norm{x}\le1\}$. The finite action set is $\mathcal{A}=\{[0,0]^\transp, [0,1]^\transp, [1,0]^\transp, [1,1]^\transp\}$\footnote{We directly assign a vector in $\Reals^d$ to each one of the four actions for ease of notation.}. The context distribution $\rho$ is uniform over $\{(x_{[1]},x_{[2]})\in X\mid x_{[2]}\le 0\}$, where we use bracketed subscripts to index elements of vectors. Notice that $\rho$ assigns zero probability to some contexts in $\X$. The reward function is:
\begin{equation*}
	\mu(x,a) = x_{[1]}a_{[1]} + x_{[2]}a_{[2]}.
\end{equation*}
Intuitively, each action selects one, both or none of the elements of the current context and receives the sum of the selected elements as a reward.

A natural $2$-dimensional representation for this problem is:
\begin{equation*}
	\phi_1(x,a) = \left[x_{[1]}a_{[1]},\qquad x_{[2]}a_{[2]}\right]^\transp,
\end{equation*}
with $\theta^\star_1=[1,1]^\transp$. 
This would be a perfectly fine representation if $\rho$ was full-support.
However, since almost surely $x_{[2]}\le 0$, $\phi_1^\star(x)$ is either $[0,0]^\transp$ or $[x_{[1]},0]^\transp$ depending on the sign of $x_{[1]}$.
Hence, optimal features do not span $\Reals^2$, and $\phi_1$ is not \hls.

Consider another representation for the same problem, this time $3$-dimensional:
\begin{equation*}
	\phi_2(x,a) = \left[x_{[1]}a_{[1]}-x_{[1]},\qquad x_{[2]}a_{[2]}-x_{[2]},\qquad x_{[1]}+x_{[2]}\right]^\transp,
\end{equation*}
with $\theta^\star_2=[1,1,1]^\transp$. It is easy to see that the two representations are equivalent. However, now $\phi^\star_2(x)$ can be $[0,-x_{[2]},x_{[1]}+x_{[2]}]^\transp$ or $[-x_{[1]},-x_{[2]},x_{[1]}+x_{[2]}]^\transp$ depending on the sign of $x_{[1]}$, optimal features span $\Reals^3$, and $\phi_2$ is \hls. 

\begin{figure*}[t]
	\centering
	\includegraphics{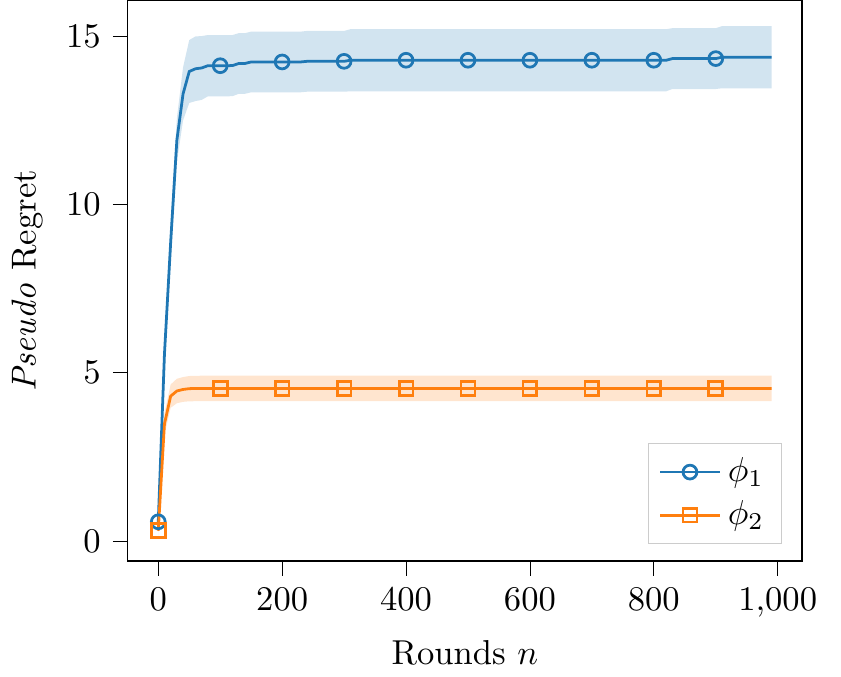}
	\caption{Regret of \linucb with two different representations on the synthetic problem with continuous contexts (App.~\ref{app:exp.cont}).}
	\label{fig:cont_oful}
\end{figure*}

In Figure~\ref{fig:cont_oful}, we show the regret of \linucb with the two representations ($\lambda=1$, $\delta=0.01$, $\sigma=0.2$, $20$ independent runs). Note that only $\phi_2$ achieves constant regret.

\subsection{Jester Dataset}

We report the comparison between \algo and the model-selection baselines on the $7$ feature representations extracted from the Jester dataset. For \regbal and \regbalelim, we use $d\sqrt{n}$ as the oracle upper bound to the regret of \linucb (instead of the theoretical one). Moreover, for all algorithms, we update all base learners with all collected samples (i.e., we share data across representations as \algo does). The results are shown in Fig.~\ref{fig:jester_modelsel}. Consistently with the other experiments, \algo outperforms all baselines and quickly transitions to constant regret thanks to the presence of multiple \hls representations.

\begin{figure*}[t]
\centering
	\includegraphics{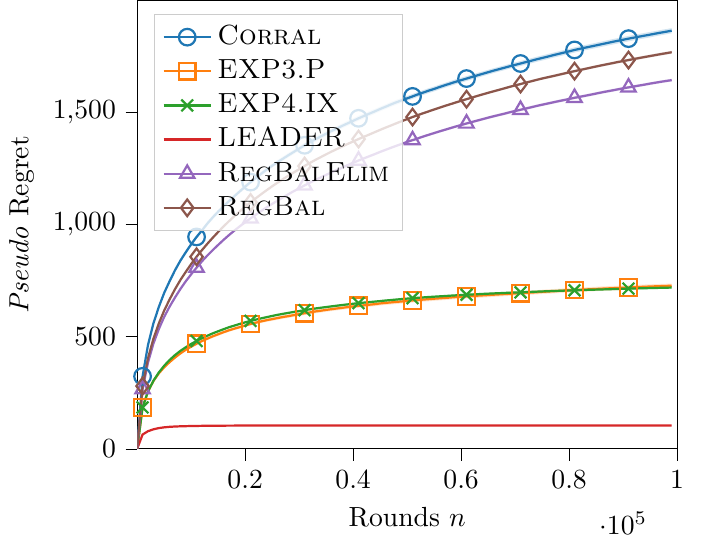}
	\caption{Regret of \algo and model-selection baselines on $7$ representations extracted from the Jester dataset.}
	\label{fig:jester_modelsel}
\end{figure*}


\subsection{Hyperparameters for Model-Selection Baselines}
In this section, we report hyperparameters and other settings used for the model-selection baselines.
All base algorithms are always updated at each time step unless otherwise stated.

\paragraph{\expfour.}
For the implicit-exploration parameter, we used the theoretical value $\gamma=\sqrt{2\log(M)/(nK)}$. The learning rate is $\eta=2\gamma$.

\paragraph{\corral. and \expthree}
We used the theoretical learning rate (respectively, exploration parameter) reported in App.~\ref{app:relatedwork}, with $c(\delta)=\sqrt{d}(\sqrt{\lambda}S + \sigma\sqrt{d-2\log(\delta))}$.

\paragraph{Regret Balancing (with Elimination)} 
As a regret oracle, we used $u:S_{ti}\mapsto4\beta_{ti}(\delta)\sqrt{t\log(\det V_{ti})}$ (for all bases).


\section{Auxiliary Results}\label{app:auxiliary}

\begin{lemma}[Azuma's inequality]\label{lemma:azuma}
Let $\{(Z_t,\mathcal{F}_t)\}_{t\in\mathbb{N}}$ be a martingale difference sequence such that $|Z_t| \leq a$ almost surely for all $t\in\mathbb{N}$. Then, for all $\delta \in (0,1)$,
\begin{align}
\mathbb{P}\left(\forall t \geq 1 : \left|\sum_{k=1}^t Z_k \right| \leq a\sqrt{t \log(2t/\delta)} \right) \geq 1-\delta.
\end{align}
\end{lemma}

\begin{lemma}[Freedman's inequality]\label{lemma:freedman}
Let $\{(Z_t,\mathcal{F}_t)\}_{t\in\mathbb{N}}$ be a martingale difference sequence such that $|Z_t| \leq a$ almost surely for all $t\in\mathbb{N}$. Then, for all $\delta \in (0,1)$,
\begin{align}
\mathbb{P}\left(\forall t \geq 1 : \left|\sum_{k=1}^t Z_k \right| \leq 2\sqrt{\sum_{k=1}^t \mathbb{V}_k[Z_k] \log(4t/\delta)} + 4a\log(4t/\delta) \right) \geq 1-\delta.
\end{align}
\end{lemma}

\begin{lemma}[Matrix Azuma,~\citealp{tropp2012user}]\label{prop:mazuma}
	Let $\{X_k\}_{k=1}^t$ be a finite adapted sequence of symmetric matrices of dimension $d$, and $\{C_k\}_{k=1}^t$ a sequence of symmetric matrices such that for all $k$, $\EV_{k}[X_k]=0$ and $X_k^2\preceq C_k^2$ almost surely. Then, with probability at least $1-\delta$:
	\begin{equation}
	\lambda_{\max}\left(\sum_{k=1}^tX_k\right) \le \sqrt{8\sigma^2\log(d/\delta)},
	\end{equation}
	where $\sigma^2=\norm{\sum_{k=1}^tC_k^2}$.
\end{lemma}

\begin{lemma}\label{lemma:span-eig}
	Let $\Phi = [\phi_1, \dots, \phi_n]$, where $\phi_i \in \mathbb{R}^d$ for $i\in[n]$, be such that $\mathrm{span}(\Phi) \subset \mathbb{R}^d$. Call $\lambda_i$ the $i$-th eigenvalue (using an arbitrary order) of the matrix $\Phi\Phi^T \in \mathbb{R}^{d\times d}$, and let $u_i$ be its corresponding eigenvector. Then, for any vector $\phi \in \mathbb{R}^d$, if $\phi \notin \mathrm{span}(\Phi)$, there exists $i\in[d]$ such that $\lambda_i = 0$ and $|\phi^T u_i| > 0$.
\end{lemma}
\begin{proof}
	We prove the lemma by contradiction. Suppose that, for all $i\in[d]$, either $\lambda_i > 0$ or $|\phi^T u_i| = 0$. Note that, since the columns of $\Phi$ do not span $\mathbb{R}^d$, $\mathrm{rank}(\Phi) = \mathrm{rank}(\Phi\Phi^T) < d$. This implies that at least one eigenvalue of $\Phi\Phi^T$ is zero. Moreover, if $\lambda_i=0$, the condition above implies that $|\phi^T u_i| = 0$. This means that $\phi$ is orthogonal to all eigenvectors of $\Phi$ associated to a zero eigenvalue. Define the matrix $\bar{\Phi} := [\Phi, \phi] \in \mathbb{R}^{d\times (n+1)}$. It must be that $\mathrm{rank}(\bar{\Phi}) > \mathrm{rank}(\Phi)$ (otherwise $\phi$ would lie in the span of the columns of $\Phi$). Note that
	\begin{align}
	\bar{\Phi}\bar{\Phi}^T = \sum_{i=1}^n \phi_i\phi_i^T + \phi\phi^T = \Phi\Phi^T + \phi\phi^T.
	\end{align}
	If $u_i$ is an eigenvector of $\Phi\Phi^T$ associated with a zero eigenvalue, then
	\begin{align}
	\bar{\Phi}\bar{\Phi}^Tu_i = \underbrace{\Phi\Phi^T u_i}_{= 0} + \phi\underbrace{\phi^T u_i}_{=0} = 0
	\end{align}
	since $u_i$ is orthogonal to $\phi$. Hence, $u_i$ is still an eigenvector of $\bar{\Phi}\bar{\Phi}^T$ associated with eigenvalue zero. Therefore, the number of non-zero eigenvalues of $\bar{\Phi}\bar{\Phi}^T$ is the same as the one of $\Phi\Phi^T$, which implies that $\mathrm{rank}(\bar{\Phi}) = \mathrm{rank}(\bar{\Phi}\bar{\Phi}^T) = \mathrm{rank}(\Phi\Phi^T) = \mathrm{rank}(\Phi)$. This yields the desired contradiction.
\end{proof}

\begin{lemma}[Kantorovich-like inequality]\label{lem:kanto}
	Let $\boldsymbol{v}\in\Reals^d$ with $\norm{\boldsymbol{v}}=1$ and $A\in\Reals^{d\times d}$ symmetric invertible with non-zero eigenvalues $\lambda_1\le\dots\le\lambda_d$ and corresponding orthonormal eigenvectors $u_1,\dots,u_d$. Let $\mathcal{I}\subseteq[d]$ be any index set. If $\boldsymbol{v}\in\spann\{u_i\}_{i\in \mathcal{I}}$ and $\lambda_i>0$ for all $i\in\mathcal{I}$:
	\begin{equation*}
	\boldsymbol{v}^\transp A^{-1}\boldsymbol{v} \le \frac{(\max_{i\in \mathcal{I}}\lambda_i+\min_{i\in \mathcal{I}}\lambda_i)^2}{4\max_{i\in \mathcal{I}}\lambda_i\min_{i\in \mathcal{I}}\lambda_i} \frac{1}{\boldsymbol{v}^\transp A\boldsymbol{v}}.
	\end{equation*}
\end{lemma}
\begin{proof}
	Since the eigenvectors are all orthogonal, $\boldsymbol{v}^\transp u_i=0$ for all $i\notin \mathcal{I}$ (meaning $i\in[d]\setminus \mathcal{I}$). Now consider the quadratic form:
	\begin{align}
	\boldsymbol{v}^\transp A\boldsymbol{v} &= \boldsymbol{v}^\transp Q\Sigma Q^\transp \boldsymbol{v} &&\text{(orthogonal diagonalization)} \nonumber\\
	&= \sum_{i=1}^d\lambda_{i}(\boldsymbol{v}^\transp u_i)^2 \nonumber\\
	&= \sum_{i\in \mathcal{I}}\lambda_{i}(\boldsymbol{v}^\transp u_i)^2 + \sum_{i\notin \mathcal{I}}\lambda_{i}(\boldsymbol{v}^\transp u_i)^2 \label{pp:cancel} \\
	&= \sum_{i\in \mathcal{I}}\lambda_{i}(\boldsymbol{v}^\transp u_i)^2 + \sum_{i\notin \mathcal{I}}\widetilde{\lambda}(\boldsymbol{v}^\transp u_i)^2 &&\text{(the second summation is zero anyway)}\label{eq:kanto.1}\\ &\coloneqq \boldsymbol{v}^\transp B\boldsymbol{v} ,\nonumber
	\end{align}
	where $\widetilde{\lambda}$ is any eigenvalue from $\{\lambda_i\}_{i\in \mathcal{I}}$ and $B$ is the symmetric matrix of which~\eqref{eq:kanto.1} is the eigendecomposition. The substitution is legit because the dot products in the second summation in~\eqref{pp:cancel} are all zero. The new matrix $B$ gives the same quadratic form but only has eigenvalues from $\{\lambda_i\}_{i\in \mathcal{I}}$.
	Since by hypothesis the surviving eigenvalues are all positive, $B$ is positive definite.
	We can do exactly the same for $\boldsymbol{v}^\transp A^{-1}\boldsymbol{v}$ since the orthonormal eigenvectors of $A^{-1}$ are the same as $A$. Hence:
	\begin{align}
	\boldsymbol{v}^\transp A^{-1}\boldsymbol{v}\boldsymbol{v}^\transp A\boldsymbol{v} &= \boldsymbol{v}^\transp B^{-1}\boldsymbol{v}\boldsymbol{v}^\transp B\boldsymbol{v} \nonumber\\
	&\le \frac{(\lambda_{\max}(B)+\lambda_{\min}(B))^2}{4\lambda_{\max}(B)\lambda_{\min}(B)} \label{pp:kanto} \\
	&=\frac{(\max_{i\in \mathcal{I}}\lambda_i+\min_{i\in \mathcal{I}}\lambda_i)^2}{4\max_{i\in \mathcal{I}}\lambda_i\min_{i\in \mathcal{I}}\lambda_i},\nonumber
	\end{align}
	where~\eqref{pp:kanto} is the standard Kantorovich matrix inequality~\citep[e.g.,][]{chen2013note}.
\end{proof}

\begin{lemma}\label{lem:minposeig}
	The smallest nonzero eigenvalue of symmetric p.s.d. matrix $A\in\Reals^{d\times d}$ is:
	\begin{equation*}
	\min_{\substack{\boldsymbol{v}\in\Imm(A)\\\norm{\boldsymbol{v}}=1}}\boldsymbol{v}^\transp A\boldsymbol{v},
	\end{equation*}
	where $\Imm(A)$ denotes the column space of $A$.
\end{lemma}
\begin{proof}
	Let $\lambda_1\le \lambda_2 \le \dots \le \lambda_d$ be the eigenvalues of $A$ with corresponding orthonormal eigenvectors $u_1,u_2,\dots,u_d$. The eigenvalues can be computed iteratively as:
	\begin{equation}
	\lambda_i = \min_{\substack{\boldsymbol{v}\in\{u_1,\dots,u_{i-1}\}^\bot\\\norm{\boldsymbol{v}}=1}}\boldsymbol{v}^\transp A\boldsymbol{v}.
	\end{equation}
	When $\lambda_{i}$ is the smallest nonzero eigenvalue of $A$, $\spann\{u_1,\dots,u_{i-1}\}$ is precisely the solution set of:
	\begin{equation}
	A\boldsymbol{v} = 0\boldsymbol{v} = 0,
	\end{equation}
	which is $\ker(A)$. Hence $\{u_1,\dots,u_{i-1}\}^\bot = \ker(A)^\bot = \Imm(A)$ since $A$ is symmetric.
\end{proof}

\end{document}